\pdfoutput=1
\newcommand{\online}{0}
\newcommand{\classstyle}{1}
\documentclass[12pt]{article}
\usepackage[utf8]{inputenc}
\usepackage{fullpage}
\usepackage{times}
\usepackage[normalem]{ulem}
\usepackage{fancyhdr,graphicx,amsmath,amssymb, mathtools, scrextend, titlesec, enumitem}
\usepackage{bm}
\usepackage{tcolorbox}
\usepackage{amsthm}
\usepackage{mathrsfs}  

\usepackage{amsfonts}
\numberwithin{equation}{section}
\ifnum\classstyle=3
\else
\newtheorem{theorem}{Theorem}[section]

\newtheorem{lemma}{Lemma}[section]

\newtheorem{remark}{Remark}[section]
\fi

\ifnum\classstyle=3
\else
\usepackage[hypertexnames=false]{hyperref}
\hypersetup{
    colorlinks = true,
    linkbordercolor = {white},
    linkcolor = {blue!50!black},
citecolor     = {blue!50!black},
}
\fi

\usepackage{units} 

\usepackage{xcolor,colortbl}

\usepackage{nicematrix}

\usepackage{caption}
\usepackage{subcaption}

\usepackage{algorithm}
\usepackage{algpseudocode}

\newcommand{\tabitem}{\textbullet~~}
\algrenewcommand\algorithmicindent{2.0em}
\newcommand{\IndState}{\State\hspace{2em}}
\usepackage{tikz}
\usepackage{pdfpages}

\usepackage{chngcntr}

\counterwithin{figure}{section}

\usepackage{ulem}

\usepackage{yhmath}

\usepackage{accents}
\newlength{\dhatheight}

\makeatletter
\@namedef{subjclassname@2020}{%
  \textup{2020} Mathematics Subject Classification}
\makeatother

\usepackage{arxiv}

\title{Generative Modelling with Tensor Train approximations of Hamilton--Jacobi--Bellman equations
}

\author{
 \textbf{David Sommer}\thanks{equal contribution} \\
	Weierstrass Institute for \\ Applied Analysis  and Stochastics \\
	Berlin, Germany \\
	\texttt{sommer@wias-berlin.de} 
 	\And
	\textbf{Robert Gruhlke}\footnotemark[1] \\
	Freie Universität Berlin \\ Institute of Mathematics \\
	Berlin, Germany \\
	\texttt{r.gruhlke@fu-berlin.de} 
 \And
 \textbf{Max Kirstein} \\
	Independent Researcher\\
	\And
 \textbf{Martin Eigel} \\
	Weierstrass Institute for \\ Applied Analysis  and Stochastics \\
	Berlin, Germany \\
	\texttt{eigel@wias-berlin.de}
	\And
	\textbf{Claudia Schillings} \\
	Freie Universität Berlin \\ Institute of Mathematics \\
	Berlin, Germany \\
	\texttt{c.schillings@fu-berlin.de}
}

\renewcommand{\shorttitle}{GM with TT approximations of HJB equations}

\begin{document}
\maketitle

\begin{abstract}

Sampling from probability densities is a common challenge in fields such as Uncertainty Quantification (UQ) and Generative Modelling (GM). In GM in particular, the use of reverse-time diffusion processes depending on the log-densities of Ornstein-Uhlenbeck forward processes are a popular sampling tool. In 
\ifnum\classstyle=2
\citep{berner2022optimal}
\else
\cite{berner2022optimal}
\fi
the authors point out that these log-densities can be obtained by solution of a \textit{Hamilton-Jacobi-Bellman} (HJB) equation known from stochastic optimal control. While this HJB equation is usually treated with indirect methods such as policy iteration and unsupervised training of black-box architectures like Neural Networks, we propose instead to solve the HJB equation by direct time integration, using compressed polynomials represented in the Tensor Train (TT) format for spatial discretization. Crucially, this method is sample-free, agnostic to normalization constants and can avoid the curse of dimensionality due to the TT compression. We provide a complete derivation of the HJB equation's action on Tensor Train polynomials and demonstrate the performance of the proposed time-step-, rank- and degree-adaptive integration method on a nonlinear sampling task in 20 dimensions.

\end{abstract}

\section{Introduction and related work}

Consider the problem of sampling from a probability measure $\mu_\ast$ on $\mathbb{R}^d$, $d\in\mathbb{N}$, with Lebesgue-density
\begin{equation}\label{eq:posterior}
    \pi_\ast(y) = \dfrac{1}{Z}\exp(-\Phi(y)),
\end{equation}
where $\Phi\colon \mathbb{R}^d \rightarrow \mathbb{R}$ is a sufficiently regular function called the \textit{potential} and $Z\in(0,\infty)$ is a normalization constant such that $\int_{\mathbb{R}^d}\pi_\ast(y)\mathrm{d}y = 1$. Throughout this manuscript, we assume that $\Phi$ is known and can be evaluated, while the normalization constant $Z$ is unknown and difficult or even impossible to compute. Over time, a myriad of different sampling methods have been devised, including Markov Chain Monte Carlo (MCMC) methods
\ifnum\classstyle=2
\cite{roberts2004general,brooks2011handbook,robert2011mcmc},
\else
\cite{roberts2004general,brooks2011handbook,robert2011mcmc},
\fi 
methods based on Stein variational gradient descent
\ifnum\classstyle=2
\cite{liu2016stein},
\else
\cite{liu2016stein},
\fi 
Langevin dynamics
\ifnum\classstyle=2\cite{roberts1996exponential,garbuno2020interacting,garbuno2020affine,reich2021fokker,reich2022kalman, eigel2022interaction},
\else
\cite{roberts1996exponential,garbuno2020interacting,garbuno2020affine,reich2021fokker,reich2022kalman, eigel2022interaction},
\fi
or Langevin dynamics preconditioned with measure transport 
\ifnum\classstyle=2\cite{zhang2023transport},
\else
\cite{zhang2023transport}
\fi
to name just a few. In the last few years, interacting particle systems have received a lot of attention 
\ifnum\classstyle=2
\cite{goodman2010ensemble,garbuno2020interacting,garbuno2020affine,reich2021fokker,reich2022kalman, eigel2022interaction}.
\else
\cite{goodman2010ensemble,garbuno2020interacting,garbuno2020affine,reich2021fokker,reich2022kalman, eigel2022interaction}.
\fi
An important application where one aims to sample from densities of the form \eqref{eq:posterior} stems from solutions of inverse problems via Bayesian inference
\ifnum\classstyle=2
\cite{stuart_2010}.
\else
\cite{stuart_2010}.
\fi

Since our approach is linked to (interacting particle-) Langevin samplers, we take a moment to review these methods in more detail. All methods proposed in 
\ifnum\classstyle=2\cite{garbuno2020interacting,garbuno2020affine,reich2021fokker,reich2022kalman} 
\else
\cite{garbuno2020interacting,garbuno2020affine,reich2021fokker,reich2022kalman} 
\fi 
work with an It\^o diffusion process of the form
\begin{equation}\label{eq:general_Langevin}
    \mathrm{d}X_t = f(X_t)\mathrm{d}t + g(X_t)\mathrm{d}W_t,
\end{equation}
where $W$ is a standard Brownian motion with appropriate dimension, $f$ is the drift and $g$ is the diffusion.
Under certain assumptions on the potential, e.g. (strong) convexity, convexity at infinity, regularity and growth conditions, 
this process is ergodic 
\ifnum\classstyle=2\cite{zhang2020wasserstein,garbuno2020affine}
\else
\cite{zhang2020wasserstein,garbuno2020affine}
\fi
and admits either $\mu_\ast$ (in the case of a single particle process) or $\otimes_{i=1}^B\mu_\ast$ (in the case of a system of $B\in\mathbb{N}$ interacting particles) as an invariant measure.



Samples from $\mu_\ast$ are obtained by propagating an initial batch of arbitrarily distributed samples through the process \eqref{eq:general_Langevin} for infinite time. In the classical overdamped Langevin dynamics, the drift term $f$ is given by the negative gradient $-\nabla\Phi$ of the potential. In state-of-the-art interacting particle methods like the \textit{Affine Invariant Langevin Dynamics} (ALDI)
\ifnum\classstyle=2
\cite{garbuno2020interacting},
\else
\cite{garbuno2020interacting},
\fi
this drift is modified by a reversible perturbation of the underlying process (see e.g. \cite[Equation 2.4]{zhang2023transport} for a general definition of reversible perturbations and \cite[Definition 3.1]{garbuno2020affine} for the specific perturbation of ALDI). Reversible perturbations can increase convergence speed 
\ifnum\classstyle=2
\cite{rey2016improving},
\else
\cite{rey2016improving},
\fi
while ensuring that the perturbed SDE maintains the same invariant measure as the unperturbed system and is still time-reversible. Even if the resulting system is time-reversible, the reverse-time process is not considered in those works, since the \textit{forward} process \eqref{eq:general_Langevin} admits $\mu_\ast$ as invariant measure.

While the time-homogeneous drift term of \eqref{eq:general_Langevin} makes these methods conceptually simple, it comes with a potential downside with regard to the class of measures $\mu_\ast$ that can be approximated. ALDI comes with theoretical convergence guarantees only in the case of a potential with Gaussian tails outside of a compact set
\ifnum\classstyle=2
\cite{garbuno2020interacting}.
\else
\cite{garbuno2020interacting}.
\fi
In 
\ifnum\classstyle=2
\cite{eigel2022interaction},
\else
\cite{eigel2022interaction},
\fi
the authors propose using a time-inhomogeneous process 
\begin{equation}
\label{eq:homotopy_Langevin}
    \mathrm{d}X_t = f(t,X_t)\mathrm{d}t + g(X_t)\mathrm{d}W_t,
\end{equation}
where $f(t,\cdot)$ is defined by gradients of log-densities of intermediate measures defined upon time dependent interpolation, e.g. a convex combination of the target potential $\Phi$ and a simpler auxiliary potential. 
By the choice of the auxiliary measure, the flow towards the target distribution is fixed. While this so called \textit{homotopy}-approach can substantially increase convergence speed in practice, e.g. to sample from multimodal target distributions, the choice of auxiliary measures allowing for optimal flows remains an open question.


Contrary, reverse-time diffusion processes offer a principled way of defining a process of the form \eqref{eq:homotopy_Langevin}, which can be used to sample from $\mu_\ast$. The key observation, dating back to 
\ifnum\classstyle=2
\cite{ANDERSON1982313}, 
\else
\cite{ANDERSON1982313},
\fi
is that the reverse-time process corresponding to \eqref{eq:homotopy_Langevin} defines again a diffusion process of the form \eqref{eq:homotopy_Langevin}. For some years, this property has been used in what is now called \textit{Diffusion Generative Modelling} 
\ifnum\classstyle=2\cite{song2019generative,ho2020denoising,song2020score}.
\else
\cite{song2019generative,ho2020denoising,song2020score}.
\fi
In contrast to Bayesian inference, where $\mu_\ast$ is known but difficult to sample from, the goal here is to generate new samples from some completely unknown data distribution of which a finite number of samples $\{x_i\}_{i=1}^D$, $D\in\mathbb{N}$, are available. The central idea is to use an Ornstein-Uhlenbeck process mapping any distribution to a standard-normal distribution $\mathcal{N}(0,I_d)$ for $t\rightarrow\infty$ and then, by using the available samples $\{x_i\}_{i=1}^D$, learning the drift of the reverse-time process, mapping $\mathcal{N}(0,I_d)$ back to the data distribution~
\ifnum\classstyle=2
\cite{song2020score}.
\else
\cite{song2020score}.
\fi
More specifically, the gradient-log-density or \textit{score} of the Ornstein-Uhlenbeck process is learned by minimizing a \textit{score-matching} objective function 
\ifnum\classstyle=2
\cite{hyvarinen2005estimation,pmlr-v115-song20a}, 
\else
\cite{hyvarinen2005estimation,pmlr-v115-song20a},
\fi
which is essentially a weighted time-average of mean-squared errors (see e.g. \cite[Equation 7]{song2020score}). The score determines the reverse process. Once the score is known, new samples from the data distribution can be obtained by sampling from the standard-normal distribution and propagating the samples through the reverse process. However, classical score-matching relies on the samples $\{x_i\}_{i=1}^D$ of the data distribution, which are usually not available in a Bayesian setting. Hence, we consider an alternative approach.

The authors of 
\ifnum\classstyle=2
\cite{berner2022optimal} 
\else
\cite{berner2022optimal}
\fi
point out that the negative log-density of a reverse-time diffusion process satisfies a \textit{Hamilton-Jacobi-Bellman} (HJB) equation. Since the score is invariant under additive constants to the log-density, it suffices to solve this HJB equation up to an additive constant to obtain the correct score. In particular, the normalization constant of the target density need not be known. Hence, solving the corresponding HJB equation is a viable method of obtaining the score in a Bayesian setting.

Tensor Trains 
\ifnum\classstyle=2
\cite{oseledets2011tensor} 
\else
\cite{oseledets2011tensor}
\fi
have been used successfully in several works on approximations of HJB equations for nonlinear optimal control, see e.g. 
\ifnum\classstyle=2\cite{oster2022approximating,dolgov2021} 
\else
\cite{oster2022approximating,dolgov2021}
\fi
and references therein. In 
\ifnum\classstyle=2\cite{oster2022approximating} 
\else
\cite{oster2022approximating}
\fi
the solution of the deterministic finite horizon HJB equation is obtained by a combination of Monte-Carlo (MC) sampling and policy iteration. While this approach is appealing due to its model-free nature, the policy iteration requires the solution of multiple nonlinear optimization problems at each time step. Furthermore, MC sampling may lead to slow convergence. In
\ifnum\classstyle=2
\cite{dolgov2021} 
\else
\cite{dolgov2021}
\fi
a spectral discretization is used, circumventing the slow convergence rate of MC sampling and achieving algebraic convergence for a class of deterministic infinite horizon optimal control problems. In contrast to these works, we propose a method not reliant on policy iteration. Furthermore, no nonlinear optimization has to be performed except at the initial time point. Instead, the HJB right-hand side is discretized by orthogonal projection onto polynomial space, resulting in an ODE in tensor space. Subsequently, this ODE is integrated using methods for time-integration of Tensor Trains.


\subsection{Contribution and Outline}
The main contribution of this work lies in providing an interpretable solver based on compressed polynomials for the reverse-time HJB equation as it appears in the context of Generative Modelling and Bayesian Inference. Specifically, we integrate the HJB equation using orthogonal projections of the right-hand side and rank-retractions onto a smooth manifold within polynomial space defined by Tensor Trains of a fixed rank. The solver adaptively chooses its stepsize based on current projection- and retraction-errors as well as the local stiffness, which is estimated by local linearizations of the HJB. This approach is sample-free and agnostic to normalization constants and can therefore be used in a Bayesian setting. We demonstrate the performance of the solver on a nonlinear test case in $d=20$ dimensions.


The outline of the rest of the paper is as follows.
\begin{itemize}
    \item Section \ref{sec:diffusion} covers the relevant theory of diffusion processes necessary to construct a process of the form \eqref{eq:homotopy_Langevin}, which can be used to sample from $\mu_\ast$. In particular, Remark \ref{rem:time-reverse_ornstein} offers one such form as a reverse-time Ornstein-Uhlenbek process. The corresponding reverse-time HJB equation determining the score of this process is given in \eqref{eq:ornstein_hjb}. 
    \item In Section \ref{sec:TTs} we introduce our approximation class for the log-densities, namely functional Tensor Trains with orthogonal polynomial ansatz functions. A motivation for this ansatz class can be found in 
    \ifnum\online=1
    the Online Supplementary Material.
    \else
    Appendix \ref{app:tt_motivation}.\fi This section further introduces all algebraic operations on tensor space necessary to solve a projected version of the HJB equation.
    \item Section \ref{sec:solver} is the main part of the paper, where we are concerned with the solution of the HJB. We state the equivalence of the HJB projected onto polynomial space of fixed degree with an ODE in tensor space (Theorem \ref{thm:tensor_hjb}). Furthermore, we give a precise version of the proposed solution algorithm (Algorithm \ref{alg:HJB_algo}).
    \item Finally, the performance of the solver is demonstrated on a Gaussian test case as well as a $20$-dimensional nonlinear potential in Section \ref{sec:numerics}.
\end{itemize}


\section{Reverse-time diffusion processes and HJB equation}\label{sec:diffusion}

Let the terminal time $T>0$ and a $d$-dimensional Ornstein-Uhlenbeck process $(X_t)_{t\in[0,T]}$ be defined by
\begin{equation}\label{eq:Ornstein_Uhlenbeck}
    \mathrm{d}X_t = -X_t\mathrm{d}t + \sqrt{2}\mathrm{d}W_t, \qquad X_0\sim\mu_\ast,
\end{equation}
where $W_t$ denotes standard $d$-dimensional Brownian motion. The probability density function $\pi_t$ of this process satisfies the Fokker-Planck equation
\begin{equation}\label{eq:ornstein_fokker}
    \partial_t \pi_t =  \Delta \pi_t+x\cdot\nabla \pi_t + d\pi_t, \quad \pi_0 = \pi_\ast,
\end{equation}
for $t\in[0,T]$. Since the (standard normal) invariant measure of \eqref{eq:Ornstein_Uhlenbeck} satisfies a log-Soboloev inequality, the corresponding law $\mu_{X_t}$ of \eqref{eq:Ornstein_Uhlenbeck} converges exponentially in Kullback-Leibler divergence (KL) to the standard normal distribution $\mathcal{N}(0,I_d)$ on $\mathbb{R}^d$ 
\ifnum\classstyle=2
\cite{markowich2000trend},
\else
\cite{markowich2000trend},
\fi
i.e.
\begin{equation}\label{eq:ornstein_convergence}
    \mathrm{KL}(\mu_{X_t} || \mathcal{N}(0,I_d)) \leq e^{-2t} \mathrm{KL}(\mu_\ast || \mathcal{N}(0,I_d)).
\end{equation}
Hence, for sufficiently large $T$, the measure $\mu_{X_T}$ will be close to a standard normal distribution in KL divergence. The following remark provides a reverse-time process $(Y_t)_{t\in[0,T]}$ with $Y_0 \sim \mu_{X_T}$ and $Y_T \sim \mu_\ast$.

\begin{remark}[Reverse-time Ornstein-Uhlenbeck process]\label{rem:time-reverse_ornstein}
    Let $(X_t)_{t\in[0,T]}$ be defined by \eqref{eq:Ornstein_Uhlenbeck}. Then, for any $\lambda\in[0,1]$ the process $(Y_t)_{t\in[0,T]}$ defined by
    \begin{equation}\label{eq:reverse_process}
        \mathrm{d}Y_t = \left[ Y_t + (2-\lambda)\nabla\log \pi_{T-t}(Y_t) \right] \mathrm{d}t + \sqrt{2(1-\lambda)}\mathrm{d}W_t, \qquad Y_0 \sim \mu_{X_T} 
    \end{equation}
    satisfies $\mu_{Y_t} = \mu_{X_{T-t}}$ and in particular $\mu_{Y_T} = \mu_\ast$. This result is an immediate consequence of \cite[Appendix G]{huang2021variational}, which covers a much wider range of diffusion processes. The most common choices for $\lambda$ are $\lambda=0$, used for the reverse process e.g. in 
    \ifnum\classstyle=2
    \cite{song2020score}, 
    \else
    \cite{song2020score},
    \fi
    and $\lambda=1$, which leads to a reverse ODE known as probability flow ODE 
    \ifnum\classstyle=2
    \cite{song2020score}.
    \else
    \cite{song2020score}.
    \fi
\end{remark}

To formulate the reverse process $(Y_t)_{t\in[0,T]}$ we need the score $\nabla\log\pi_t$. If a sufficient number of samples of $\mu_\ast$ are available, we can apply score matching techniques (see \cite{song2019generative,song2020score} and references therein). Lacking these samples, we could try to obtain $\pi_t$ by solving \eqref{eq:ornstein_fokker}, but the fact that $\pi_t$ needs to be a density for every $t$ makes this approach cumbersome for approximation methods. Instead, we apply a \textit{Hopf-Cole transformation} $v_t\coloneqq -\log\pi_t$ to \eqref{eq:ornstein_fokker}. A short calculation by product and chain rule 
\ifnum\classstyle=3
\else
(see Appendix \ref{app:hopf_cole}) 
\fi
yields that $v_t$ satisfies the PDE
\begin{equation}\label{eq:ornstein_hjb}
    \partial_t v_t = \Delta v_t + x\cdot \nabla v_t -  \|\nabla v_t\|_2^2 - d, \qquad v_0 = -\log \pi_\ast,
\end{equation}
for $t\in[0,T]$. This nonlinear PDE is the time-reverse of a HJB equation appearing in finite-horizon stochastic optimal control. As 
\ifnum\classstyle=2
\cite{berner2022optimal} 
\else
\cite{berner2022optimal}
\fi
pointed out, we can now apply techniques from optimal control to approximate the score. A straightforward way is to approximately solve the HJB equation \eqref{eq:ornstein_hjb} with some suitable class of functions such as Neural Networks
\ifnum\classstyle=2\cite{zhou2021,berner2020numerically,nusken2021solving}.
\else
\cite{zhou2021,berner2020numerically,nusken2021solving}.
\fi
Instead of this black-box approach, we propose solving \eqref{eq:ornstein_hjb} by means of compressed polynomials represented by a low-rank tensor format, the details of which are provided in the next section. In contrast to Neural Networks, this approach is highly interpretable and utilizes the structure of the HJB equation. In particular, we make frequent use of the fact that the right-hand side $F(v) \coloneqq \Delta v + x\cdot \nabla v -  \|\nabla v\|_2^2 - d$ of \eqref{eq:ornstein_hjb} can be split into a constant, linear and nonlinear contribution, given by
\begin{align}
    \operatorname{Const}(v) &= d, \label{eq:hjb_const}\\
    \operatorname{Lin}(v)  &=  \Delta v + x\cdot \nabla v, \label{eq:hjb_lin}\\ 
    \operatorname{NonLin}(v) & = - \|\nabla v\|^2. \label{eq:hjb_nonlin}
\end{align}

Before going into the detail about the polynomial approximation in the following section, we  briefly sketch some of the core ideas. 

First, we note the constant term \eqref{eq:hjb_const} can be dropped from \eqref{eq:ornstein_hjb} since the score is agnostic to constant shifts of the log-density. More precisely, $v_t$ is a solution to \eqref{eq:ornstein_hjb} if and only if $\overline{v}_t \coloneqq v_t + td$ is a solution to $\partial_t \overline{v}_t = \Delta \overline{v}_t + x\cdot \nabla \overline{v}_t -  \|\nabla \overline{v}_t\|_2^2, \overline{v}_0 = -\log \pi_\ast$, $t\in [0,T]$. The two solutions $v_t$ and $\overline{v}_t$ differ only by a constant shift for every $t\in[0,t]$, hence the score satisfies $\nabla\log\pi_t = -\nabla v_t = -\nabla\overline{v}_t$. By the same reasoning, an arbitrary constant can be added to the initial condition of \eqref{eq:ornstein_hjb} without affecting the score. By choosing this constant equal to $-\log(Z)$, we achieve $-\log\pi_\ast-\log(Z) = \Phi$. Thus, from now on we consider the equation
\begin{equation}\label{eq:shifted_hjb}
    \partial_t v_t = \operatorname{Lin}(v_t) + \operatorname{NonLin}(v_t), \qquad v_0 = \Phi.
\end{equation}

Morevover, if $v_t$ is a polynomial of fixed degree $N\in\mathbb{N}$ for any $t$, then $\operatorname{Lin}(v_t)$ is also a polynomial of degree $N$. This means that if $v_0$ is a polynomial of a fixed degree, integrating only the linear part of \eqref{eq:shifted_hjb} would yield a polynomial of same degree for all $t\in[0,T]$. For the nonlinear part $\operatorname{NonLin}(v_t)$ this is only true for $N=2$. In this quadratic case, \eqref{eq:shifted_hjb} can be solved to arbitrary accuracy. In particular, if $\mu_\ast$ is a zero mean Gaussian with density
\begin{equation}
    \pi_\ast(x) = \frac{1}{\sqrt{(2\pi)^{d}|\Sigma|}}e^{-x^{\intercal}\Sigma^{-1} x}
\end{equation}
for positive definite $\Sigma\in\mathbb{R}^{d, d}$, then \eqref{eq:shifted_hjb} corresponds to the HJB equation of a linear-quadratic optimal control problem with solution given by $v_t(x) = x^{\intercal}P_tx$, where $P_t\in\mathbb{R}^{d, d}$, $t\in[0,T]$, solves a Ricatti matrix differential equation\ifnum\classstyle=3
.
\else
(see Appendix \ref{app:oc}).
\fi

Solving \eqref{eq:shifted_hjb} e.g. with an explicit Euler method for time discretization leads to a steady increase of the degree over time for all initial degrees larger than $N=2$. This is due to the nonlinear term: if $v_t$ for some $t\in[0,T]$ is a polynomial of degree $N$, then $\operatorname{NonLin}(v_t)$ is (in general) a polynomial of degree $\leq 2N$. To prevent this degree increase, we \textit{project} the nonlinear part of the right-hand side back onto the space spanned by polynomials of degree $N$ before performing the time integration step. Furthermore, since the linear space of polynomials suffers from the curse of dimensionality, we use a \textit{compression} or \textit{retraction} after every time step, finding a best approximation of the new iterate in a low dimensional manifold. In the case of an explicit Euler method, the resulting integration scheme can be written as
\begin{equation}\label{eq:euler_sketch}
    v_{t+\tau_t} = \operatorname{Compression}\left[v_t + \tau_t \left( \operatorname{Lin}(v_t) + \operatorname{Projection}\left[\operatorname{NonLin}(v_t)\right] \right)\right],
\end{equation}
where $\tau_t>0$ is the current adaptively chosen stepsize. The precise definition of all terms involved is the subject of the next section.

\section{Functional Tensor Trains (FTT) and Tensor Trains (TT)}\label{sec:TTs}


\begin{table}
\begin{center}
  {\renewcommand{\arraystretch}{1.5}
    \begin{tabular}{cl}
    \hline\hline
          \rowcolor{gray!10!white}  $\bm{n}\in\mathbb{N}_0^d $ &  dimension array $\bm{n}=(n_1,\ldots,n_d)$ \\
          $k\bm{n}+l$ &  $(kn_1+l, \ldots, k n_d +l)$ for $k,l\in\mathbb{N}_0$ \\
          \rowcolor{gray!10!white} 
          $[\bm{n}]$ & indexing $[\bm{n}] = \bigtimes_{i=1}^d \{0,\ldots,n_i\}$ \\
          $\bm{n}_1\geq \bm{n}_2$, $\bm{n}\geq k$ & 
          component wise comparison $\bm{n},\bm{n}_1,\bm{n}_2\in\mathbb{N}^d$, $k\in\mathbb{N}$ \\
          \rowcolor{gray!10!white} 
          $\bm{\alpha}$, $\bm{\beta}$, $\bm{\gamma}$ & multiindex in $\mathbb{N}_0^d$, note that we always index starting from $0$ \\
          $\mathbb{R}^{\bm{n}}$  & tensor space $\mathbb{R}^{n_1,\ldots,n_d}$ \\
          \rowcolor{gray!10!white} 
          $\bm{A},\bm{B},\bm{C}$ &  tensor elements in $\mathbb{R}^{\bm{n}}$ \\
          $\bm{r}$ & rank $\bm{r}=(r_1,\ldots,r_{d-1})$ in $\mathbb{N}^{d-1}$ \\
          \rowcolor{gray!10!white} 
          $\bm{r}^1\bm{r}^2$ &  multiplication $\bm{r}^1\bm{r}^2=(r^1_1 r_1^2,\ldots,r^1_{d-1}r^2_{d-1})$ in $\mathbb{N}^{d-1}$ \\
        $k_i, l_i$ &  rank enumeration indices in $\{1,\ldots, r_i\}$ \\
        \rowcolor{gray!10!white} 
        $A_i,B_i,C_i$ &  component order $3$ tensor in $\mathbb{R}^{r_{i-1},n_i+1,r_i}$ 
      with entries indexed by $[k_{i-1},\alpha_i,k_i]$ \\
      $A_i[\alpha_i]$ & matrix extraction $A_i[\alpha_i]=A_i[\,:\, , \alpha_i, \, :\,]\in\mathbb{R}^{r_{i-1},r_i}$ of component tensor $A_i$ \\
      \rowcolor{gray!10!white} 
        $A_i[k_{i-1},\,:\,,k_i]$ & vector extraction in $\mathbb{R}^{n_i+1}$ for each rank enumeration $k_{i-1},k_i$ \\
        $\bm{A}[\bm{\alpha}]$ & tensor indexing $\bm{A}[\alpha_1,\ldots,\alpha_d]$ for $\bm{A}\in\mathbb{R}^{\bm{n}}$, $\bm{\alpha}\in[\bm{n}]$, $\bm{n}\in\mathbb{N}^d_0$\\
         \hline\hline
    \end{tabular}
    }
\end{center}
\caption{List of compact notations used in this work.}
\label{tab:notation}
\end{table}

In this section we introduce the approximation class used as a spatial discretization for the HJB equation. Let $K\subset \mathbb{R}^d$ be a compact hypercube defined by $a_{i}, b_{i}\in\mathbb{R}$ with $a_{i} < b_{i}$ for $i=1,\ldots,d$ and  $K=\bigtimes_{i=1}^d [a_{i},b_{i}]$. 
A function $f\colon K\to \mathbb{R}$ is said to have functional Tensor Train (FTT) 
\ifnum\classstyle=2
\cite{oseledets2013constructive} 
\else
\cite{oseledets2013constructive}
\fi
rank $\overline{\bm{r}}=(\overline{r}_1,\ldots, \overline{r}_{d-1})\in\mathbb{N}^{d-1}$ with the convention $\overline{r}_0=\overline{r}_d=1$, 
if it can be written as
\begin{equation}\label{eq:ftt}
f(x_1,\ldots, x_d) = F_1(x_1) F_2(x_2) \cdots F_d(x_d)
\end{equation}
with matrix valued functions  $F_i(x_i)\in\mathbb{R}^{\overline{r}_{i-1},\overline{r}_i}$, $x_i\in [a_i,b_i]$ for $i=1,\ldots,d$. For discussions regarding the approximation of functions of mixed regularity or compositional structures we refer to 
\ifnum\classstyle=2
\cite{bachmayr2021approximation,griebel2023low,bachmayr2023low}.
\else
\cite{bachmayr2021approximation,griebel2023low,bachmayr2023low}.
\fi

In order to obtain a discrete approximation class, for each $i=1,\ldots, d$ and $\alpha\in\mathbb{N}_0$ let $p_{\alpha}^i$ denote the $\alpha$-th orthonormal Legendre polynomial with respect to the standard $L^2$ inner product on $[a_i,b_i]$. For $\bm{n}\in\mathbb{N}_0^d$, we define the discrete set of orthonormal polynomials of degree $\bm{n}$ by
\begin{equation}
\Pi_{\bm{n}} \coloneqq
\{ p_{\bm{\alpha}} \coloneqq \bigotimes_{i=1}^d p_{\alpha_i}^i \,|\, 
\bm{\alpha}\in[\bm{n}]
\},
\end{equation}
where $[\bm{n}]$ is defined as in Table \ref{tab:notation}.
For $f$ with FTT rank $\overline{\bm{r}}$, we then may approximate 
\begin{equation}
\label{eq:TT_poly_format}
f(x_1,\ldots,x_d)\approx \sum\limits_{\bm{\alpha}\in[\bm{n}]} \bm{C}[\bm{\alpha}] p_{\bm{\alpha}}(x_1,\ldots,x_d),
\end{equation}
with a tensor array $\bm{C}\in\mathbb{R}^{\bm{n}+1}$ with Tensor Train (TT) rank $\bm{r}=(r_1,\ldots,r_{d-1})^\intercal\in\mathbb{N}^{d-1}$ bounded by the FTT rank $\overline{\bm r}$. In particular we have the decomposition into a Tensor Train (or Matrix Product State) format
\begin{equation}
    \bm{C}[\bm{\alpha}] = C_1[\alpha_1]C_2[\alpha_2]\cdots C_d[\alpha_d],
\end{equation}
with matrices $C_i[\alpha_i]\in\mathbb{R}^{r_{i-1},r_i}$ and the convention that $r_0=r_d=1$. 
Note that the relation of $\overline{\bm{r}}$ and $\bm{r}$ depends on the relation of $F_i$ and the polynomials in $i$-th direction. In particular it holds  $\overline{\bm{r}}=\bm{r}$ if for all $i=1,\ldots,d$ and $\alpha=0,\ldots, n_i$ it holds
$$
\int\limits_{a_i}^{b_i} F_i(x_i) p_{\alpha}^i(x_i) \mathrm{d}x_i \neq \bm{0} \in\mathbb{R}^{\overline{r}_{i-1},\overline{r}_i}.
$$
Provided that the ranks can be bounded, the TT format exhibits a storage complexity of \\ $\mathcal{O}(\max(n_1,\ldots,n_d) d \max(r_1,\dots,r_{d-1})^2)$, which scales only linearly in the dimension $d$, hence avoiding the curse of dimensionality. The set of such Tensor Trains of fixed rank $\bm{r}$ defines a manifold $\mathcal{M}_{\bm{r}} \subset \mathbb{R}^{\bm{n}+1}$, see e.g.
\ifnum\classstyle=2
\cite{holtz2012manifolds}.
\else
\cite{holtz2012manifolds}.
\fi


As a first step of our HJB solver, we propose to approximate $V_0=-\log\pi^\ast$ in a functional Tensor Train format based on orthogonal polynomial space discretization as in \eqref{eq:TT_poly_format} for some TT rank $\bm{r}\in\mathbb{N}^{d-1}$. A motivation for this type of approximation for Bayesian posteriors can be found in 
\ifnum\online=1
the Online Supplementary Material.
\else
Appendix \ref{app:tt_motivation}. 
\fi
In what follows we discuss the actions of the linear and nonlinear operators defined in \eqref{eq:hjb_lin} and \eqref{eq:hjb_nonlin} on functions given in that format. To that end, we define for any tensor $\bm{A}\in\mathbb{R}^{\bm{n}+1}$ the associated polynomial $v_A\in\operatorname{span} \Pi_{\mathbf{n}} $ by
\begin{equation}\label{eq:TT_repr}
    v_{\bm{A}} = \sum\limits_{\bm{\alpha}\in [\bm{n}]} \bm{A}[\bm{\alpha}] p_{\bm{\alpha}}.
\end{equation}

\subsection{The linear part}
\label{sec:linear_part}
This section is concerned with the operator
$\operatorname{Lin}$ from \eqref{eq:hjb_lin}, appearing in the right-hand side of the HJB in \eqref{eq:shifted_hjb}.

Let the differential operator $\mathcal{D} \colon \mathcal{C}^2(\mathbb{R})\to\mathcal{C}(\mathbb{R})$ be defined as $\mathcal{D}v= \partial^2_x v + x\partial_x v$ for $v\in\mathcal{C}^2(\mathbb{R})$ and let $\mathcal{I}\colon\mathcal{C}^2(\mathbb{R})\to\mathcal{C}^2(\mathbb{R})$ denote the identity operator. Then, it holds 
\begin{equation}
    \label{eq:LinOpt_cont}
    \operatorname{Lin} = \mathcal{D} \otimes \mathcal{I} \otimes \ldots \otimes \mathcal{I}  + \mathcal{I}  \otimes \mathcal{D}  \otimes \mathcal{I}  \otimes \ldots \otimes \mathcal{I}  + \ldots + \mathcal{I}   \otimes \ldots \otimes \mathcal{I}  \otimes \mathcal{D} . 
\end{equation}
 As a first result we discuss the effect of the operator on functions $v$ given in FTT format.
\begin{lemma}\label{lem:lin_ftt_rank}
    Let $f\in\mathcal{C}^2(K)$ have FTT-rank $\bm{r}\in\mathbb{N}^{d-1}$. Then, $\operatorname{Lin}(f)$ has FTT-rank at most $2\bm{r}$.
\end{lemma}
\begin{proof}
The assertion follows immediately since $\operatorname{Lin}(f)$ defines a \textit{Laplace-like sum} of FTTs, meaning that each summand only modifies a single component of the FTT. More precisely, we have

\ifnum\classstyle=1
\begin{equation}\label{eq:lin_ftt_bound}
\operatorname{Lin}(f)(x) 
=
    \begin{bmatrix}
        \mathcal{D}F_1(x_1) & F_1(x_1)
    \end{bmatrix}
    \begin{bmatrix}
        F_2(x_2) & 0 \\ \mathcal{D}F_2(x_2) & F_2(x_2)
    \end{bmatrix} \cdots 
    \begin{bmatrix}
        F_{d-1}(x_{d-1}) & 0 \\ \mathcal{D}F_{d-1}(x_{d-1}) & F_{d-1}(x_{d-1})
    \end{bmatrix}  
    \begin{bmatrix}
         F_d(x_d) \\ \mathcal{D}F_d(x_d)
    \end{bmatrix},
\end{equation}
\fi 
\ifnum\classstyle=0
\begin{equation}\label{eq:lin_ftt_bound}
\begin{aligned}
\operatorname{Lin}(f)(x) 
&=
    \begin{bmatrix}
        \mathcal{D}F_1(x_1) & F_1(x_1)
    \end{bmatrix}
    \begin{bmatrix}
        F_2(x_2) & 0 \\ \mathcal{D}F_2(x_2) & F_2(x_2)
    \end{bmatrix} \cdots \\
    &\cdots
    \begin{bmatrix}
        F_{d-1}(x_{d-1}) & 0 \\ \mathcal{D}F_{d-1}(x_{d-1}) & F_{d-1}(x_{d-1})
    \end{bmatrix}  
    \begin{bmatrix}
         F_d(x_d) \\ \mathcal{D}F_d(x_d)
    \end{bmatrix},
\end{aligned}
\end{equation}
\fi 
\ifnum\classstyle=2
\begin{equation}\label{eq:lin_ftt_bound}
\begin{aligned}
\operatorname{Lin}(f)(x) 
&=
    \begin{bmatrix}
        \mathcal{D}F_1(x_1) & F_1(x_1)
    \end{bmatrix}
    \begin{bmatrix}
        F_2(x_2) & 0 \\ \mathcal{D}F_2(x_2) & F_2(x_2)
    \end{bmatrix} \cdots \\
    &\cdots
    \begin{bmatrix}
        F_{d-1}(x_{d-1}) & 0 \\ \mathcal{D}F_{d-1}(x_{d-1}) & F_{d-1}(x_{d-1})
    \end{bmatrix}  
    \begin{bmatrix}
         F_d(x_d) \\ \mathcal{D}F_d(x_d)
    \end{bmatrix},
\end{aligned}
\end{equation}
\fi 
\ifnum\classstyle=3
\begin{equation}\label{eq:lin_ftt_bound}
\operatorname{Lin}(f)(x) 
=
    \begin{bmatrix}
        \mathcal{D}F_1(x_1) & F_1(x_1)
    \end{bmatrix}
    \begin{bmatrix}
        F_2(x_2) & 0 \\ \mathcal{D}F_2(x_2) & F_2(x_2)
    \end{bmatrix} \cdots 
    \begin{bmatrix}
        F_{d-1}(x_{d-1}) & 0 \\ \mathcal{D}F_{d-1}(x_{d-1}) & F_{d-1}(x_{d-1})
    \end{bmatrix}  
    \begin{bmatrix}
         F_d(x_d) \\ \mathcal{D}F_d(x_d)
    \end{bmatrix},
\end{equation}
\fi 
which defines a product of matrix valued functions as in \eqref{eq:ftt}. The rank bound follows immediately from the block structure of \eqref{eq:lin_ftt_bound} and the dimensions of $F_i,\mathcal{D}F_i$ for $i=1,\ldots,d$.
\end{proof}
When applied to polynomials, the linear operator can be expressed in terms of its action on the polynomial's coefficients. More precisely, the discretization of $\operatorname{Lin}$ on the finite set $\Pi_{\bm{n}}$ for some $\bm{n}=(n_1,\ldots,n_d)^\intercal\in\mathbb{N}_0^d$ implies a linear operator $\bm{L}\colon \mathbb{R}^{\bm{n}+1}\to\mathbb{R}^{\bm{n}+1}$ given as
\begin{equation}
    \label{eq:LinOpt_discrete}
    \bm{L}:= \sum\limits_{i=1}^d \bm{L}_i, \quad \bm{L}_i:= \left(\bigotimes_{j=1}^{i-1} I_{n_j+1}\right) \otimes D_{i} \otimes \left(\bigotimes_{j=i+1}^{d} I_{n_j+1}\right),
\end{equation}
with identity matrix $I_n\in\mathbb{R}^{n, n}$. For the structure of the matrix $D_{i}\in\mathbb{R}^{n_i+1,n_i+1}$ we refer to
\ifnum\online=1
the Online Supplementary Material.
\else
Appendix \ref{sec:LinOp}, specifically equation \eqref{eq:app_di}.\fi
For the moment it suffices to note that $D_i$ governs the action of the differential operator $\mathcal{D}$ on the coefficients of the polynomials in dimension $i$. It can be shown that the action of $\operatorname{Lin}$ on a polynomial corresponds to algebraic manipulation of the coefficient tensor with respect to $\bm{L}$, which is the result of the following lemma.
\begin{lemma}
    Let $\bm{n}\in\mathbb{N}_0^d$, $\operatorname{Lin}$ and $\bm{L}$ from \eqref{eq:LinOpt_cont} and \eqref{eq:LinOpt_discrete}. Then, for $\bm{A}\in\mathbb{R}^{\bm{n}+1}$ we have
    \begin{equation}
    \operatorname{Lin}v_{\bm{A}} = v_{\bm{L}\bm{A}}.
    \end{equation}
\end{lemma}
\begin{proof}
Let $\bm{L}_i[\bm{\beta}, \bm{\alpha}] :=   \left(\bigotimes_{j=1}^{i-1} I_{n_j+1}[\beta_j,\alpha_j]\right) \otimes D_{n_i}[\beta_i,\alpha_i] \otimes \left(\bigotimes_{j=i+1}^{d} I_{n_j+1}[\beta_j,\alpha_j]\right)$. Then, the action of $\bm{L}_i$ on $\bm{A}$ defines a tensor $\bm{B}_i$ given as 
$
\bm{B}_i[\bm{\beta}] = \sum_{\bm{\alpha}\in[\bm{n}]} \bm{L}_i[\bm{\beta}, \bm{\alpha}] \bm{A}[\bm{\alpha}].
$
Moreover, $\bm{L}\bm{A}=\sum\limits_{i=1}^d \bm{B}_i$.
Hence, 
\begin{align*}
     \operatorname{Lin} v_{\bm{A}} 
     &= 
     \sum\limits_{i=1}^d\sum_{\bm{\alpha}\in[\bm{n}]}
     \left(\bigotimes_{j=1}^{i-1} \mathcal{I}\right) \otimes \mathcal{D} \otimes \left(\bigotimes_{j=i+1}^{d} \mathcal{I}\right) \bm{A}[\bm{\alpha}]  p_{\alpha_1}^1\otimes \cdots \otimes p_{\alpha_d}^d \\
     &= 
     \sum\limits_{i=1}^d
\sum_{\bm{\beta}\in[\bm{n}]}
\sum_{\bm{\alpha}\in[\bm{n}]}
     \bm{L}_i[\bm{\beta}, \bm{\alpha} ] \bm{A}[\bm{\alpha}]  p_{\alpha_1}^1\otimes \cdots \otimes p_{\alpha_d}^d \\
&= \sum\limits_{i=1}^d
\sum_{\beta\in[\bm{n}]} \bm{B}_i[\bm{\beta}]   p_{\beta_1}^1\otimes \cdots \otimes p_{\beta_d}^d \\
&= 
v_{\sum\limits_{i=1}^d \bm{B}_i}
\end{align*}
\end{proof}
The contraction $\bm{L}\bm{A}$ is cumbersome for full tensors $\bm{A}$. However, 
it is easy to implement if $\bm{A}\in\mathcal{M}_{\bm{r}}$ is a Tensor Train of fixed rank $\bm{r}\in\mathbb{N}^{d-1}$, such that $\bm{A}[\alpha]=A_1[\alpha_1]A_2[\alpha_2]\cdots A_d[\alpha_d]$ with $A_i[\alpha_i]\in\mathbb{R}^{r_{i-1},r_i}$ for $i=1,\ldots,d$. In this case, let $D_{i,A_i}[\beta_i] := \sum\limits_{\alpha_i=0}^{n_i} D_{i}[\beta_i,\alpha_i]A_i[\alpha_i]$ for $i=1,\ldots,d$. Then, 
\ifnum\classstyle=1
\begin{equation*}
\begin{aligned}
    \left(\bm{L}\bm{A}\right)[\beta] 
    &= 
    \sum\limits_{i=1}^d\sum_{\alpha\in[\bm{n}]}\bm{L}_i[\beta,\alpha]A[\alpha] \\
    &= 
     \sum\limits_{i=1}^d\sum_{\alpha\in[\bm{n}]}\left(\bigotimes_{j=1}^{i-1} I_{n_j+1}[\beta_j,\alpha_j]A_j[\alpha_j]\right) \otimes D_{i}[\beta_i,\alpha_i]A_i[\alpha_i] \otimes \left(\bigotimes_{j=i+1}^{d} I_{n_j+1}[\beta_j,\alpha_j]A_j[\alpha_j]\right) \\
     &= 
     \sum\limits_{i=1}^d \left(\bigotimes_{j=1}^{i-1} A_j[\beta_j]\right) \otimes D_{i,A_i}[\beta_i] \otimes \left(\bigotimes_{j=i+1}^{d} A_j[\beta_j]\right) \\
     &= \begin{bmatrix}
        D_{1,A_1}[\beta_1]^{\intercal} & A_1[\beta_1]^{\intercal}
    \end{bmatrix}
    \begin{bmatrix}
        A_2[\beta_2]^{\intercal} & 0 \\ D_{2,A_2}[\beta_2]^{\intercal} & A_2[\beta_2]^{\intercal}
    \end{bmatrix} \cdots 
    \begin{bmatrix}
        A_{d-1}[\beta_{d-1}]^{\intercal} & 0 \\ D_{d-1,A_{d-1}}[\beta_{d-1}]^{\intercal} & A_{d-1}[\beta_{d-1}]^{\intercal}
    \end{bmatrix}  
    \begin{bmatrix}
         A_{d}[\beta_{d}] \\ D_{d,A_d}[\beta_d]
    \end{bmatrix}.
\end{aligned}
\end{equation*}
\fi
\ifnum\classstyle=2
\begin{equation*}
\begin{aligned}
    \left(\bm{L}\bm{A}\right)[\beta] 
    &= 
    \sum\limits_{i=1}^d\sum_{\alpha\in[\bm{n}]}\bm{L}_i[\beta,\alpha]A[\alpha] \\
     &= 
     \sum\limits_{i=1}^d \left(\bigotimes_{j=1}^{i-1} A_j[\beta_j]\right) \otimes D_{i,A_i}[\beta_i] \otimes \left(\bigotimes_{j=i+1}^{d} A_j[\beta_j]\right) \\
     &= \begin{bmatrix}
        D_{1,A_1}[\beta_1]^{\intercal} & A_1[\beta_1]^{\intercal}
    \end{bmatrix}
    \begin{bmatrix}
        A_2[\beta_2]^{\intercal} & 0 \\ D_{2,A_2}[\beta_2]^{\intercal} & A_2[\beta_2]^{\intercal}
    \end{bmatrix} \cdots \\
    &\cdots
    \begin{bmatrix}
        A_{d-1}[\beta_{d-1}]^{\intercal} & 0 \\ D_{d-1,A_{d-1}}[\beta_{d-1}]^{\intercal} & A_{d-1}[\beta_{d-1}]^{\intercal}
    \end{bmatrix}  
    \begin{bmatrix}
         A_{d}[\beta_{d}] \\ D_{d,A_d}[\beta_d]
    \end{bmatrix}.
\end{aligned}
\end{equation*}
\fi
\ifnum\classstyle=0
\begin{equation*}
\begin{aligned}
    \left(\bm{L}\bm{A}\right)[\beta] 
    &= 
    \sum\limits_{i=1}^d\sum_{\alpha\in[\bm{n}]}\bm{L}_i[\beta,\alpha]A[\alpha] \\
     &= 
     \sum\limits_{i=1}^d \left(\bigotimes_{j=1}^{i-1} A_j[\beta_j]\right) \otimes D_{i,A_i}[\beta_i] \otimes \left(\bigotimes_{j=i+1}^{d} A_j[\beta_j]\right) \\
     &= \begin{bmatrix}
        D_{1,A_1}[\beta_1]^{\intercal} & A_1[\beta_1]^{\intercal}
    \end{bmatrix}
    \begin{bmatrix}
        A_2[\beta_2]^{\intercal} & 0 \\ D_{2,A_2}[\beta_2]^{\intercal} & A_2[\beta_2]^{\intercal}
    \end{bmatrix} \cdots \\
    &\cdots
    \begin{bmatrix}
        A_{d-1}[\beta_{d-1}]^{\intercal} & 0 \\ D_{d-1,A_{d-1}}[\beta_{d-1}]^{\intercal} & A_{d-1}[\beta_{d-1}]^{\intercal}
    \end{bmatrix}  
    \begin{bmatrix}
         A_{d}[\beta_{d}] \\ D_{d,A_d}[\beta_d]
    \end{bmatrix}.
\end{aligned}
\end{equation*}
\fi
Hence, $\bm{L}\bm{A}$ is given in TT format through a \textit{Laplace-like sum} 
with TT rank bounded by $2\bm{r}$, which is consistent with Lemma \ref{lem:lin_ftt_rank}. In particular, this formula involves only contractions of the matrices $D_i\in\mathbb{R}^{n_i+1,n_i+1}$ with the order three tensors $A_i \in \mathbb{R}^{r_{i-1},n_i+1,r_i}$ for $i=1,\ldots,d$. No contraction with the full tensor $\bm{A}$ is required.

\subsection{The nonlinear part}
This section is concerned with the operator
$\operatorname{NonLin}(\,\cdot\,)=\|\nabla \,\cdot\,\|^2$ from  \eqref{eq:hjb_nonlin}, appearing in the right-hand side of the HJB equation in \eqref{eq:shifted_hjb}, in case the arguments are functions given in FTT format. This operator is a combination of partial derivatives, squares and a summation. We split the results into two Lemmas. First, we derive a more general bound on the FTT-rank of a product of functions with bounded FTT-rank.

\begin{lemma}
\label{lemma:fsquare_ttrank}
    Let $g$, $f$ have FTT-rank $\bm{r}^f$ and $\bm{r}^g$, respectively. Then $g\cdot f$ has FTT-rank at most $\bm{r}^g\bm{r}^f$.
\end{lemma}
\begin{proof}
    We write $f(x) = F_1(x_1)\cdot \ldots \cdot F_d(x_d)$ and $g(x) = G_1(x_1)\cdot \ldots \cdot G_d(x_d)$ with $F_i(x_i)\in\mathbb{R}^{r^f_{i-1},r^f_i}$, $G_i(x_i)\in\mathbb{R}^{r^g_{i-1},r^g_i}$ for $i=1,\ldots,d$. Let $\otimes_k$ denote the standard matrix-Kronecker product with the convention that for two scalar values $a,b\in\mathbb{R}$ we set $a\otimes_k b = a\cdot b$. Then, we have
    \begin{equation*}
        f(x)g(x) = F_1(x_1)\cdot \ldots \cdot F_d(x_d)\cdot G_1(x_1)\cdot \ldots \cdot G_d(x_d) = F_1(x_1)\otimes_k G_1(x_1)\cdot \ldots \cdot F_d(x_d)\otimes_k G_d(x_d),
    \end{equation*}
    where $F_i(x_i)\otimes_k G_i(x_i)\in\mathbb{R}^{r^f_{i-1}r^g_{i-1},r^f_{i}r^g_{i}}$.
\end{proof}
Second, the rank bound of the nonlinear right hand side is provided.
\begin{lemma}
\label{lemma:f_partial_sum_ttrank}
    Let $f$ have FTT-rank $\bm{r}$, then $\operatorname{NonLin}(f) = \|\nabla f\|^2 = \sum\limits_{i=1}^d  \left( \frac{\partial f}{\partial x_i}\right)^2$ has FTT-rank at most $2\bm{r}^2$.
\end{lemma}
\begin{proof}
Note that $\frac{\partial f}{\partial x_i} = F_1(x_1)\ldots \partial_{x_i}F_i(x_i)\ldots F_d(x_d)$ has FTT-rank $\leq \bm{r}$ for all $i$. By Lemma \ref{lemma:fsquare_ttrank}, $(\frac{\partial f}{\partial x_i})^2$ has FTT-rank at most $\bm{r}^2$. To bound the FTT-rank of $\sum_{i=1}^d  \left( \frac{\partial f}{\partial x_i}\right)^2$, the derivation is the same as in the proof of Lemma \ref{lem:lin_ftt_rank}, only that the operator $\mathcal{D}$ is replaced by an operator mapping $\mathcal{C}^1(\mathbb{R})\rightarrow \mathcal{C}(\mathbb{R})$ and $v\mapsto (\partial_x v)^2$.
\end{proof}
We now turn our view on the discretization of the corresponding operator with respect to $\Pi_{\bm{n}}$. 

\subsubsection{The operator in Tensor Train format}
For a practical algorithm, we need a discretization of $\operatorname{NonLin}$ on the finite set $\Pi_{\mathbf{n}}$ such as \eqref{eq:LinOpt_discrete} for the linear part. Here, we refrain from deriving a formula in the general setting of a full coefficient tensor and directly examine the case of a TT with fixed rank. We consider the square operation first. Let $\bm{n}\in\mathbb{N}^d$ and the \textit{multiplication} operation $M_g \colon f\mapsto g\cdot f$, where $g, f$ are given by
\begin{align}
\label{eq:polyform_f_g}
    f(x) = v_{\bm{A}}(x) =  \sum\limits_{\bm{\alpha}} \bm{A}[\bm{\alpha}] \prod\limits_{i=1}^d p_{\alpha_i}^i (x_i), \quad 
    g(x) = v_{\bm{B}}(x) =  \sum\limits_{\bm{\beta}\in[\bm{n}]} \bm{B}[\bm{\beta}] \prod\limits_{j=1}^d p_{\beta_j}^j (x_j)
\end{align} 
with tensors $\bm{A}, \bm{B}\in \mathbb{R}^{\bm{n}+1}$ both given in Tensor Train format with TT-rank $\bm{r}=(r_1,\ldots,r_{d-1})\in\mathbb{N}^{d-1}$ and
\begin{equation}
   \bm{A}[\bm{\alpha}] = A_1[\alpha_1]\cdots A_d[\alpha_d], \quad 
\bm{B}[\bm{\beta}] = B_1[\beta_1]\cdots B_d[\beta_d]. 
\end{equation}
We aim to define a Tensor Train operator $\bm{M}_{\bm{B}} \colon \mathbb{R}^{\bm{n}+1} \rightarrow \mathbb{R}^{2\bm{n}+1}$ such that
\begin{align}
    M_g(f) = v_{\bm{M}_{\bm{B}}(\bm{A})} = \sum\limits_{\bm{\gamma}\in[2\bm{n}]} \bm{M}_{\bm{B}}(\bm{A})[\bm{\gamma}] \prod\limits_{i=1}^d p_{\gamma_i}^i (x_i).
\end{align}
For $n\in\mathbb{N}_0$ let $T_{i,n}\in\mathbb{R}^{n+1,n+1}$ denote the transformation matrix mapping the coefficients of Legendre polynomials up to degree $n$ on $[a_i,b_i]$ to the corresponding coefficients of standard monomials $1,x,x^2,\ldots$ up to degree $n$. Let 
\begin{align}
 \hat{\bm{A}}[\bm{\alpha}] &:= 
                 \hat{A}_1[\alpha_1]\cdots \hat{A}_d[\alpha_d], \quad
 \hat{A}_i[\alpha_i] = \sum\limits_{\alpha_i'=0}^{n_i} T_{i,n_i}[\alpha_i,\alpha_i']A_i[\alpha_i'], \\
\hat{\bm{B}}[\bm{j}] &:=  
                   \hat{B}_1[\beta_1]\cdots \hat{B}_d[\beta_d],\quad 
 \hat{B}_i[\beta_i]:=\sum\limits_{\beta_i'=0}^{n_i} T_{i,n_i}[\beta_i,\beta_i']B_i[\beta_i'],
\end{align}
define the coefficient tensors of $f$ and $g$ with respect to monomials. Now, for $i=1\ldots,d$ and $\alpha_i=0,\ldots,n_i$ define the matrix $D_{i,\alpha_i}$ by

\begin{equation}
    D_{i,\alpha_i}:=\begin{bNiceArray}{cccccc}[margin]
     \Block{1-6}{\mathbf{0}_{\alpha_i,n_i+1}} & &&&& \\
    \hline
    \Block{3-6}{I_{n_i+1,n_i+1}} & &&&& \\
    &&&&& \\
    &&&&& \\ 
    \hline 
    \Block{2-6}{\mathbf{0}_{n_i+1-\alpha_i,n_i+1}}& &&&&   \\
    &&&&&
    \end{bNiceArray} \in \mathbb{R}^{2n_i+1, n_i+1},
\end{equation}
where $\bm{0}_{m,n} \in\mathbb{R}^{m,n}$ is a matrix with all entries equal to $0$, which we define to be empty if $m$ or $n$ equal zero.
Furthermore, for $i=1,\ldots, d$, $k_{i-1},\ell_{i-1}\in\{1,\ldots,r_{i-1}\}$, $k_{i},\ell_i\in\{1,\ldots, r_i\}$ we define the vector $ \hat{\mathrm{C}}_i[k_{i-1}, \ell_{i-1}, k_i, \ell_i] \in\mathbb{R}^{2 n_i +1}$ 
as
\begin{equation}
\hat{\mathrm{C}}_i[k_{i-1}, \ell_{i-1}, k_i, \ell_i] = 
\sum\limits_{\alpha_i=0}^{n_i}
D_{i, \alpha_i} \hat{B}_i[k_{i-1},\,:\,, k_i]
\hat{A}_i[\ell_{i-1}, \alpha_i, \ell_i].
\end{equation}
Note that $D_{i,\alpha_i}\hat{B}_i[k_{i-1}, \,:\, , k_i]  $ denotes a matrix-vector multiplication, whereas $\hat{A}_i[\ell_{i-1}, \alpha_i, \ell_i]$ is scalar valued. 

With slight abuse of notation, we denote the $\gamma_i$-th entry of the vector $\hat{\mathrm{C}}_i[k_{i-1}, \ell_{i-1}, k_i, \ell_i] $ \ifnum\classstyle=0 \\ \fi
\ifnum\classstyle=3 \\ \fi by $
\hat{\mathrm{C}}_i[k_{i-1}, \ell_{i-1}, \gamma_i, k_i, \ell_i]\in\mathbb{R} $, which defines an order $5$ tensor $\hat{C}_{i} \in\mathbb{R}^{r_{i-1}, r_{i-1}, 2n_i+1, r_i, r_i}$. For convenience, we reshape $\hat{C}_i$ to an order $3$ tensor by flattening together the first two and last two dimensions, again overloading notation with $\hat{C}_i\in\mathbb{R}^{r_{i-1}^2, 2n_i+1, r_i^2}$. Now we revert to the Legendre polynomial system and define
the coefficient tensor $\bm{C}\in\mathbb{R}^{2\bm{n}+1}$ given in TT format by
\begin{equation}
\label{eq:fg_hadamard_legendrepoly_tt_tensor}
     \bm{C}[\bm{\gamma}] := C_1[\gamma_1]\cdots C_d[\gamma_d], 
     \quad 
     C_i[\gamma_i] = \sum\limits_{\gamma_i'=0}^{2n_i} T_{i,2n_i}^{-1}[\gamma_i,\gamma_i']\hat{C}_i[\gamma_i'].
\end{equation}
This construction yields the following result.
\begin{lemma}
\label{lemma:fg_hadamard}
    Let $f$ and $g$ have FTT-rank $\bm{r}$ and given as in \eqref{eq:polyform_f_g}. Then, $fg$ has FTT-rank at most $\bm{r}^2$, in particular 
    \begin{equation}
    g(x)f(x) = 
    \sum\limits_{\bm{\gamma}\in[2\bm{n}]}  C[\bm{\gamma}] \prod\limits_{i=1}^d p_{\gamma_i}^i(x_i), 
\end{equation}
with coefficient tensor $\bm{C}$ with TT-rank at most $\bm{r}^2$ given by \eqref{eq:fg_hadamard_legendrepoly_tt_tensor}.
\end{lemma}
By this Lemma, we have 
\begin{equation}
    \bm{M}_{\bm{B}}(\bm{A})= \bm{C} 
\end{equation}
with $\bm{C}$ from \eqref{eq:fg_hadamard_legendrepoly_tt_tensor}. For ease of notation, we further define the \textit{square} operation $\bm{S}\colon \mathbb{R}^{\bm{n}+1}\rightarrow \mathbb{R}^{2\bm{n}+1}, \bm{A}\mapsto \bm{M}_{\bm{A}}(\bm{A})$.

Finally, note that the partial derivative $\partial_{x_i}$ defines a linear operator that, analogous to Section \ref{sec:linear_part}, implies a linear operator $\bm{L}_{x_i}\colon\mathbb{R}^{\bm{n}+1}\to\mathbb{R}^{\bm{n}+1}$ based on the polynomial discretization such that $\partial_{x_i} v_{\bm{A}}= v_{\bm{L}_{x_i}\bm{A}}$. This operator has the form $\bm{L}_{x_i} = \mathcal{I}\otimes \ldots \otimes \mathcal{I} \otimes D_{x_i}\otimes \mathcal{I}\otimes \ldots \otimes \mathcal{I}$ with $D_{x_i}\in\mathbb{R}^{n_i+1,n_i+1}$ given in 
\ifnum\online=1
the Online Supplementary Material.
\else
Appendix \ref{app:nonlinear_part}.\fi 
Putting all of the above together, we see that 
\ifnum\classstyle=1
\begin{equation}
\label{eq:non_linear_general_algebraic_formula}
    \langle \nabla v_{\bm{B}}, \nabla v_{\bm{A}} \rangle = \sum\limits_{i=1}^d (\partial_{x_i} v_{\bm{A}}) (\partial_{x_i} v_{\bm{B}}) 
    =  \sum\limits_{i=1}^d  v_{\bm{L}_{x_i}\bm{A}} v_{\bm{L}_{x_i}\bm{B}}
    =  \sum\limits_{i=1}^d
    v_{\bm{M}_{{\bm{L}_{x_i}\bm{B}}}(\bm{L}_{x_i}\bm{A})}
    =  v_{\sum\limits_{i=1}^d\bm{M}_{{\bm{L}_{x_i}\bm{B}}}(\bm{L}_{x_i}\bm{A})} 
\end{equation}
\fi
\ifnum\classstyle=2
\begin{equation}
\begin{aligned}
    \label{eq:non_linear_general_algebraic_formula}
    \langle \nabla v_{\bm{B}}, \nabla v_{\bm{A}} \rangle &= \sum\limits_{i=1}^d (\partial_{x_i} v_{\bm{A}}) (\partial_{x_i} v_{\bm{B}}) 
    =  \sum\limits_{i=1}^d  v_{\bm{L}_{x_i}\bm{A}} v_{\bm{L}_{x_i}\bm{B}}\\
    &=  \sum\limits_{i=1}^d
    v_{\bm{M}_{{\bm{L}_{x_i}\bm{B}}}(\bm{L}_{x_i}\bm{A})}
    =  v_{\sum\limits_{i=1}^d\bm{M}_{{\bm{L}_{x_i}\bm{B}}}(\bm{L}_{x_i}\bm{A})} 
\end{aligned}
\end{equation}
\fi
\ifnum\classstyle=0
\begin{equation}
\begin{aligned}
    \label{eq:non_linear_general_algebraic_formula}
    \langle \nabla v_{\bm{B}}, \nabla v_{\bm{A}} \rangle &= \sum\limits_{i=1}^d (\partial_{x_i} v_{\bm{A}}) (\partial_{x_i} v_{\bm{B}}) 
    =  \sum\limits_{i=1}^d  v_{\bm{L}_{x_i}\bm{A}} v_{\bm{L}_{x_i}\bm{B}}\\
    &=  \sum\limits_{i=1}^d
    v_{\bm{M}_{{\bm{L}_{x_i}\bm{B}}}(\bm{L}_{x_i}\bm{A})}
    =  v_{\sum\limits_{i=1}^d\bm{M}_{{\bm{L}_{x_i}\bm{B}}}(\bm{L}_{x_i}\bm{A})} 
\end{aligned}
\end{equation}
\fi
\ifnum\classstyle=3
\begin{equation}
\label{eq:non_linear_general_algebraic_formula}
    \langle \nabla v_{\bm{B}}, \nabla v_{\bm{A}} \rangle = \sum\limits_{i=1}^d (\partial_{x_i} v_{\bm{A}}) (\partial_{x_i} v_{\bm{B}}) 
    =  \sum\limits_{i=1}^d  v_{\bm{L}_{x_i}\bm{A}} v_{\bm{L}_{x_i}\bm{B}}
    =  \sum\limits_{i=1}^d
    v_{\bm{M}_{{\bm{L}_{x_i}\bm{B}}}(\bm{L}_{x_i}\bm{A})}
    =  v_{\sum\limits_{i=1}^d\bm{M}_{{\bm{L}_{x_i}\bm{B}}}(\bm{L}_{x_i}\bm{A})} 
\end{equation}
\fi
This leads to a Tensor Train operator representing the nonlinear part \eqref{eq:hjb_nonlin}. In particular, for $\bm{A},\bm{B}\in\mathbb{R}^{\bm{n}+1}$ let
\begin{align}
    {\bm{N}\!\!\bm{L}}(\bm{A}) &\coloneqq -\sum_{i=1}^d \bm{S}(\bm{L}_{x_i}\bm{A}),\label{eq:nonlinear_operator}\\
    {\bm{N}\!\!\bm{L}}_{\bm{B}}(\bm{A}) &\coloneqq -\sum_{i=1}^d \bm{M}_{\bm{L}_{x_i}\bm{B}}(\bm{L}_{x_i}\bm{A}).\label{eq:linearized_nonlinear_operator}
\end{align}
Then, by \eqref{eq:non_linear_general_algebraic_formula}, we have $\operatorname{NonLin}(v_{\bm{A}}) = v_{{\bm{N}\!\!\bm{L}}(\bm{A})} $ and $-\langle \nabla v_{\bm{B}}, \nabla v_{\bm{A}} \rangle = v_{{\bm{N}\!\!\bm{L}}_{\bm{B}}(\bm{A})}$. This concludes the derivation of the nonlinear part.

\subsection{Projection and retraction}
The discussion so far shows that linear and nonlinear operations on the polynomial discretization with Tensor Trains may increase the rank as well as the underlying polynomial degree. Therefore, we shall discuss operations that keep a fixed polynomial degree and a fixed TT-rank with possible error control, namely \textit{projection} and \textit{retraction}. Regarding the projection, let $\bm{n},\bm{m}\in\mathbb{N}_0$, $\bm{n}\leq \bm{m}$ and define $\mathcal{P}_{\bm{m},\bm{n}} \colon \operatorname{span}\Pi_{\bm{m}}\to \operatorname{span}\Pi_{\bm{n}}$ by
\begin{align}\label{eq:projection}
    \mathcal{P}_{\bm{m},\bm{n}}(\cdot)  \coloneqq \sum_{\alpha_1,\ldots,\alpha_d=0}^{n_1,\ldots, n_d} \bigotimes_{i=1}^d p^i_{\alpha_{i}} 
    \left\langle \bigotimes_{i=1}^d p^i_{\alpha_{i}}, \,\cdot\, \right\rangle
\end{align}
Due to the orthonormality of the $p^i_{\alpha_i}$, the projection is simply obtained by truncating the coefficients, as the following result states.
\begin{lemma}\label{lem:projection}
    For $\bm{n}\leq\bm{m}$ and $\bm{A}\in\mathbb{R}^{\bm{m}+1}$ we have $\mathcal{P}_{\bm{m},\bm{n}}V_{\bm{A}} = V_{\bm{P}_{\bm{m},\bm{n}}\bm{A}}$, where $\bm{P}_{\bm{m},\bm{n}}\colon \mathbb{R}^{\bm{m}+1}\rightarrow\mathbb{R}^{\bm{n}+1}$ is defined by
    \begin{equation}
    \label{eq:discrete_projection}
        (\bm{P}_{\bm{m},\bm{n}}\bm{A})[\alpha_1,\ldots,\alpha_d] = \bm{A}[\alpha_1,\ldots,\alpha_d]
    \end{equation}
    for all $\bm{A}\in\mathbb{R}^{\bm{m}+1}$ and $\bm{\alpha}\in\mathbb{N}^{\bm{n}}_0$.
\end{lemma}
Note that by Parseval's identity, the projection error in $L^2$-norm can be computed by simply adding the squares of the elements that are eliminated by the projection, i.e. with assumptions of Lemma \ref{lem:projection}, we have
    \begin{align}
        \|\mathcal{P}_{\bm{m},\bm{n}}v_{\bm{A}} -v_{\bm{A}} \|^2_{L^2(K)} = \sum_{\alpha_1=n_1+1,\ldots,\alpha_d=n_d+1}^{m_1,\ldots,m_d} \bm{A}[\alpha_1,\ldots,\alpha_d]^2.
    \end{align}
A possible realization of a \textit{retraction operator}
\begin{equation}
\label{eq:retraction_op}
    \bm{R}_{\bm{r}}\colon \bigcup_{\hat{\bm{r}}\geq \bm{r}} \mathcal{M}_{\hat{\bm{r}}}\to \mathcal{M}_{\bm{r}},
\end{equation} for given fixed rank $\bm{r}\in\mathbb{N}^{d-1}$, is obtained by using the \textit{TT rounding} scheme first presented in \cite[Algorithm 2]{oseledets2011tensor}, which is based on efficient high-order singular value decomposition exploiting the structure of TTs. The operators in \eqref{eq:discrete_projection} and \eqref{eq:retraction_op} provide us with the necessary tensor operations to fix the degree as well as the rank of the HJB solution, concluding this section.





\section{A direct low-rank HJB solver}\label{sec:solver}




In this section, we consider polynomial potentials $\Phi\in\operatorname{span}\Pi_{\bm{n}}$ for some $\bm{n}\in\mathbb{N}_0^d$. If the potential is not available in polynomial form, we can obtain a suitable polynomial approximation e.g. by the \textit{Alternating Linear Scheme} (ALS) 
\ifnum\classstyle=2
\cite{holtz2012alternating} 
\else
\cite{holtz2012alternating}
\fi
as was done in 
\ifnum\classstyle=2\cite{oster2022approximating} 
\else
\cite{oster2022approximating}
\fi
for the purpose of approximating value functions. Crucially, the ALS yields an approximation in a chosen low rank TT format.
For $\Phi\in\operatorname{span}\Pi_{\bm{n}}$, we consider a projected version of the modified HJB equation \eqref{eq:shifted_hjb} restricted to the hypercube $K$ defined by
\begin{align}\label{eq:projected_hjb}
   \left\{
    \begin{array}{rcl}
 \partial_t v_t &=& \mathcal{P}_{2\bm{n},\bm{n}}\left[ \operatorname{Lin}(v_t) + \operatorname{NonLin}(v_t) \right], \\
 v_0 &=& \Phi,
    \end{array}
   \right.
   \quad \text{in } K,
\end{align}
for $t\in[0,T]$ and some $T>0$ large enough. Note, that the projection only acts on the nonlinear part, as the linear part does not increase the polynomial degree.

With the work from the previous section, we can show that this PDE is equivalent to an ODE on a tensor space.
 Let $\bm{L}$, ${\bm{N}\!\!\bm{L}}$, ${\bm{N}\!\!\bm{L}}_{\bm{B}}$ for any $\bm{B}\in\mathbb{R}^{\bm{n}+1}$ and $\bm{P} \equiv \bm{P}_{2\bm{n}+1,\bm{n}+1}$ be given by \eqref{eq:LinOpt_discrete}, \eqref{eq:nonlinear_operator}, \eqref{eq:linearized_nonlinear_operator} and \eqref{eq:discrete_projection}, respectively. Then the following theorem holds true.
 

\begin{theorem}[Projected HJB equation is equivalent to tensor-valued ODE]\label{thm:tensor_hjb}
    Let $\bm{A}(t)\in\mathbb{R}^{\mathbf{n}+1}$ be a solution of the tensor-valued ODE
    \begin{equation}\label{eq:tensor_hjb}
       \bm{{\dot{A}}}(t) = \bm{L}\bm{A}(t) + \bm{P}{\bm{N}\!\!\bm{L}}(\bm{A}(t)), \qquad \bm{A}(0) = \bm{A}_0,
    \end{equation}
    for $t\in[0,T]$. Then $v_t \coloneqq v_{\bm{A}(t)}$ solves \eqref{eq:projected_hjb}. Conversely, if $v_t\in\operatorname{span}\Pi_{\bm{n}}$ solves \eqref{eq:projected_hjb}, then there exists a unique $\bm{A}(t)\in\mathbb{R}^{\mathbf{n}+1}$ such that $v_t = v_{\bm{A}(t)}$ and $\bm{A}(t)$ solves \eqref{eq:tensor_hjb}.
\end{theorem}
\begin{proof}
Let $\bm{A}(t)\in\mathbb{R}^{\mathbf{n}+1}$ solve \eqref{eq:tensor_hjb}. Then, 
\ifnum\classstyle=3
\begin{equation}
\begin{aligned}
    \dot{v}_{\bm{A}(t)} &= v_{\bm{{\dot{A}}}(t)} = v_{\bm{L}\bm{A}(t) + \bm{P}{\bm{N}\!\!\bm{L}}(A(t))} = v_{\bm{L}\bm{A}(t)} + v_{\bm{P}{\bm{N}\!\!\bm{L}}(\bm{A}(t))} \\
    &= \operatorname{Lin}(v_{\bm{A}(t)}) + \mathcal{P}_{2\bm{n}+1,\bm{n}+1}\left[ \operatorname{NonLin}(v_{\bm{A}(t)}) \right]
\end{aligned}
\end{equation}
\else
$\dot{v}_{\bm{A}(t)} = v_{\bm{{\dot{A}}}(t)} = v_{\bm{L}\bm{A}(t) + \bm{P}{\bm{N}\!\!\bm{L}}(A(t))} = v_{\bm{L}\bm{A}(t)} + v_{\bm{P}{\bm{N}\!\!\bm{L}}(\bm{A}(t))} = \operatorname{Lin}(v_{\bm{A}(t)}) + \mathcal{P}_{2\bm{n}+1,\bm{n}+1}\left[ \operatorname{NonLin}(v_{\bm{A}(t)}) \right]$ 
\fi
and $v_{\bm{A}(0)} = v_{\bm{A}_0} = \Phi$, showing the first part of the claim. Conversely, if $v_t\in\operatorname{span}\Pi_{\bm{n}}$ solves \eqref{eq:projected_hjb}, then there exists a unique $\bm{A}(t)$ with $v_t = v_{\bm{A}(t)}$ and $v_{\bm{{\dot{A}}}(t)} = \partial_t v_t = v_{\bm{L}\bm{A}(t) + \bm{P}{\bm{N}\!\!\bm{L}}(\bm{A}(t))}$. Since the mapping $\bm{A}\mapsto v_{\bm{A}}$ is injective, this yields the second part of the claim.
\end{proof}
The solution algorithm for \eqref{eq:tensor_hjb} which will be presented in the following relies on local linearizations of the HJB for stiffnes based stepsize control. Hence, we state the following result on the form of such local linearizations.

\begin{lemma}[Local linearization]\label{lem:linearized_tensor_hjb}
    Let $\bm{B}\in\mathbb{R}^{\mathbf{n}+1}$. Then, the linearization of \eqref{eq:tensor_hjb} at $\bm{B}$ is given by
    \begin{equation}\label{eq:linearized_hjb}
        \bm{{\dot{A}}}(t) = (\bm{L} +2\bm{P}{\bm{N}\!\!\bm{L}}_{\bm{B}}) \bm{A}(t) - \bm{P}{\bm{N}\!\!\bm{L}}(\bm{B}).
    \end{equation}
\end{lemma}
    \begin{proof}
    Note that the linearization of $\operatorname{NonLin}(v) = -\|\nabla v\|^2$ around a fixed $v_0\in\operatorname{span}\Pi_{\mathbf{n}}$ is given by
\begin{equation}
   \operatorname{NonLin}_{v_0}(v) = -2\langle \nabla v_0, \nabla v\rangle + \|\nabla v_0\|^2 = -2\langle \nabla v_0, \nabla v\rangle - \operatorname{NonLin}(v_0).
\end{equation}
Now, for $\bm{A},\bm{B}\in\mathbb{R}^{\mathbf{n}+1}$ we have
\begin{align*}
\operatorname{NonLin}_{v_{\bm{B}}}(v_{\bm{A}}) &= -2\langle \nabla v_{\bm{B}}, \nabla v_{\bm{A}}\rangle - \operatorname{NonLin}(v_{\bm{B}}) \\
    &= 2v_{{\bm{N}\!\!\bm{L}}_{\bm{B}}\bm{A}} - v_{{\bm{N}\!\!\bm{L}}(\bm{B})} \\
    &= v_{2{\bm{N}\!\!\bm{L}}_{\bm{B}}\bm{A}-{\bm{N}\!\!\bm{L}}(\bm{B})}.
\end{align*}
Since the other operators appearing on the right hand side of \eqref{eq:projected_hjb} are linear, \eqref{eq:linearized_hjb} follows.
\end{proof}
By Theorem \ref{thm:tensor_hjb}, it suffices to solve \eqref{eq:tensor_hjb} for $\bm{A}(t)$ since this solution defines the solution of \eqref{eq:projected_hjb} via $v_t = v_{\bm{A}(t)}$.
In the rest of this section, we present principled ways of computing approximate solutions to \eqref{eq:tensor_hjb} on the low rank manifold $\mathcal{M}_{\mathbf{r}}$. Two methods are investigated: 
\begin{enumerate}
    \item A simple explicit Euler scheme with adaptive step sizes and retraction after every step, see Section \ref{sec:explicit_euler}.
    \item A dynamical low rank integrator designed for time integration of Tensor Trains 
    \ifnum\classstyle=2
    \cite{lubich2015}, 
    \else
    \cite{lubich2015},
    \fi
    see Section \ref{sec:dlra}.
\end{enumerate}

\subsection{Time adaptive explicit Euler scheme}
\label{sec:explicit_euler}

\paragraph{Preliminaries.} In the following, we define a number of time points $N\in\mathbb{N}$, a sequence of times $0 = t_0 < t_1 < \ldots < t_{N} = T$, a TT-rank function $t\mapsto \bm{r}_t\in\mathbb{N}^{d-1}$ assigning to every time a Tensor Train rank and discrete approximations  $\mathcal{M}_{\mathbf{r}_{t_n}} \ni \bm{Y}_{t_n}\approx \bm{A}(t_n)$, $n=0,\ldots,N$, to the solution $\bm{A}(t)$ of \eqref{eq:tensor_hjb}. Throughout this section, let $\tau_{\max}, \delta_{\operatorname{proj}}, \delta_{\operatorname{rank}}, \delta_{\operatorname{contr}}  > 0$ and a reduction parameter $\rho\in (0,1)$. Denote the potential of the standard normal distribution by $v_{\infty}(x) = \|x\|^2/2$ and note that this function has FTT rank $(2,\ldots,2)$ \cite[Theorem 2]{oseledets2013constructive}. In practice we choose $\bm{r}_t$ to be bounded by  
$\operatorname{TT-rank}(v_0)$ and $\operatorname{TT-rank}(v_\infty)$ with adaptive rank reduction based on TT-rounding error induced by the retraction from \eqref{eq:retraction_op}.

\paragraph{Time adaptive explicit Euler step.} Starting with $n=0$, we have $\bm{Y}_{t_n}\in\mathcal{M}_{\bm{r}_{t_n}}$. By Section \ref{sec:TTs}, the right-hand side of \eqref{eq:tensor_hjb} applied to $\bm{Y}_{t_n}$, i.e. the tensor $\bm{L}\bm{Y}_{t_n} + \bm{P}{\bm{N}\!\!\bm{L}}(\bm{Y}_{t_n})$
has TT-rank at most $ 2\bm{r} + 2\bm{r}^2$
and so the addition
\begin{equation}\label{eq:euler_project}
    \overline{\bm{Y}}_{t_{n}+\tau_n} = \bm{Y}_{t_n} + \tau_n(\bm{L}\bm{Y}_{t_n} + \bm{P}{\bm{N}\!\!\bm{L}}(\bm{Y}_{t_n}))
\end{equation}
has TT-rank at most  $ 3\mathbf{r} + 2\mathbf{r}^2$ for any $\tau_n > 0$.

Since we require the next iterate to be a Tensor Train of rank $\mathbf{r}_{t_n+\tau_n}$, we retract to the appropriate manifold, setting
\begin{equation}\label{eq:euler_retract}
    \bm{Y}_{t_{n}+\tau_n} = \bm{R}_{\bm{r}_{t_n+\tau_n}}
    ( \overline{\bm{Y}}_{t_{n}+\tau_n}),
\end{equation}
where $ \bm{R}_{\bm{r}}$ denotes the TT rounding procedure based on higher order singular value decomposition and mapping to $\mathcal{M}_{\bm{r}}$, which was presented in \cite[Algorithm 2]{oseledets2011tensor}.
Note that \eqref{eq:euler_retract} corresponds to \eqref{eq:euler_sketch} with the $\operatorname{Compression}$ given by the retraction operator, i.e. by the higher order singular value decomposition. 
Up to now the choice of the step size $\tau_n$ was arbitrary. In what follows we set constraints on the step size $\tau_n$ based on three stability criteria. 

\paragraph{Criterion 1: local stiffness.}
At each iteration, we restrict the stepsize dependent on the local stiffness of the ODE. 
We use a heuristic based on local linearizations of \eqref{eq:tensor_hjb} to determine suitable upper bounds for the stepsize. By Lemma \ref{lem:linearized_tensor_hjb}, the local stiffness at the current iterate $\bm{Y}_{t_n}$ is governed by the linear operator
\begin{equation}
\label{eq:stifness_op}
    \bm{H}_{\bm{Y}_{t_n}} \coloneqq \bm{L} + 2\bm{P}{\bm{N}\!\!\bm{L}}_{\bm{Y}_{t_n}}.
\end{equation}

If the current iterate $\bm{Y}_{t_n}$ defines a zero mean Gaussian with diagonal covariance $\operatorname{diag}(a_{ii},i=1,\ldots,d)$, the eigenvalues of $\bm{H}_{\bm{Y}_{t_n}}$ can be bounded by $2\sum_{i=1}^d |1-2a_{ii}|$ (the details of the calculation can be found in 
\ifnum\online=1
the Online Supplementary Material).
\else
Appendix \ref{app:eigenvalues}).\fi
In general, $\bm{H}_{\bm{Y}_{t_n}}$  defines a non-symmetric TT operator. To the knowledge of the authors, estimation of the largest absolute eigenvalue of general non-symmetric TT operators is an open question. Here, we rely on a simpler idea. In particular as we are dealing with real valued tensors $\bm{Y}_{t_n}$, we avoid analyzing the operator action on complex space. In contrast, we are interested in the effect of the current operator in the neighborhood of the current iterate. This is realized by estimating the largest absolute \textit{real} eigenvalue of $\bm{H}_{\bm{Y}_{t_n}}$ denoted by $\lambda_{t_n}$ with corresponding eigenspace that is not orthogonal to the current iterate $\bm{Y}_{t_n}$,
by a power iteration. The resulting scheme is detailed in Algorithm \ref{alg:power_ub}. The current Tensor Train iterate $\bm{Y}_{t_n}$ serves as an initial guess for the eigentensor. The procedure then resembles a standard power iteration with an additional retraction step in line 6, which reduces computational burden.
In practice we are only interested in the absolute value of the eigenvalue or a meaningful upper bound $\overline{\lambda}_{t_n}$ and not in the corresponding eigentensor. Note that the eigenvalue usually converges at much higher order than the eigentensor. The aforementioned upper bound then is obtained through a simple rounding up strategy with a specified number of accurate non-zero digits, see Algorithm \ref{alg:power_ub}. Based on the return $\overline{\lambda}_{t_n}$ of the power iteration, we define a maximal stable stepsize $\tau_{\lambda}$ by
\begin{align}
\label{eq:max_stable_step}
    \tau_{\lambda} \coloneqq \frac{2\rho}{|\overline{\lambda}_{t_n}|}.
\end{align}
In experiments, this stiffness estimation proves essential for the solver to converge.

\begin{algorithm}
	\caption{
 Upper bound estimating the principal real eigenvalue $\lambda_{t_n}$ of $\bm{H}_{\bm{Y}_{t_n}}$ from \eqref{eq:stifness_op} based on power iteration.}\label{alg:power_ub}
	\begin{algorithmic}[1]
 \Require
 {
    $\left\{\begin{array}{ll}
         \text{\tabitem current iterate } \bm{X}_0=\bm{Y}_{t_n} &  \text{\tabitem maximum allowed TT rank } \bm{r}\in\mathbb{N}^{d-1} \\ 
         \text{\tabitem application of }\bm{H}_{Y_{t_n}} &  \text{\tabitem number of correct non-zero digits } p\in\mathbb{N}
    \end{array}
    \right. $
 }
 \vspace{0.5em}
 \Ensure upper bound $\overline{\lambda}_{t_n}$ 
 \vspace{0.5em}
 \State Let $\lambda_{t_n}^k$ denote the $k$-th iterate.
 \State Set $k=0$.
 \While{$p$-th non-zero digit of $\lambda_{t_n}^k$ is changing}
 \State Let $K\in\mathbb{N}$, $\bm{X}_0\in\mathcal{M}_{\mathbf{r}}$.
	\State $\hat{\bm{X}}_{k} = \bm{X}_{k} / \| \bm{X}_{k} \|_F$			
    \State $\bm{X}_{k+1} = \bm{R}_{\mathbf{r}}(\bm{H}_{\bm{Y}_{t_{n}}}\hat{\bm{X}}_{k})$
                \State $\lambda_{t_n}^k = \langle \hat{\bm{X}}_k , \bm{X}_{k+1}\rangle$
     \State $k=k+1$
    \EndWhile
    \State Define position $P$ of first non zero digit with $P= \lceil - \log_{10}(\lambda_{t_n}^k)\rceil $.
    \State Define upper bound treshold $\epsilon_p = 10^{- \left(P + p \right)}$.
    \State Define $\overline{\lambda}_{t_n}^k = \lambda_{t_n}^k + \epsilon_p$.
	\end{algorithmic} 
\end{algorithm}

\paragraph{Criterion 2: local relative projection error.} 
For stepsize $\tau > 0$ consider the iterate $\overline{\bm{Y}}_{t_n+\tau}$ defined by \eqref{eq:euler_project} for $\tau_n=\tau$ and let $\overline{\overline{\bm{Y}}}_{t_n+\tau} = \bm{Y}_{t_n} + \tau(\bm{L}\bm{Y}_{t_n} + {\bm{N}\!\!\bm{L}}(\bm{Y}_{t_n}))$ be an Euler step with the non-projected equation. Let 
\begin{align}
\label{eq:step_size_proj}
    \tau_{\operatorname{proj}} \coloneqq   
    \begin{cases}
    \tau_{\max}, &  ~\textnormal{if}~  \| \bm{P}{\bm{N}\!\!\bm{L}}(\bm{Y}_{t_n}) - {\bm{N}\!\!\bm{L}}(\bm{Y}_{t_n})\|_F = 0,\\
        \frac{\delta_{\operatorname{proj}}}{\| \bm{P}{\bm{N}\!\!\bm{L}}(\bm{Y}_{t_n}) - {\bm{N}\!\!\bm{L}}(\bm{Y}_{t_n})\|_F / \|{\bm{N}\!\!\bm{L}}(\bm{Y}_{t_n})\|_F}, &  ~\textnormal{else}.  
    \end{cases}
\end{align}
Then, for any $\tau \leq \tau_{\operatorname{proj}}$ we get 
\begin{equation}\label{eq:tau_proj}
    \| \overline{\bm{Y}}_{t_n+\tau}- \overline{\overline{\bm{Y}}}_{t_n+\tau} \|_F \leq \tau\| \bm{P}{\bm{N}\!\!\bm{L}}(\bm{Y}_{t_n}) - {\bm{N}\!\!\bm{L}}(\bm{Y}_{t_n})\|_F \leq \delta_{\operatorname{proj}} \|{\bm{N}\!\!\bm{L}}(\bm{Y}_{t_n})\|_F.
\end{equation}
Hence, the projection error of the Euler step, normalized with respect to the magnitude of the degree increasing nonlinear part ${\bm{N}\!\!\bm{L}}(\bm{Y}_{t_n})$, is bounded from above by $\delta_{\operatorname{proj}}$.

\paragraph{Criterion 3: local relative retraction error.} 

Determine maximum $\tau_{\operatorname{rank}}$ such that
\begin{equation}
\label{eq:tau_rank}
   \frac{ \| \overline{\bm{Y}}_{t_n+\tau_{\operatorname{rank}}} - \bm{Y}_{t_n+\tau_{\operatorname{rank}}} \|_F}{\| \overline{\bm{Y}}_{t_n+\tau_{\operatorname{rank}}} \|_F} \leq \delta_{\operatorname{rank}}.
\end{equation}
Here, we initially choose $\tau_{\operatorname{rank}}^0 = \tau_{n-1}$ and proceed with $\tau_{\operatorname{rank}}^k = \frac{1}{2}\tau_{\operatorname{rank}}^{k-1}$ until $\tau_{\operatorname{rank}}^k$ fulfils condition \eqref{eq:tau_rank}. Then, we use bisection iteration to determine the maximum 
$\tau_{\operatorname{rank}} \in (\frac{1}{2}\tau_{\operatorname{rank}}^{k}, \tau_{\operatorname{rank}}^k]$ satisfying \eqref{eq:tau_rank}.

\paragraph{Final stepsize choice.} After these three criteria, the next step size $\tau_n$ in \eqref{eq:euler_retract} and the next time $t_{n+1}$ are defined as
\begin{align}
    \tau_n &\coloneqq \min\{ \tau_{\max}, \tau_{\lambda}, \tau_{\operatorname{proj}}, \tau_{\operatorname{rank}}, T-t_n \},\label{eq:euler_stepsize} \\
    t_{n+1} &\coloneqq t_n + \tau_n,
\end{align}
where the term $T-t_n$ ensures that we end exactly at terminal time $T$. The single time step \eqref{eq:euler_retract} is repeated for $n=0,1,\ldots$ with stepsize \eqref{eq:euler_stepsize} until $t_{n+1}=T$, in which case we define $N=n+1$. \\

In addition to the adaptivity in the stepsize, the solver also incorporates adaptivity in the polynomial degree as well as the TT rank, which is detailed in the following.

\paragraph{Adaptive decrease of polynomial degree}
Motivated by the fact that $v_t\to v_{\infty}\in\Pi_{(2,\ldots,2)}$ as $t\to \infty$ at exponential rate, we introduce a simple adaptive choice for the polynomial degree. Assume that the degrees of $\bm{Y}_{t_n}$ at time $t_n$ are given by $\bm{n}_{t_n}\in\mathbb{N}_0^d$.
Let $ \bm{Y}_{t_n}^k$ denote the order $d-1$ tensor, which for $k=1,\ldots,d$ is given as 
$$
\bm{Y}_{t_n}^k = \left( \bm{Y}_{t_n}[\bm{\alpha}]\right)_{\bm{\alpha}\in[\bm{n}_{t_n}], \alpha_k = (\bm{n}_{t_n})_k}.
$$
This is a slice of the full coefficient tensor $\bm{Y}_{t_n}$ fixing $\alpha_k = (\bm{n}_{t_n})_k$ which is the highest polynomial degree in the $k$-th dimension at time $t_n$. Now, in case of
\begin{equation}
\label{eq:C_slices}
    \| \bm{Y}_{t_n}^k \|_F \leq \delta_{\text{contr}}, 
\end{equation}
we truncate the highest polynomial degree in the $k$-th direction. Since $\bm{Y}_{t_n}$ is given in TT format with
$$
\bm{Y}_{t_n}[\bm{\alpha}] = Y_{t,1}[\alpha_1] \cdot \ldots  \cdot Y_{t,d}[\alpha_d],
$$
this operation is realized by truncation of the component tensor $Y_k$ and possibly adapting the TT-ranks. 

\paragraph{Adaptive choice of TT rank}
Motivated by the conjecture, that the FTT rank of $v_t$ is bounded by the FTT rank of $v_0$ and $v_\infty$, i.e. $\bm{r}_\infty=(2,\ldots,2)$, we perform two retraction steps with respect to these bounds after the time step at time $t_n$. First a retraction with respect to the rank 
\begin{equation}
\label{eq:rank_reduction}
    \hat{\bm{r}}_{t_{n}+\tau_n}  = \max\{ \bm{r}_{t_n}, \bm{r}_\infty \}
\end{equation}
is performed where the maximum is understood component wise. This serves to ensure that the rank of the solution remains bounded by the maximum of the initial rank and the rank of the standard normal potential. Furthermore, a rounding procedure \cite[Algorithm 2]{oseledets2011tensor} with respect to $\delta_{\text{contr}}$ is performed to potentially further decrease the rank and thus define $\bm{r}_{t_n+\tau_n}$. In practice both retraction steps can be performed efficiently in a single operation, which leads to \eqref{eq:euler_retract}.\\

The proposed time adaptive explicit Euler scheme is summarized in Algorithm \ref{alg:HJB_algo}.

\ifnum\classstyle=0
\begin{algorithm}
		\caption{
Time adaptive explicit Euler Scheme to solve HJB equation based on Tensor Trains }\label{alg:HJB_algo}
	\begin{algorithmic}[1]
 \Require
 {
    $\left\{\begin{array}{ll}
         \text{\tabitem } v_0 \text{ given in TT format},\\[0.2em]   \text{\tabitem } T>0 \text{ maximum finite time horizon} , \\[0.2em]
         \text{\tabitem } \tau_{\text{max}} > 0 \text{,  bound for the stepsize}\\[0.2em]
          \text{\tabitem  reduction stiffness parameter }\rho \in(0,1), \\[0.2em]
         \text{\tabitem step size proposal hyperparameter }  
         \delta_{\operatorname{proj}}, \delta_{\operatorname{rank}}, \\[0.2em]
         \text{\tabitem degree of freedom contribution tolerance }
         \delta_{\operatorname{contr}}>0.
    \end{array}
    \right. $
 }
 \vspace{0.5em}
 \Ensure Discrete sequence $(v_{t_n})_n$ defined on subsequently determined adaptive time points $t_n\in[0,T].$
 \vspace{0.5em}
  \State Set $t=0$.
  \While{$t\leq T$} 
  \State \textbf{Determine next time step}:   
  \IndState Compute maximal stable stepsize $\tau_\lambda$. \Comment{see \eqref{eq:max_stable_step}}
  \IndState Compute step size proposal $\tau_{\operatorname{proj}}$ based on projection error. 
  \Comment{see \eqref{eq:step_size_proj}}
  \IndState Compute step size proposal $\tau_{\operatorname{rank}}$ based on relative retraction error. \Comment{see \eqref{eq:tau_rank}}
 \IndState Determine final step size  $\tau = \min\{ \tau_{\max}, \tau_{\lambda}, \tau_{\operatorname{proj}}, \tau_{\operatorname{rank}}, T-t \}$.
  \State \textbf{Perform a single Euler step}
  \IndState Set $t=t+\tau$.
  \IndState Approximate $v_t$ via algebraic manipulation of the underlying TT format.  \Comment{see \eqref{eq:euler_project}}
  \IndState Perform a retraction step of the resulting coefficient in TT format. \Comment{see  \eqref{eq:euler_retract}} 
  \State \textbf{(Re-)compression}
  \IndState Check for potential polynomial degree decrease using $\delta_{\operatorname{contr}}$. \Comment{see  \eqref{eq:C_slices}}
  \IndState Check for potential rank reduction using $\delta_{\operatorname{contr}}$. \Comment{see \eqref{eq:rank_reduction}}
 \EndWhile
	\end{algorithmic} 
\end{algorithm}
\fi
\ifnum\classstyle=1
\begin{algorithm}
		\caption{
Time adaptive explicit Euler Scheme to solve HJB equation based on Tensor Trains }\label{alg:HJB_algo}
	\begin{algorithmic}[1]
 \Require
 {
    $\left\{\begin{array}{ll}
         \text{\tabitem } v_0 \text{ given in TT format},\\[0.2em]   \text{\tabitem } T>0 \text{ maximum finite time horizon} , \\[0.2em]
         \text{\tabitem } \tau_{\text{max}} > 0 \text{,  bound for the stepsize}\\[0.2em]
          \text{\tabitem  reduction stiffness parameter }\rho \in(0,1), \\[0.2em]
         \text{\tabitem step size proposal hyperparameter }  
         \delta_{\operatorname{proj}}, \delta_{\operatorname{rank}}, \\[0.2em]
         \text{\tabitem degree of freedom contribution tolerance }
         \delta_{\operatorname{contr}}>0.
    \end{array}
    \right. $
 }
 \vspace{0.5em}
 \Ensure Discrete sequence $(v_{t_n})_n$ defined on subsequently determined adaptive time points $t_n\in[0,T].$
 \vspace{0.5em}
  \State Set $t=0$.
  \While{$t\leq T$} 
  \State \textbf{Determine next time step}:   
  \IndState Compute maximal stable stepsize $\tau_\lambda$. \Comment{see \eqref{eq:max_stable_step}}
  \IndState Compute step size proposal $\tau_{\operatorname{proj}}$ based on projection error. 
  \Comment{see \eqref{eq:step_size_proj}}
  \IndState Compute step size proposal $\tau_{\operatorname{rank}}$ based on relative retraction error. \Comment{see \eqref{eq:tau_rank}}
 \IndState Determine final step size  $\tau = \min\{ \tau_{\max}, \tau_{\lambda}, \tau_{\operatorname{proj}}, \tau_{\operatorname{rank}}, T-t \}$.
  \State \textbf{Perform a single Euler step}
  \IndState Set $t=t+\tau$.
  \IndState Approximate $v_t$ via algebraic manipulation of the underlying TT format.  \Comment{see \eqref{eq:euler_project}}
  \IndState Perform a retraction step of the resulting coefficient in TT format. \Comment{see  \eqref{eq:euler_retract}} 
  \State \textbf{(Re-)compression}
  \IndState Check for potential polynomial degree decrease using $\delta_{\operatorname{contr}}$. \Comment{see  \eqref{eq:C_slices}}
  \IndState Check for potential rank reduction using $\delta_{\operatorname{contr}}$. \Comment{see \eqref{eq:rank_reduction}}
 \EndWhile
	\end{algorithmic} 
\end{algorithm}
\fi
\ifnum\classstyle=2
\begin{algorithm}
		\caption{
Time adaptive explicit Euler Scheme to solve HJB equation based on Tensor Trains }\label{alg:HJB_algo}
	\begin{algorithmic}[1]
 \Require
 {
    $\left\{\begin{array}{ll}
         \text{\tabitem } v_0 \text{ given in TT format},\\[0.2em]   \text{\tabitem } T>0 \text{ maximum finite time horizon} , \\[0.2em]
         \text{\tabitem } \tau_{\text{max}} > 0 \text{,  bound for the stepsize}\\[0.2em]
          \text{\tabitem  reduction stiffness parameter }\rho \in(0,1), \\[0.2em]
         \text{\tabitem step size proposal hyperparameter }  
         \delta_{\operatorname{proj}}, \delta_{\operatorname{rank}}, \\[0.2em]
         \text{\tabitem degree of freedom contribution tolerance }
         \delta_{\operatorname{contr}}>0.
    \end{array}
    \right. $
 }
 \vspace{0.5em}
 \Ensure Discrete sequence $(v_{t_n})_n$ defined on subsequently determined adaptive time points $t_n\in[0,T].$
 \vspace{0.5em}
  \State Set $t=0$.
  \While{$t\leq T$} 
  \State \textbf{Determine next time step}:   
  \IndState Compute maximal stable stepsize $\tau_\lambda$. \Comment{see \eqref{eq:max_stable_step}}
  \IndState Compute step size proposal $\tau_{\operatorname{proj}}$ based on projection error. 
  \Comment{see \eqref{eq:step_size_proj}}
  \IndState Compute step size proposal $\tau_{\operatorname{rank}}$ based on relative retraction error. \Comment{see \eqref{eq:tau_rank}}
 \IndState Determine final step size  $\tau = \min\{ \tau_{\max}, \tau_{\lambda}, \tau_{\operatorname{proj}}, \tau_{\operatorname{rank}}, T-t \}$.
  \State \textbf{Perform a single Euler step}
  \IndState Set $t=t+\tau$.
  \IndState Approximate $v_t$ via algebraic manipulation of the underlying TT format.  \Comment{see \eqref{eq:euler_project}}
  \IndState Perform a retraction step of the resulting coefficient in TT format. \Comment{see  \eqref{eq:rank_reduction}} 
  \State \textbf{(Re-)compression}
  \IndState Check for potential polynomial degree decrease using $\delta_{\operatorname{contr}}$. \Comment{see  \eqref{eq:C_slices}}
  \IndState Check for potential rank reduction using $\delta_{\operatorname{contr}}$. \Comment{see \eqref{eq:euler_retract}}
 \EndWhile
	\end{algorithmic} 
\end{algorithm}
\fi
\ifnum\classstyle=3
\begin{algorithm}
		\caption{
Time adaptive explicit Euler Scheme to solve HJB equation based on Tensor Trains }\label{alg:HJB_algo}
	\begin{algorithmic}[1]
 \Require
 {
    $\left\{\begin{array}{ll}
         \text{\tabitem } v_0 \text{ given in TT format},\\[0.2em]   \text{\tabitem } T>0 \text{ maximum finite time horizon} , \\[0.2em]
         \text{\tabitem } \tau_{\text{max}} > 0 \text{,  bound for the stepsize}\\[0.2em]
          \text{\tabitem  reduction stiffness parameter }\rho \in(0,1), \\[0.2em]
         \text{\tabitem step size proposal hyperparameter }  
         \delta_{\operatorname{proj}}, \delta_{\operatorname{rank}}, \\[0.2em]
         \text{\tabitem degree of freedom contribution tolerance }
         \delta_{\operatorname{contr}}>0.
    \end{array}
    \right. $
 }
 \vspace{0.5em}
 \Ensure Discrete sequence $(v_{t_n})_n$ defined on subsequently determined adaptive time points $t_n\in[0,T].$
 \vspace{0.5em}
  \State Set $t=0$.
  \While{$t\leq T$} 
  \State \textbf{Determine next time step}:   
  \IndState Compute maximal stable stepsize $\tau_\lambda$. \Comment{see \eqref{eq:max_stable_step}}
  \IndState Compute step size proposal $\tau_{\operatorname{proj}}$ based on projection error. 
  \Comment{see \eqref{eq:step_size_proj}}
  \IndState Compute step size proposal $\tau_{\operatorname{rank}}$ based on relative retraction error. \Comment{see \eqref{eq:tau_rank}}
 \IndState Determine final step size  $\tau = \min\{ \tau_{\max}, \tau_{\lambda}, \tau_{\operatorname{proj}}, \tau_{\operatorname{rank}}, T-t \}$.
  \State \textbf{Perform a single Euler step}
  \IndState Set $t=t+\tau$.
  \IndState Approximate $v_t$ via algebraic manipulation of the underlying TT format.  \Comment{see \eqref{eq:euler_project}}
  \IndState Perform a retraction step of the resulting coefficient in TT format. \Comment{see  \eqref{eq:euler_retract}} 
  \State \textbf{(Re-)compression}
  \IndState Check for potential polynomial degree decrease using $\delta_{\operatorname{contr}}$. \Comment{see  \eqref{eq:C_slices}}
  \IndState Check for potential rank reduction using $\delta_{\operatorname{contr}}$. \Comment{see \eqref{eq:rank_reduction}}
 \EndWhile
	\end{algorithmic} 
\end{algorithm}
\fi


\subsection{Dynamical low rank approximation}
\label{sec:dlra}

While the time adaptive explicit Euler scheme presented in the previous section offers a conceptually simple integration method, \textit{Dynamical low rank appromxation} (DLRA) 
\ifnum\classstyle=2
\cite{lubich2007,lubich2013,lubich2015} 
\else
\cite{lubich2007,lubich2013,lubich2015}
\fi
methods offer another principled way of approximately integrating tensor valued ODEs of the form \eqref{eq:tensor_hjb}.

Here, the idea is to formulate an approximation of a tensor valued ODE $$\bm{{\dot{A}}}(t) = \bm{F}(t,\bm{A}(t)), \qquad \bm{A}(0) = \bm{A}_0,$$ where $\bm{n}\in\mathbb{N}^{d}$, $\bm{A}(t)\in\mathbb{R}^{\bm{n}}$ and $\bm{F}\colon [0,\infty]\times\mathbb{R}^{\bm{n}}\rightarrow\mathbb{R}^{\bm{n}}$ on a fixed rank manifold $\mathcal{M}_{\bm{r}}$. This is done via projection of the right-hand side onto the tangent space of $\mathcal{M}_{\bm{r}}$. More precisely, for a fixed $\bm{r}\in \mathbb{N}^{d-1}$, the approximation is defined as
    \begin{equation}\label{eq:tensor_hjb_dlra}
       \bm{{\dot{Y}}}(t) = \mathcal{P}_{\mathcal{T}_{\bm{Y}(t)}}\bm{F}(t,\bm{Y}(t)), \qquad \bm{Y}(0) = \bm{Y}_0 \approx \bm{A}_0,
    \end{equation}
where $\bm{Y}_0\in\mathcal{M}_{\bm{r}}$ and $\mathcal{P}_{\mathcal{T}_{\bm{Y}(t)}}$ denotes the orthogonal projection (in Frobenius norm) onto the tangent space of $\mathcal{M}_{\bm{r}}$ in $\bm{Y}(t)$. Note that due to this projection, a solution of \eqref{eq:tensor_hjb_dlra} satisfies $\bm{Y}(t)\in\mathcal{M}_{\bm{r}}$ for all $t$. In 
\ifnum\classstyle=2
\cite{eigel2023dynamical} 
\else
\cite{eigel2023dynamical}
\fi
the authors use an explicit Euler discretization of \eqref{eq:tensor_hjb_dlra} for the solution of HJB equations appearing in deterministic optimal control based on spatially discretized parabolic PDEs. However, leveraging the form of the tangent space, the projector on the right hand side can be decomposed into a sum of projectors corresponding to orthogonal subspaces. In
\ifnum\classstyle=2
\cite{lubich2013},
\else
\cite{lubich2013}, 
\fi
the authors propose to use this sum structure for a Lie-Trotter type splitting scheme in the case of a matrix valued ODE, which is termed the \textit{projector-splitting} integrator. Consequently, 
\ifnum\classstyle=2
\cite{lubich2015} 
\else
\cite{lubich2015}
\fi
extends the projector splitting integrator to the tensor setting. One of the key properties of this integrator is that each discrete step preserves the rank $\bm{r}$. 

In our scheme, using a step with the integrator from 
\ifnum\classstyle=2
\cite{lubich2015} 
\else
\cite{lubich2015}
\fi
instead of the explicit Euler step \eqref{eq:euler_project} leads to a new iterate $\overline{\bm{Y}}_{t_{n}+\tau_n}$ with the same rank as $\bm{Y}_{t_{n}}$. Hence, the retraction \eqref{eq:euler_retract} becomes a mere rounding procedure and the rank of two consecutive iterates is monotonically decreasing. This is a desirable property if the initial rank satisfies $\bm{r}_{t_{0}} \geq (2,\ldots,2)$. For $\bm{r}_{t_{0}} \ngeq (2,\ldots,2)$, the projector-splitting is unsuited because it restricts the rank from above to $\bm{r}_{t_{n}} \leq \bm{r}_{t_{0}}$ and so $\bm{r}_{t_{n}}$ can not converge to the correct rank $(2,\ldots,2)$. 

Incorporating more recent state-of-the-art dynamical low rank integrators for matrix valued ODEs such as 
\ifnum\classstyle=2
\cite{ceruti2022rank,ceruti2023parallel}
\else
\cite{ceruti2022rank,ceruti2023parallel}
\fi
to the Tensor Train setting could lead to significant improvements of the proposed method. In particular, the Basis
Update $\&$ Galerkin (BUG) integrator 
\ifnum\classstyle=2
\cite{ceruti2022rank} 
\else
\cite{ceruti2022rank}
\fi
introduces rank adaptivity, while the fully parallel integrator 
\ifnum\classstyle=2
\cite{ceruti2023parallel} 
\else
\cite{ceruti2023parallel}
\fi
could additionally greatly speed up computations in high dimensions. Their application in our method is a topic of future research.

\subsection{Evaluation of the low-rank model}

As the result of section \ref{sec:explicit_euler} or \ref{sec:dlra} we have a representation of the value function 
in the spirit of \eqref{eq:TT_repr}
at discrete set of time points the form $t\in \{t_0,t_1,\ldots,T\}$ of the form
\begin{equation}
\label{eq:solution_contraction}
    v_t(x)  = \sum\limits_{\bm\alpha\in [\bm{n}_t]} \boldsymbol Y_t[\bm\alpha] p_{\bm\alpha}(x),
\end{equation}
for some $\bm{n}_t\in\mathbb{N}_0^d$ and $Y_t$ given in tensor train format resulting as the discrete solution of \eqref{eq:tensor_hjb}.

We now want to discuss the evaluation of $v_t(x)$ at arbitrary time $t\in[0,T]$ and $x\in \mathbb{R}^d$. This is motivated by the 
 reverse-time sampling process, which is permitted to be time adaptive and may require evaluation in time points not included in the set $\{t_0,\ldots,t_N\}$ .
 
For this we propose a very simple solution. Let $t^\ast \in[0,T]$. Let 
\begin{equation}
    \overline{t} = \max\{ t \in \{t_0,\ldots,t_N\}  \colon t \leq t^\ast\} 
\end{equation}
Let $\tau = t^\ast - \underline{t}$. Then, we compute the coefficient representation in Tensor train format of $v_{t^\ast}$ through a single Euler- or DLRA step with step size $\tau$.
Note that this step size is within the step size bounds implied by the adaptive scheme proposed earlier. In particular, $\tau$ is smaller than the step size implied by local stiffness.

Lastly, we discuss how the evaluation of the model class is performed in practice. Aside from the evaluation of the polynomial basis functions, only matrix- and vector products have to be computed to evaluate $v_t$. This efficient evaluation is one of the strengths of the Tensor Train format. For $x=(x_s,\ldots,x_d)\in\mathbb{R}^d$ and $t\in[0,T]$, the approximation is defined by a TT $\boldsymbol{Y}_t$ with dimensions $\bm n_t = (n_{t,1},\ldots, n_{t,d})$ and ranks $\mathbf{r}_t = (r_{t,1},\ldots,r_{t,d})$. To evaluate \eqref{eq:solution_contraction}, one first computes $p^i_{j}(x_i)$ for all $i=1,\ldots,d$ and $j=0,\ldots,n_{t,i}$. Now the TT format provides a decomposition of the form $\boldsymbol{Y}_t[\bm\alpha] = Y_{t,1}[\alpha_1]Y_{t,2}[\alpha_2]\cdots Y_{t,d}[\alpha_d]$, where $Y_{t,i}\in \mathbb{R}^{r_{t,i-1},n_{t,i},r_{t,i}}$, $\bm\alpha \in [\bm n_t]$ and hence $\alpha_i$ runs from $0$ to $n_{t,i}$. In particular, \eqref{eq:solution_contraction} implies
\begin{equation}\label{eq:solution_contraction_explicit}
  \begin{aligned}
    v_t(x) &= \sum_{\alpha_1=0}^{n_{t,1}}\ldots \sum_{\alpha_d=0}^{n_{t,d}} Y_{t,1}[\alpha_1]Y_{t,2}[\alpha_2]\cdots Y_{t,d}[\alpha_d]  p^1_{\alpha_1}(x_1)p^2_{\alpha_2}(x_2)\ldots p^{d}_{\alpha_d}(x_d) \\
    &=: Y_{t,1}^{x_1}\cdot Y_{t,2}^{x_2}\cdots Y_{t,d}^{x_d},
\end{aligned}  
\end{equation}
where $Y_{t,i}^{x_i} \in \mathbb{R}^{r_{t,i-1},r_{t,i}}$ results from a simple contraction of $Y_{t,i}$ with the vector $(p^i_{1}(x_i),\ldots,p^i_{n_{t,i}}(x_i))$ over the $n_{t,i}$-dimension.  $Y_{t,1}^{x_1}\cdot Y_{t,2}^{x_2}\cdots Y_{t,d}^{x_d}$ is now a simple matrix product. Note since $r_{t,0}=r_{t,d}=1$, this product boils down to a matrix-vector product, when performed from left to right or vice-versa, yielding a scalar value.
\section{Numerical results}\label{sec:numerics}

Based on Remark \ref{rem:time-reverse_ornstein}, we generate approximate samples from $\mu_\ast$ by means of the discrete process described in Algorithm \ref{alg:reverse_sampling}. The algorithm utilizes the reverse-time process from Remark \ref{rem:time-reverse_ornstein} with $\lambda=0$ discretized at the time-points $t_n$ at which approximate solutions $\bm{Y}_{t_n}$ of the projected HJB \eqref{eq:projected_hjb} are available. These approximations define our surrogate for the score $\nabla \log\pi_{t}$ based on 
\begin{equation}
    v_{t_n} \approx - \log \pi_{t_n}, \quad n=0,\ldots,N,  
\end{equation}
where $v_{t_n} \coloneqq v_{\bm{Y}_{t_n}}$ is understood in the sense of \eqref{eq:TT_repr}. The inner loop over $k$ in Algorithm \ref{alg:reverse_sampling} consists of additional \textit{Langevin-postprocessing} steps 
\ifnum\classstyle=2
\cite{song2020score} 
\else
\cite{song2020score}
\fi
after every step with the reverse process.

As a necessary condition for convergence of the computed solutions $v_{t_n}$ to the potential $v_\infty(x)=\frac{1}{2}x^{\intercal}I_d x$ of the standard normal distribution, we consider convergence of the coefficients of the quadratic part. More precisely, since $v_{t_n}$ is a polynomial, we can always write
\begin{equation}
    v_{t_n}(x) = a_{t_n} + b_{t_n}^{\intercal}x + x^{\intercal}\Sigma_{t_n}^{-1}x + \textnormal{higher order terms},
\end{equation}
with $a_{t_n}\in\mathbb{R},b_{t_n}\in\mathbb{R}^d$ and a symmetric $\Sigma_{t_n}\in\mathbb{R}^{d\times d}$. In this section, we call \textit{covariance error at time} $t_n$ the term 
\begin{equation}\label{eq:cov_error}
  \operatorname{CovErr}(t_n) = \left\| \Sigma^{-1}_{t_n} - I_d/2 \right \|_F / \left \| I_d/2 \right \|_F,  
\end{equation}
i.e. the relative error in Frobenius norm between the current precision matrix and the precision matrix of the standard normal distribution.

We remark that, in the test cases we considered, the results produced by the dynamical low rank integrator 
\ifnum\classstyle=2
\cite{lubich2015} 
\else
\cite{lubich2015}
\fi
(using the same heuristics for adaptive stepsize determination) are similar to the results produced by an explicit Euler stepping with subsequent retraction. Hence, we only present the results of the latter. 

\begin{algorithm}
	\caption{
Sampling from $\pi_\ast$}\label{alg:reverse_sampling}
	\begin{algorithmic}[1]
 \Require
 {
    $\left\{\begin{array}{ll}
         \text{\tabitem Initial samples $\{z^i_0\}_{i=1}^I \sim \mathcal{N}(0,I_d)$ },   \\[0.2em]
         \text{\tabitem Times $\{t_n\}_{n=1}^{N}$ and discrete HJB solution $ \{v_{\bm{Y}_{t_n}}\}_{n=1}^N$ defined by Algorithm \ref{alg:HJB_algo}} , \\[0.2em]
         \text{\tabitem Stepsize $\tau$ and number of steps $K\in\mathbb{N}$ for Langevin postprocessing,} \\[0.2em]
         \text{\tabitem Parameter $\lambda\in [0,1]$ for reverse-time process}
    \end{array}
    \right. $
 }
  \vspace{0.1cm}
 \Ensure Samples $\{z^i_N\}_{i=1}^I \sim \mu_\ast$.
 \vspace{0.1cm}
 \For{$i=1,\ldots,I$ in parallel} 
 \State Generate time points $\{t_n^i\}_{n=1}^N$.
 \For{$n=0,1,\ldots,N-1$}
 \State Set $\tau^i_n = t_{n+1}^i-t_n^i$ 
 \State Sample $\xi_n \sim \mathcal{N}(0,I_d)$ if $\lambda\neq 1$.
     
	\State Set $z^i_{n+1} = z^i_n + \left[ z_n^i + (2-\lambda)\nabla v_{\bm{Y}_{T-t_n}}(z_n^i) \right]\tau_n^i  + \sqrt{2(1-\lambda)\tau_n}\xi_n$. \Comment{Reverse-time process step}
      \For{$\ell=0,1,\ldots,L$} 
 \State Sample $\xi_k \sim \mathcal{N}(0,I_d)$
    \State  $z^i_{n+1} \longleftarrow z^i_{n+1} - \tau  \nabla v_{\bm{Y}_{T-t_n}}(z^i_{n+1})  + \sqrt{2\tau}\xi_k$. \Comment{Langevin post-processing step}
     \EndFor
 \EndFor
\EndFor
	\end{algorithmic} 
\end{algorithm}

\subsection{Verification result: Gaussian setting}

\paragraph{Problem definition} Let $d=10$, $K = [-5,5]^{10}$ and $\Phi(x) = x^{\intercal}\Sigma^{-1}x$, where $\Sigma$ is a randomly generated symmetric positive definite matrix (we sample entries of a matrix $A$ uniformly on $[0,1]$ and then define $\Sigma^{-1} = A^{\intercal}A + 0.1 I_d$). Note that in this setting the polynomial degree of the HJB solution is bounded by $\bm{n}=(2,\ldots,2)$ as $\pi_t$ remains a Gaussian density if $\pi_0$ and $\pi_\infty$ are Gaussian.

\paragraph{Parameters} For Algorithm \ref{alg:HJB_algo}, we choose $T=12$, $\tau_{\max}=0.1$, $\rho = 0.2$, $\delta_{\operatorname{proj}} = \delta_{\operatorname{rank}} = 0.01$, $\delta_{\operatorname{contr}} = 10^{-8}$.

\paragraph{Evaluation} By 
\ifnum\online=1
Lemma B.1 in the Online Supplementary Material,
\else
Lemma \ref{lem:rank_gauss},\fi 
$\Phi$ has an FTT-rank given by $\bm{r} = (3,4,5,6,7,6,5,4,3)$. Since the solution of the HJB is a strictly quadratic polynomial for all times (meaning that no higher or lower degrees than 2 appear), 
\ifnum\online=1
the same lemma
\else
Lemma \ref{lem:rank_gauss}\fi 
also yields that the FTT rank of the solution is bounded from above by $\bm{r}$ for all $t_n$. In Figure \ref{fig:gaussian_rank_plot} the ranks of the solution during integration are displayed. Once the solution reaches a covariance error of $\sim 10^{-7}$, the solver starts to truncate the ranks, meaning that at this point higher ranks give a contribution to the solution which is less than $\delta_{\operatorname{contr}} = 10^{-8}$ in relative Frobenius norm. Finally, all ranks higher than $2$ are truncated, which is to be expected since the standard normal potential has FTT rank $\bm{r}\equiv 2$. At $t=12$, the covariance error has decreased to $\sim 10^{-11}$.\\
Figure \ref{fig:gaussian_stepsize_plot} displays the stepsizes chosen by the solution algorithm. Since the polynomial degree does not increase and the ranks are bounded from above by the initial rank, the stiffness estimate \eqref{eq:max_stable_step} determines the stepsize. \\
As a further demonstration of the consistency of the method, we consider the solution algorithm for different maximal 
stepsizes $\tau_{\max} \in \{0.1,0.01,0.001\}$. 
By the standard theory of the explicit Euler scheme, we expect the time discretization error to be in $\mathcal{O}(\tau_{\max})$. As the second order polynomial ansatz introduces no spatial discretization error in the Gaussian setting, the time discretization is the only source of error in the learned score. 
Figure \ref{fig:Error_plot_gaussian} clearly displays the $\mathcal{O}(\tau_{\max})$ dependence in both the relative Frobenius-error of the covariance matrix as well as the $L^2$-error of the score.

\begin{figure}
    \centering
    \includegraphics[width=0.6\textwidth]{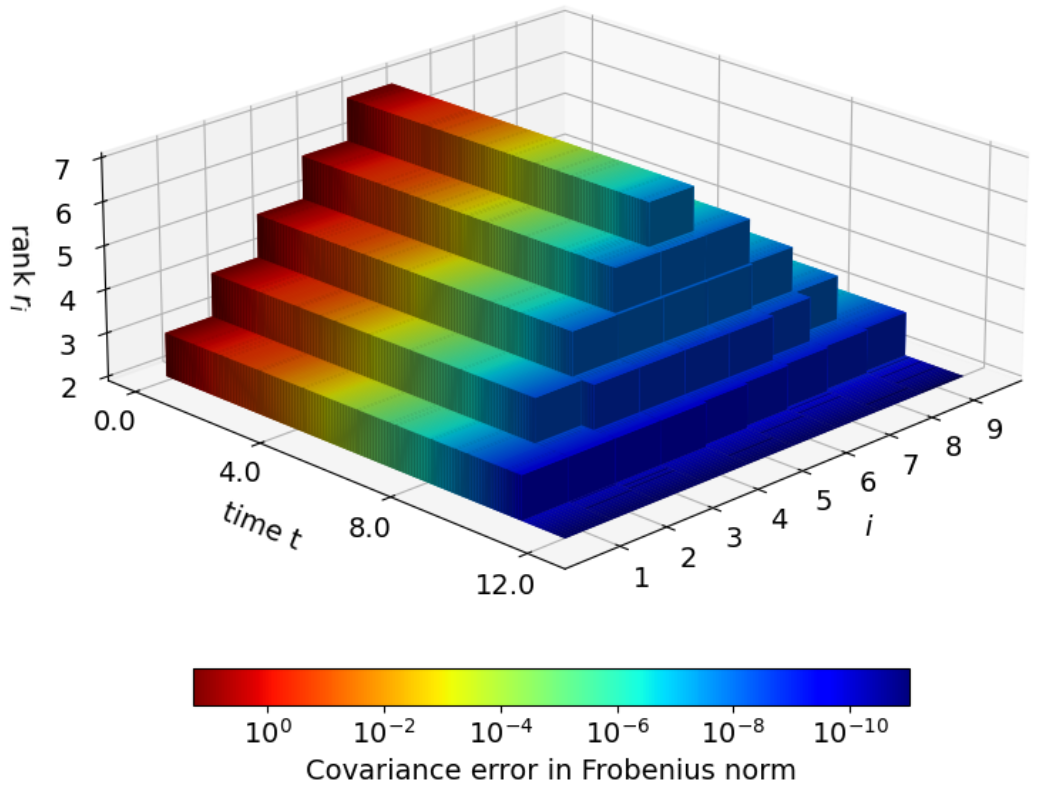}
    \caption{\textit{Development of the solution ranks and the covariance error \eqref{eq:cov_error} over time in the Gaussian setting. Once the solution is close to convergence (in terms of the covariance error), the ranks decrease to the rank $(2,\ldots,2)$ of the potential of the standard normal distribution.}}
    \label{fig:gaussian_rank_plot}
\end{figure}

\begin{figure}
    \centering
    \includegraphics[width=\textwidth]{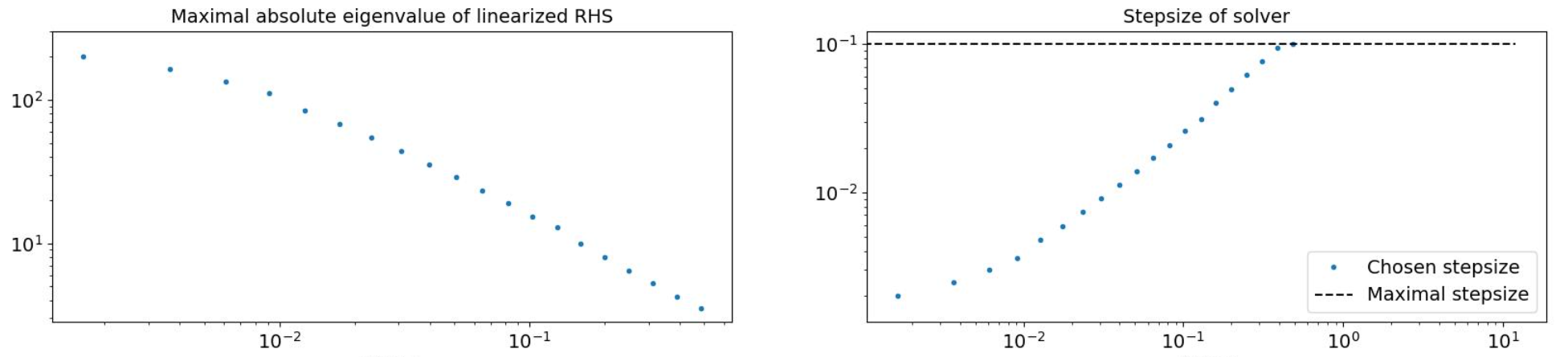}
    \caption{\textit{Approximations of the maximal absolute eigenvalues of the linearized right-hand side $|\overline{\lambda}_{t}|$ determined by the power method (left) and accordingly chosen stepsize $2\rho/|\overline{\lambda}_{t}|$ (right) over time in the Gaussian setting. Note that the eigenvalues decrease monotonically, permitting a monotonous increase of the stepsize until the maximal permitted stepsize $\tau_{\max} = 0.1$ is reached.}}
    \label{fig:gaussian_stepsize_plot}
\end{figure}

\begin{figure}
    \includegraphics[width = 0.9\linewidth]{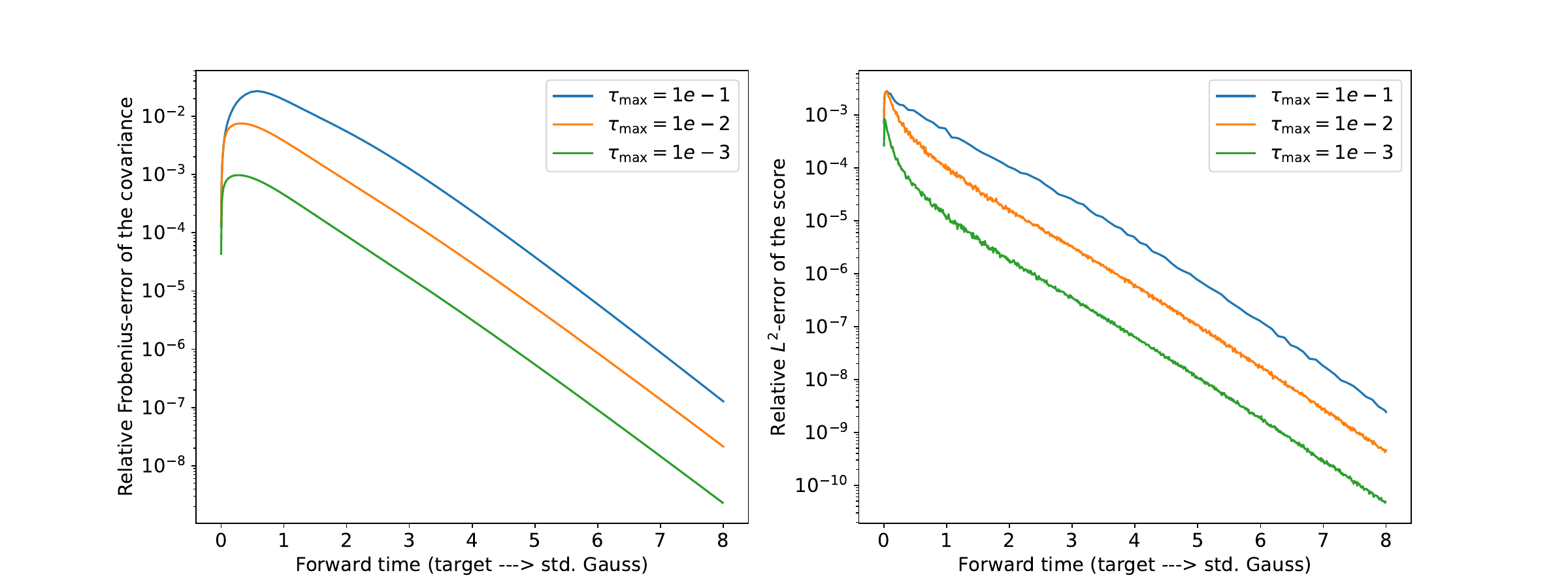}
    \caption{\textit{Comparison of the discretized HJB solution and its derived gradient for different maximum time steps $\tau_{\mathrm{max}}=0.1,0.01,0.001$ to the exact score in the setup of random Gaussian target distribution for $d = 10$. One can clearly see the $\mathcal{O}(\tau_{\max})$ dependence of the error in both gradient and covariance matrix of the approximate solution. Note that for small times $t \ll 1$, the discretizations for maximal stepsizes $0.1$ and $0.01$ lead to similar $L^2$-error in the gradient, which can be attributed to an equally small adaptively chosen stepsize in this early (stiff) regime (compare to the right plot in Figure \ref{fig:gaussian_stepsize_plot} which shows the realized step sizes for $\tau_{\max} = 0.1$).}}
    \label{fig:Error_plot_gaussian}
\end{figure}

\subsection{Mixed nonlinear density}\label{sec:mixed_nonlinear}

\paragraph{Problem definition}
Let $d=20$, $K = [-5,5]^2\times [-2,2]^2 \times [-5,5]^2 \times [-2,2]^{14}$.
Consider the transport map $\mathcal{T}\colon\mathbb{R}^2\to\mathbb{R}^2$ and matrix $\Sigma$  with 
\begin{equation}
     \label{eq:banana}
      \mathcal{T}(x,y) = (x, y + x^2+1), \quad \Sigma = 
      \begin{pmatrix}
          1 & 0.9 \\ 0.9 & 1
      \end{pmatrix}.
\end{equation}

Let $\Phi_1(x,y) = v_\infty( \Sigma^{-1} \mathcal{T}(x,y))$ , 
$\Phi_2(x,y) = x^4 + y^4 -4x^2 -4y^2 - 0.4x+0.1y+8$ and $\Phi_3(x,y) = x^6 + y^6 + 3xy$. Define $\Phi(x) = \Phi_1(x_1,x_2) + \Phi_2(x_3,x_4) + \Phi_3(x_5,x_6) + \sum_{i=7}^{20}x_i^2$. The first six dimensions of this potential define a \textit{banana}-shaped marginal density in the first two dimensions, a nonsymmetric multimodal marginal density in the third and forth dimensions, and a bimodal marginal density in the fifth and sixth dimensions (see the right most column in Figure \ref{fig:fun_mix}).  By construction, this potential has rank $\bm{r} = (3,2,2,2,3,2,\ldots,2)$.

\paragraph{Parameters} We choose $\bm{n} = (4,2,4,4,6,6,2,\ldots,2)\in\mathbb{N}^{20}$ according to the degrees appearing in the potential. For Algorithm \ref{alg:HJB_algo} we set $T=10$, $\tau_{\max}=0.05$, $\delta_{\operatorname{proj}} = \delta_{\operatorname{rank}} = 0.01$, $\delta_{\operatorname{contr}} = 10^{-8}$. To account for the high stiffness of the equation at small time $t\ll 1$, we set the stiffness parameter $\rho$ in Algorithm~\ref{alg:HJB_algo} to $\rho = 0.001$ as long as $t < 10^{-6}$ and $\rho=0.5$ for $t \geq 10^{-6}$. Langevin postprocessing (see Algorithm \ref{alg:reverse_sampling}) is performed with $L=100$ steps and stepsize $\tau=0.005$.

\paragraph{Evaluation} While the rank between independent parts of the density does not increase under the HJB flow, the initial ranks $r_1 = r_5 = 3$ may increase due to the time stepping scheme and hence incur a truncation error. However, with the specified values for $\rho$ we discover that the stepsize resulting from the stiffness criterion \eqref{eq:max_stable_step} satisfies both the projection and the truncation criterion \eqref{eq:tau_proj}, \eqref{eq:tau_rank} with the requested tolerance, suggesting that the solver keeps these errors sufficiently small. Figure \ref{fig:fun_mix_stepsize_plot} shows these stepsizes with a jump around $t=10^{-6}$ due to the increase in the stiffness control parameter $\rho$. Figure \ref{fig:fun_mix_convergence_plot} shows the exponential decay in the covariance error \eqref{eq:cov_error} between the HJB solution and the standard normal distribution. Note that there is an initial spike in the error for small times $t$. In experimentation, this spike seems to decrease in magnitude when permitting higher polynomial degrees. Hence, we can attribute it to a discretization error. The optimal choice of permitted degrees to balance accuracy and computational feasibility is an open question at this point. We conjecture that it is at this point that future research will prove most fruitful: the difficult region close to $t=0$, where the true solution of the HJB is far away from the standard normal potential. Finally, Figure \ref{fig:fun_mix} shows the densities corresponding to the HJB solution obtained by Algorithm \ref{alg:HJB_algo} and the samples at the corresponding time points in the reverse process. We note that the curvature of the banana potential in the first two dimensions as well as the multimodalities in higher dimensions are recovered by the method. Finally, we note the large number of postprocessing steps used in this example. We observed a drastic decrease in sample quality for less postprocessing steps. 
\ifnum\classstyle=1
This observation motivates the discussion in the next section.
\fi


\begin{figure}
    \centering\includegraphics[width=0.95\textwidth]{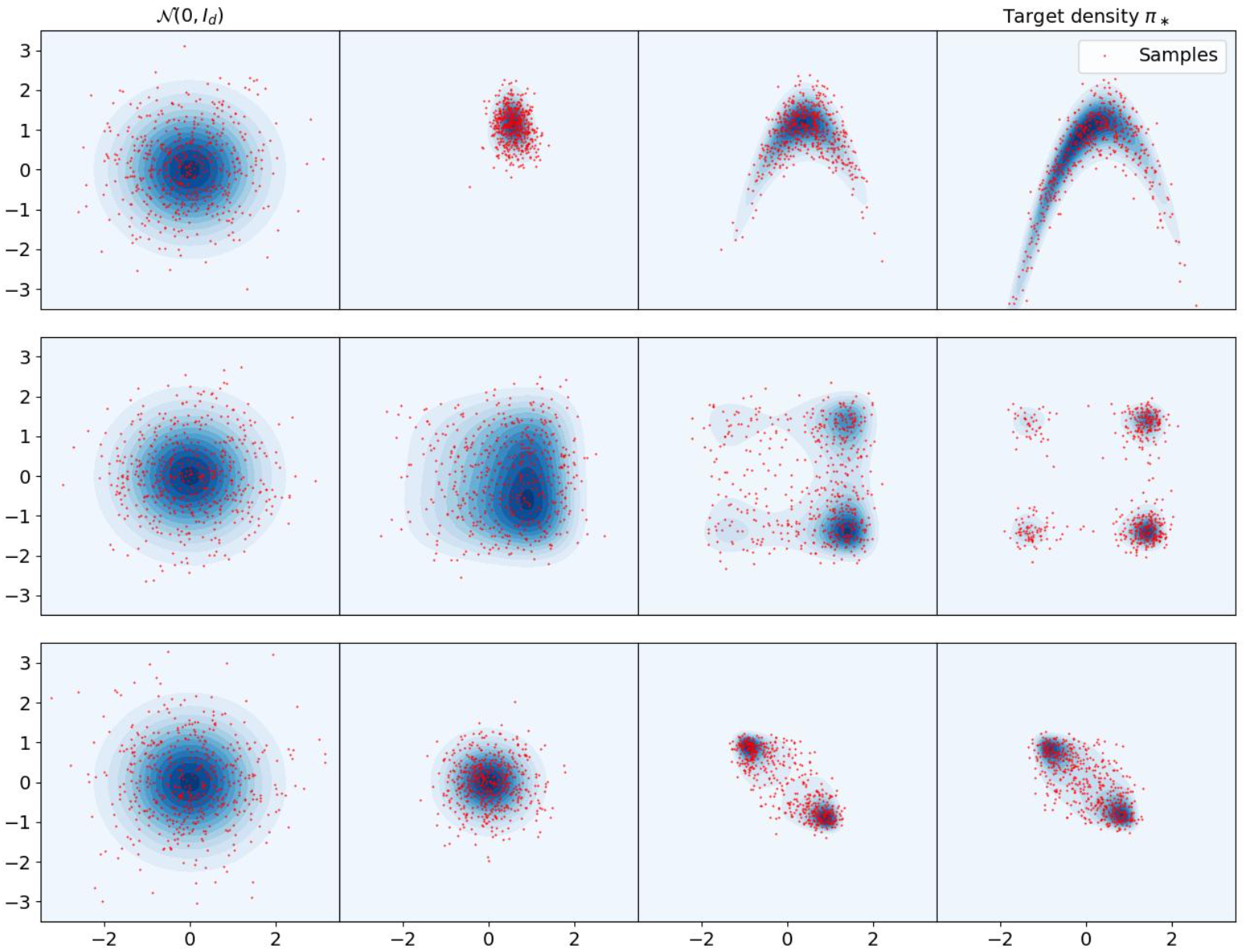}
    \caption{\textit{Development of marginal densities (\textcolor{blue}{blue}) and the samples produced by the corresponding reverse process defined by Algorithm \ref{alg:reverse_sampling} (\textcolor{red}{red}) in the setting of the mixed nonlinear case (Section \ref{sec:mixed_nonlinear}). The first row shows the values of the densities and samples on the $(x_1,x_2)$-plane, which is governed by the the Banana potential. The second row concerns the $(x_3,x_4)$-plane, which is governed by the nonsymmetric multimodal potential. The third row shows the $(x_5,x_6)$-dimension, governed by the bimodal potential. On the level of the HJB solver, the plot should be viewed \textit{from right to left} since the target density (right) is transformed to a standard Gaussian (left). On the level of the reverse process, the samples (\textcolor{red}{red}) move from the standard Gaussian on the left to the target measure on the right. We note that in all cases the sampler is able to reproduce the multimodality and curvature of the corresponding density.}}
    \label{fig:fun_mix}
\end{figure}

\begin{figure}
    \centering
    \includegraphics[width=\textwidth]{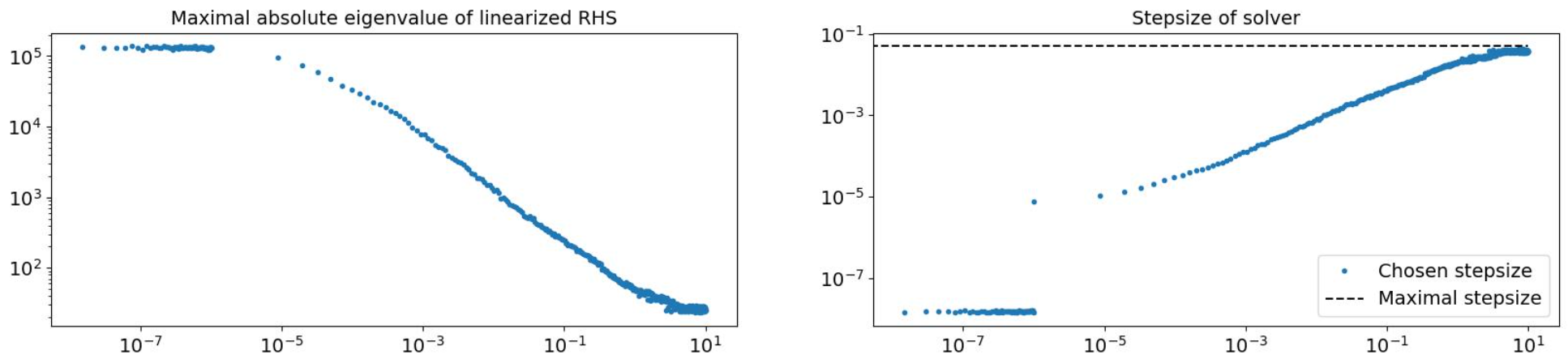}
    \caption{\textit{pproximations of the maximal absolute eigenvalues of the linearized HJB right-hand side (left) and accordingly chosen time stepsizes (right) as in Figure \ref{fig:gaussian_stepsize_plot} but for the mixed nonlinear potential from Section \ref{sec:mixed_nonlinear}. Note the jump in the stepsize at $t=10^{-6}$ which corresponds to a change in the stiffnes control parameter $\rho$. Up to small perturbations which may be attributed to inaccuracy of the power method the stepsizes are monotonically increasing again.}}
    \label{fig:fun_mix_stepsize_plot}
\end{figure}

\begin{figure}
    \centering
    \includegraphics[width=0.6\textwidth]{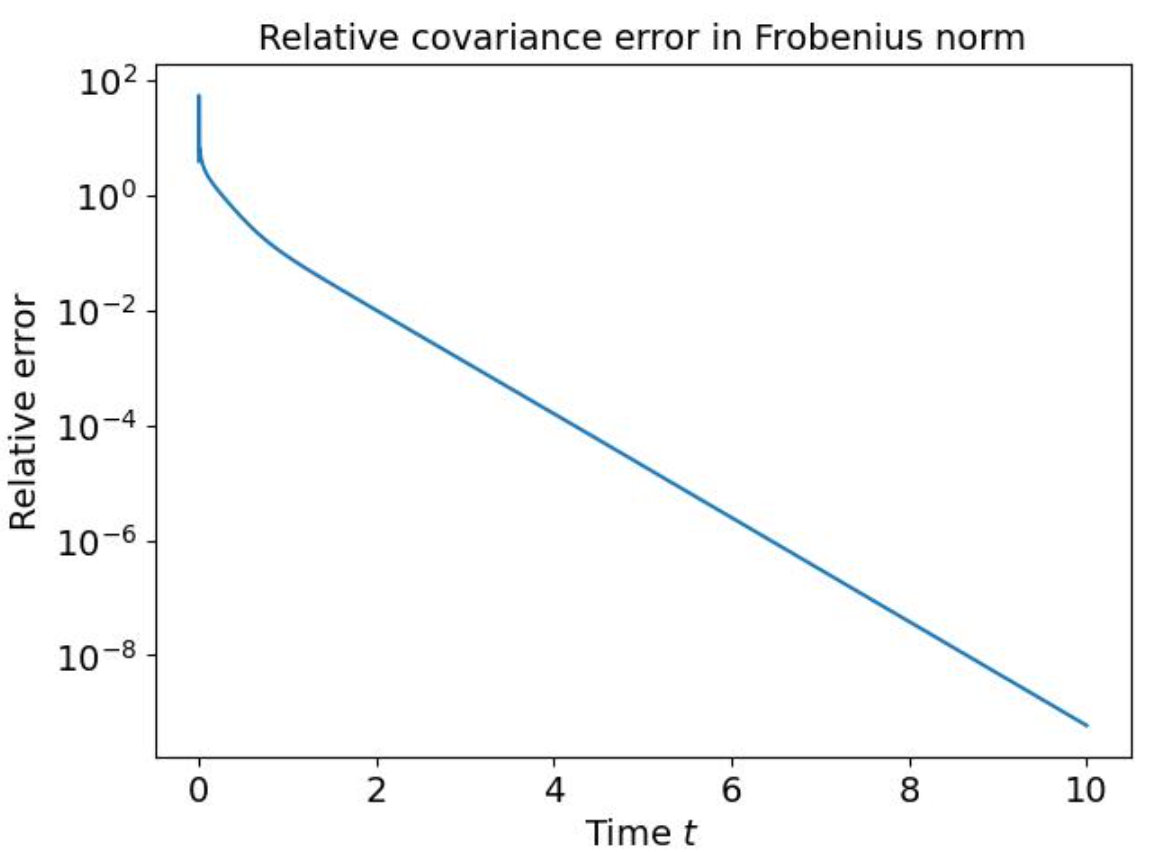}
    \caption{\textit{Decay of the covariance error \eqref{eq:cov_error} for the mixed nonlinear potential of Section \ref{sec:mixed_nonlinear}. Note that after an initial spike, which may be attributed to degree and rank increase of the true HJB solution, the error decays exponentially.}}
    \label{fig:fun_mix_convergence_plot}
\end{figure}

\subsection{Sensitivity of score-based sample generation under perturbations}

In \cite{song2020score} it was pointed out that using the flow-ODE instead of the reverse SDE, we actually can rely on high-order integration schemes and thus reduce the time discretization error significantly. However, both the reverse SDE and the flow ODE are unstable under perturbations. In fact, while the magnitude of perturbations decreases exponentially under the forward process, it can grow exponentially under the reverse process. This can be easily seen in the linear case. Assume that $\mu_\ast = \mathcal{N}(0,\Sigma)$ for $\Sigma = cI$, where $c>1$. Then for all $t\geq 0$, the law of \eqref{eq:Ornstein_Uhlenbeck} satisfies $\mu_{X_t} = \mathcal{N}(0,\Sigma_t)$ for $\Sigma_t = \Sigma e^{-2t} + (1-e^{-2t})I = (1+(c-1)e^{-2t})I$.
Hence for the flow ODE ($\lambda = 1$) and for the reverse SDE ($\lambda =0)$, the dynamic (respective drift) term has the form 
$$
f(t,X_t) = X_t + (2-\lambda) \nabla \log\pi_{T-t}(X_t) = (I - (2-\lambda)\Sigma^{-1}_{T-t})X_t.
$$
The eigenvalues of the linear dynamic of the ODE ($\lambda = 1$) are $1-(1 + (c-1)e^{-2(T-t)})^{-1} > 0$ for all $t$. Hence, perturbations in the initial data will grow exponentially in magnitute.

In practice, a good sample quality for the generation process relies on control of the time discretization, the error made in the score approximation and the accuracy of the start distribution. While the latter can be reduced by letting samples drawn from $\mathcal{N}(0,I)$ be transported to the terminal distribution of the forward process, $\mu_{X_T}$, where by construction $\mu_{X_T}\approx \mathcal{N}(0,I)$, 
the approximation error of the score still has negative effect on the sampling quality. 

In order to illustrate this, we consider an example independent of the low-rank approximation, that allows  for exact error control with respect to an accessible exact reference solution. For this, consider the two dimensional example with potential $\Phi_2(x,y) = x^4 + y^4 -4x^2 -4y^2 - 0.4x+0.1y+8$ from section \ref{sec:mixed_nonlinear}, i.e. $\mathrm{d}\mu_\ast(x) = \frac{1}{Z}\exp(-\Phi_2(x))\mathrm{d}x = \pi_\ast(x) \mathrm{d}x$. Since the forward process is an Ornstein-Uhlenbeck process, the exact density is obtained by convolution, in particular 
\begin{equation}
\label{eq:exact_pi}
\pi_t(x) = \int\limits_{\mathbb{R}^2} \pi_t(x | x_0) \pi_\ast(x_0)\mathrm{d}x_0.
\end{equation}
with the transition density 
$$
\pi_t(x | x_0)  = \frac{1}{\sqrt{ | 2\pi \Sigma(t)|}} e^{-\frac{1}{2}(x-M(t,x_0))^T \Sigma(t)^{-1}(x-M(t,x_0)) }
$$
with 
$$
M(t,x_0) = e^{-tI}x_0 = e^{-t}x_0, \quad 
\Sigma(t)  = \int\limits_{0}^t  e^{(s-t)I} \sigma\sigma^T e^{(s-t)I}\mathrm{d}s 
= (1 - e^{-2t})I.
$$
Now consider for $Q\in\mathbb{N}$ the approximation of $\pi_t$ through $\pi_t^Q$ defined as 
$$
\pi_t^Q(x) := \sum\limits_{i,j=1}^Q w_{ij} \pi_t(x|x_{ij}) \pi_\ast(x_{ij})
$$
via Gaussian quadrature encoded in the weights $(w_{ij})$ and abcissas $(x_{ij})$ on a suitable large tensor domain in $\mathbb{R}^2$ containing most of the support of $\pi_\ast$. 
We then define the approximation $v_t^M$ of the exact value function $v_t$ as 
$$
v_t(x) = -\log \pi_t(x) \approx v_t^Q(x) := -\log \pi_t^Q(x)
$$
In the experiment we consider the case of $Q=3$ for low-accuracy, $Q=10$ for medium accuracy and $Q=50$ for high-accuracy on the domain $[-5,5]^2$. As can be seen in Figure~\ref{fig:instability_sample}, the sample quality is very poor in the situation of low or medium accuracy. In our experiments the effect of perturbations due to low accuracy could be mitigated by Langevin postprocessing. 

\begin{figure}
    \begin{minipage}[t]{0.3\linewidth}
  \begin{center}
  \begin{tikzpicture}
 \node at (0,0) {\includegraphics[width = 0.8\linewidth]{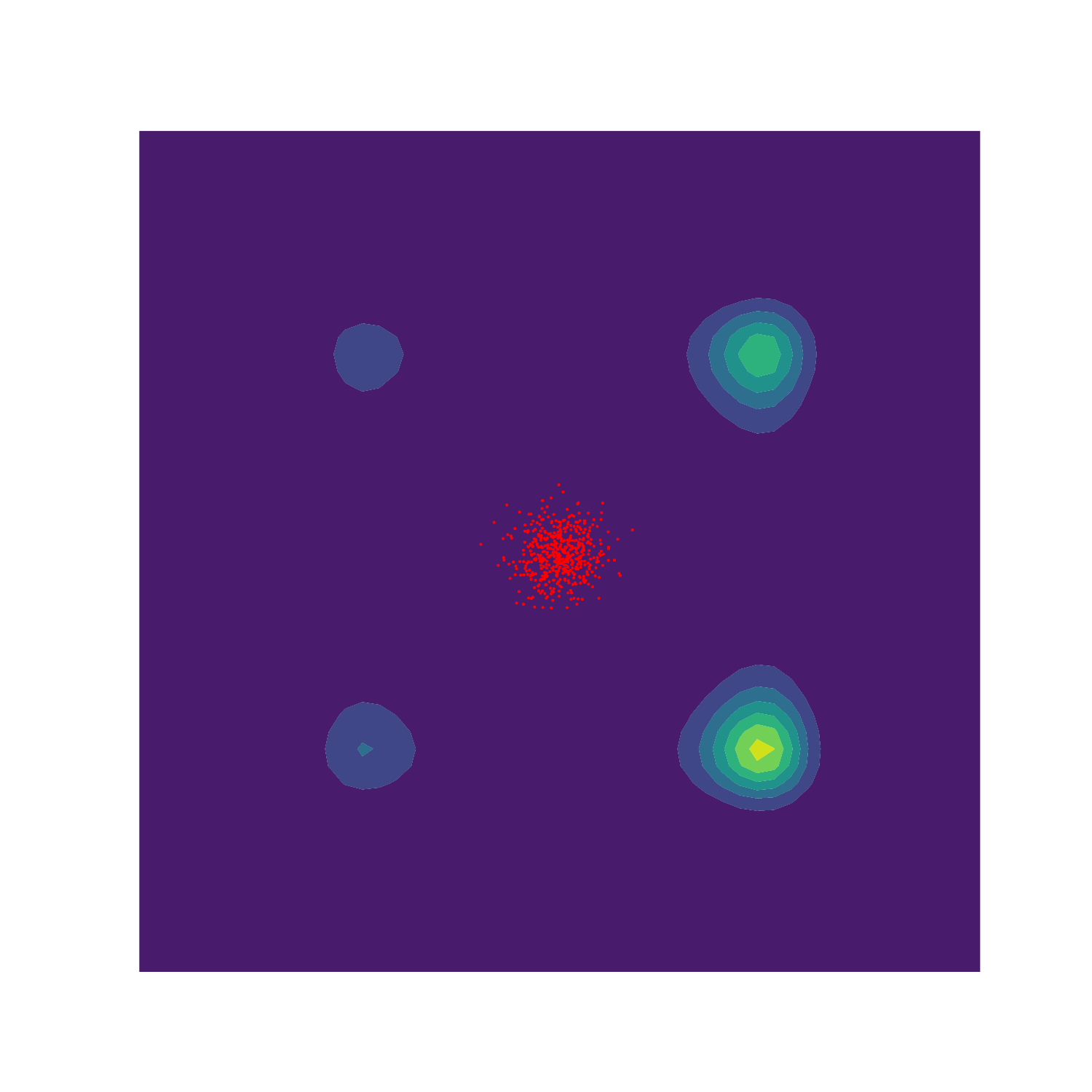}};
 \node at (0,5em)  {\small{\text{ low acc}}};
  \node at (0,-10em) {\includegraphics[width = 0.8\linewidth]{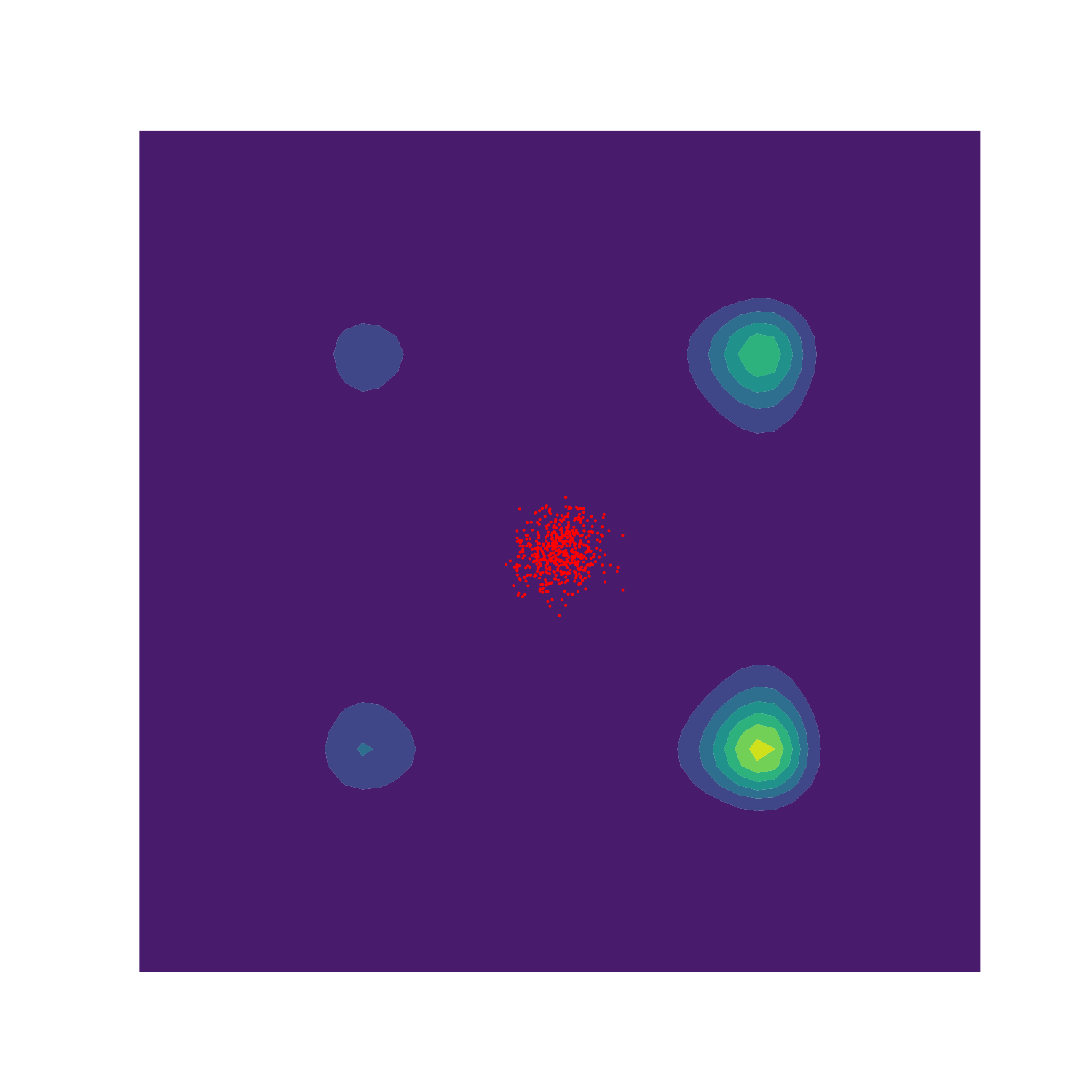}}; 
  \end{tikzpicture}
  \end{center}
 \end{minipage}
 \begin{minipage}[t]{0.3\linewidth}
  \begin{center}
    \begin{tikzpicture}
      \node at (0,0) {\includegraphics[width = 0.8\linewidth]{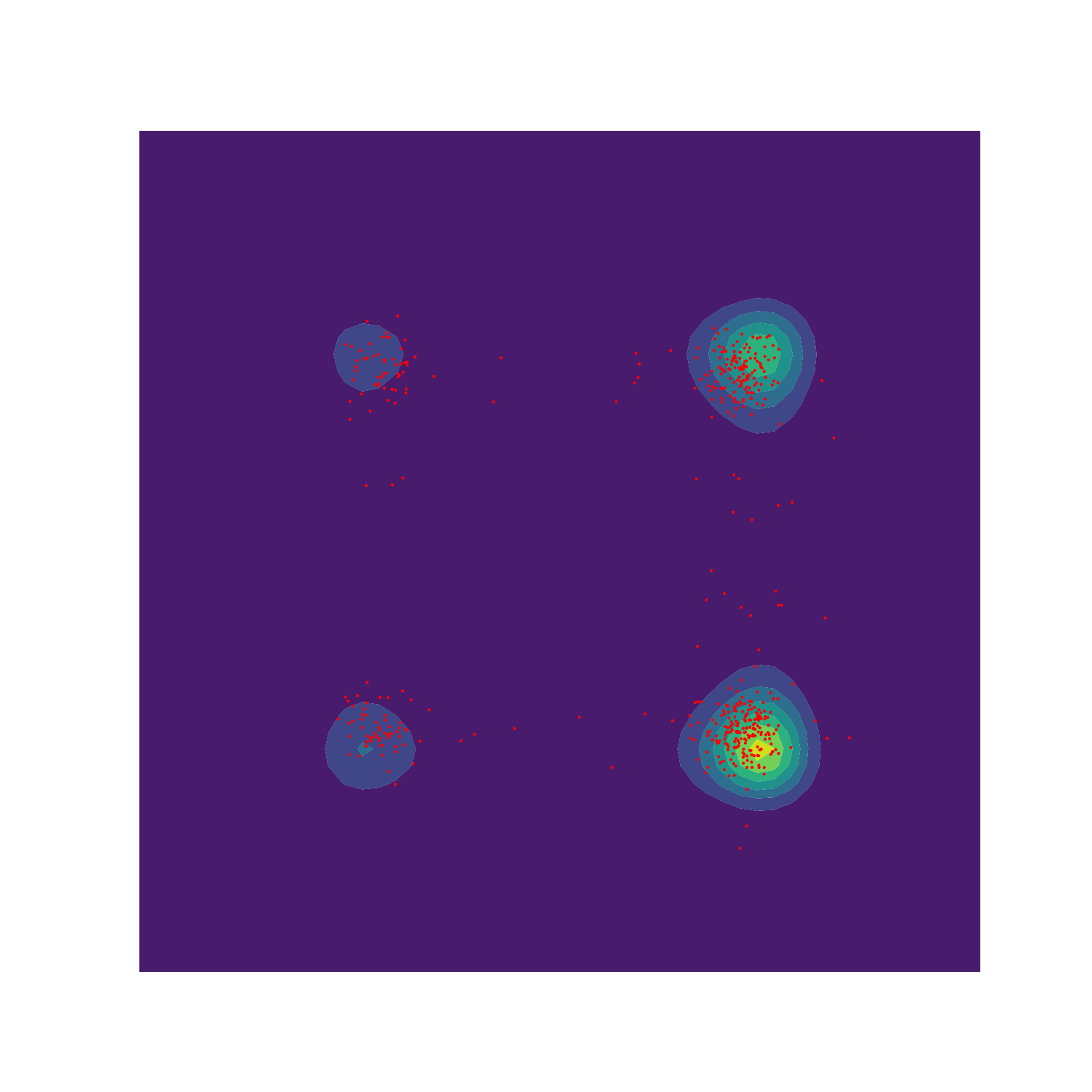}};
      \node at (0,5em)  {\small{\text{medium acc}}};
       \node at (0,-10em) {\includegraphics[width = 0.8\linewidth]{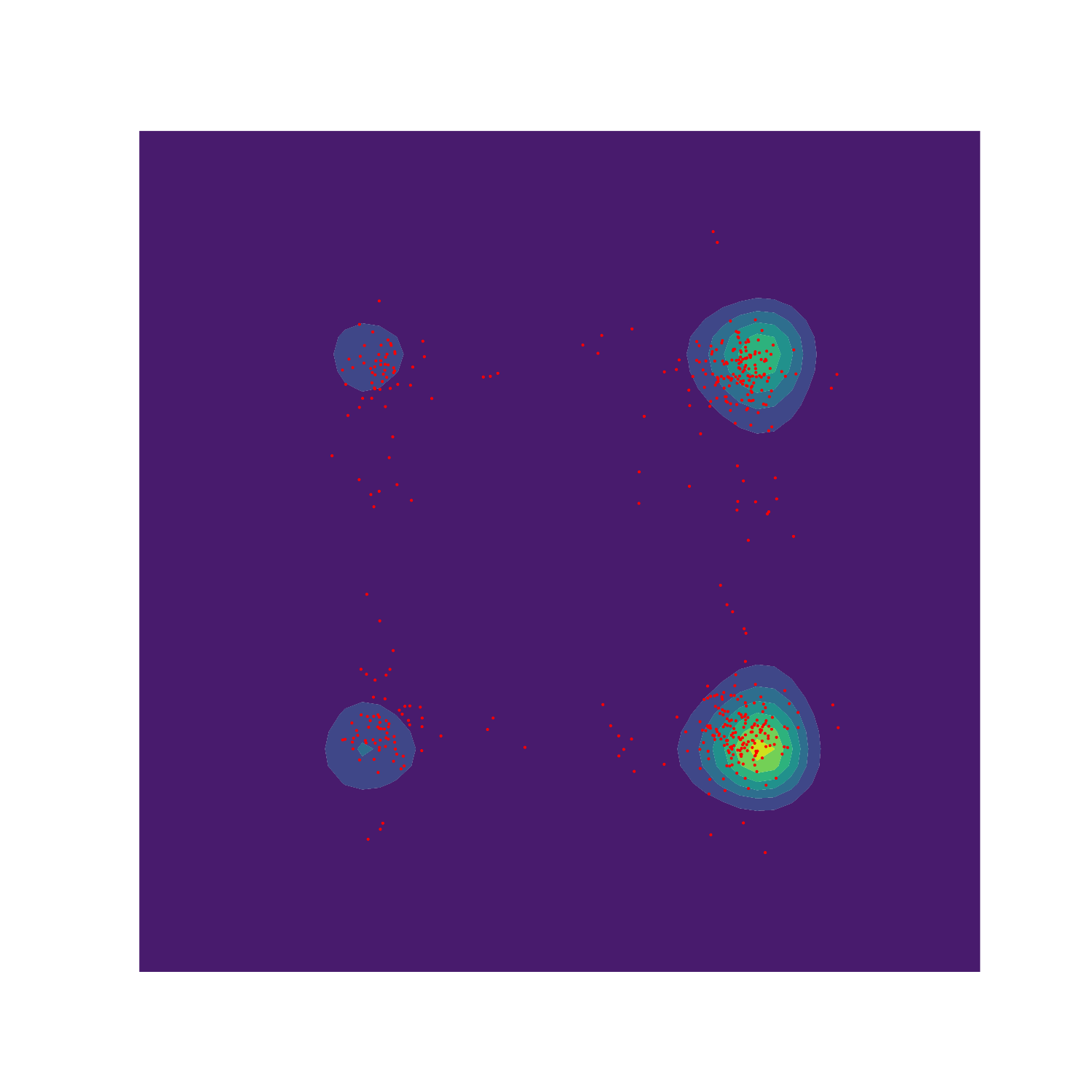}};
       \end{tikzpicture}
\end{center}
\end{minipage}
  \begin{minipage}[t]{0.3\linewidth}
    \begin{center}
      \begin{tikzpicture}
        \node at (0,0) {\includegraphics[width = 0.8\linewidth]{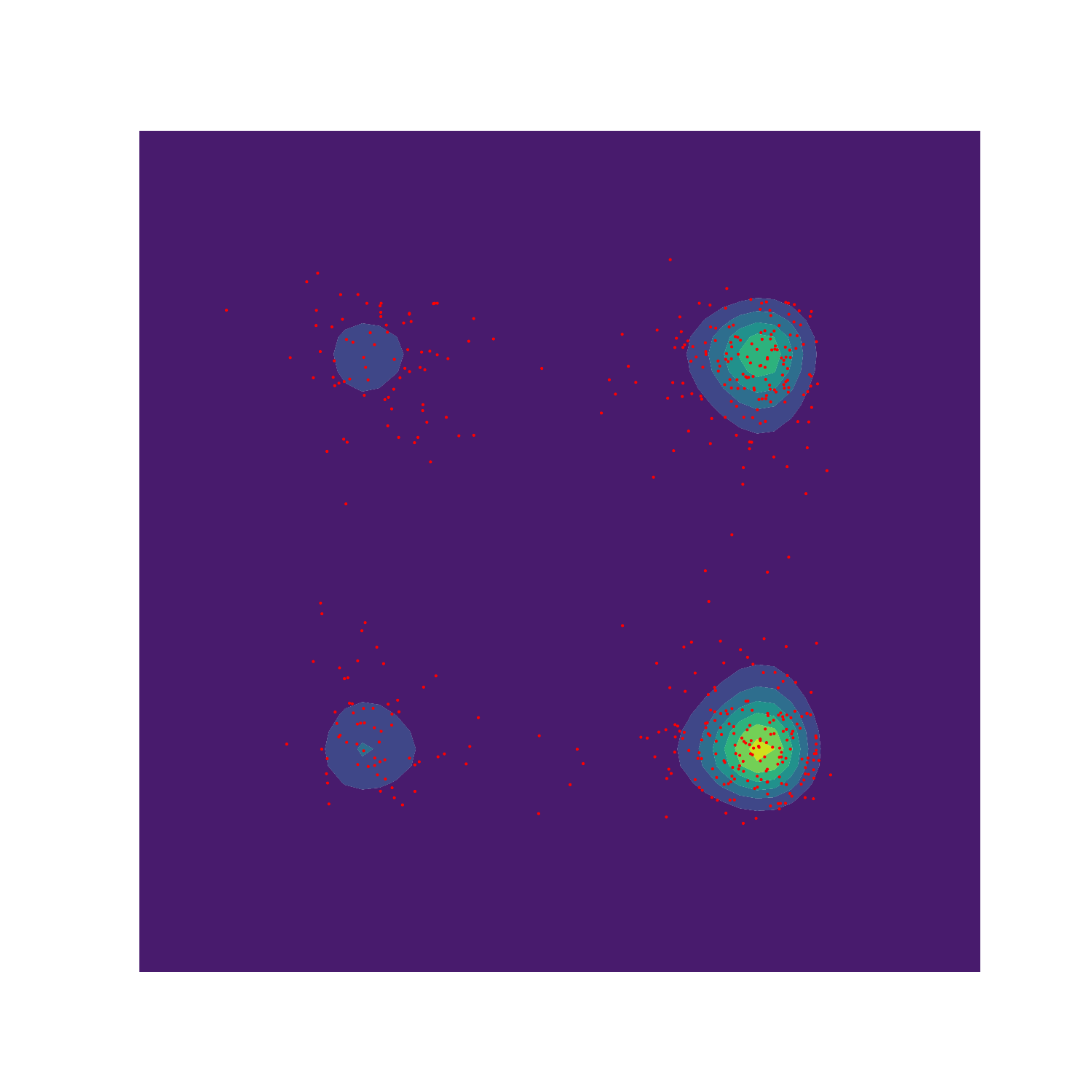}};
        \node at (0,5em)  {\small{\text{high acc}}};
        \node at (0,-10em) {\includegraphics[width = 0.8\linewidth]{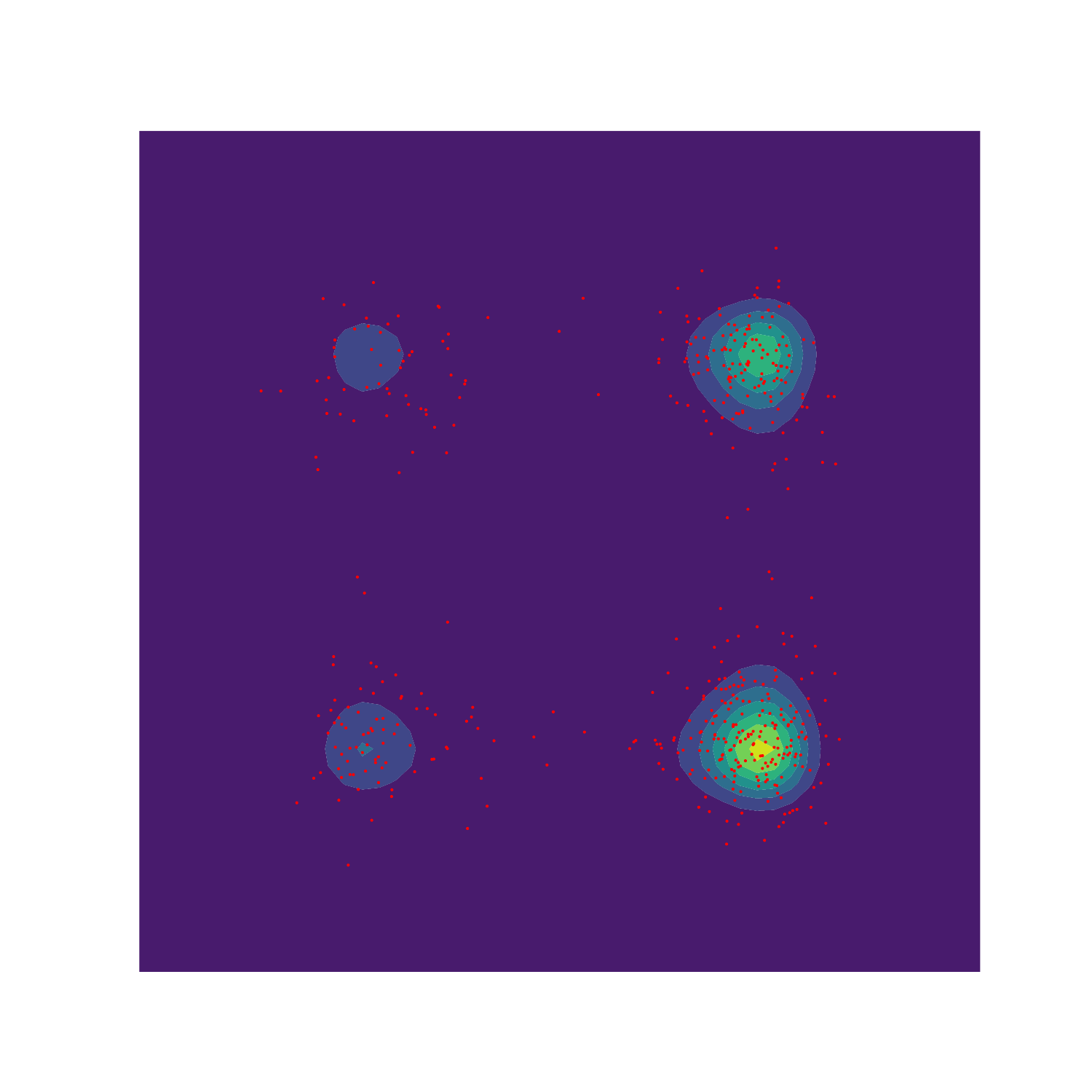}};
         \end{tikzpicture}
  \end{center}
  \end{minipage}
\caption{\textit{Sample quality illustration for different accuracies for the sample generation through the flow ODE (top row) and the reverse SDE (bottom row).}}
\label{fig:instability_sample}
\end{figure}
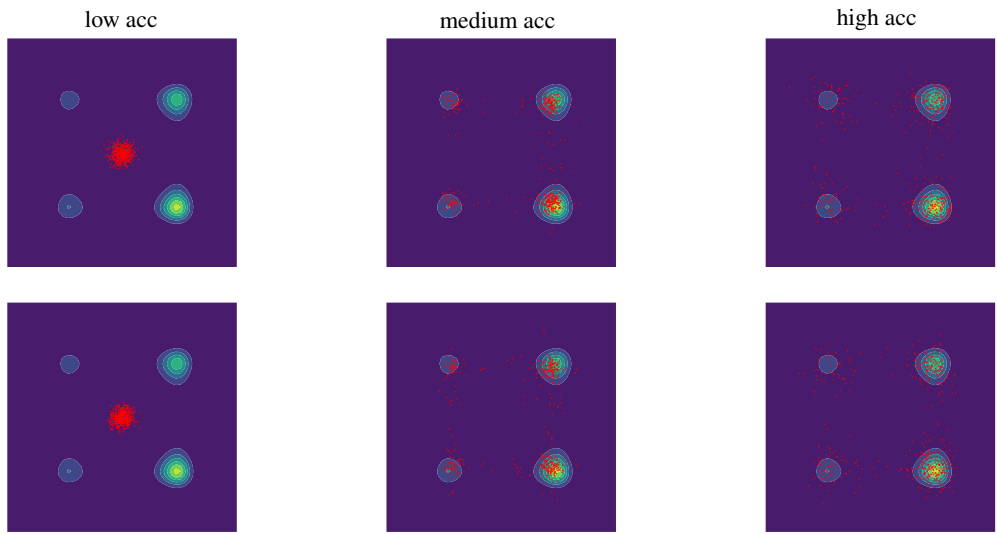
Motivated by this observation, we suggest an alternative to the reverse SDE sampling, given by the sampling process
\begin{equation}
\label{eq:homotopy_langevin}
    \mathrm{d}X_t = -\nabla \log\pi_{\max\{T-\alpha(t),0\}}(X_t)\mathrm{d}t+\sqrt{2}\mathrm{d}W_t,
\end{equation}
where $\pi_0=\pi_\ast$, and $\alpha\colon [0,\infty)\to [0,T]$ is some suitably chosen time dilation. A piecewise constant $\alpha$ corresponds to the Langevin postprocessing employed in Algorithm \ref{alg:reverse_sampling}. For a high number of postprocessing steps, the reverse-SDE step arguably becomes numerically negligible, and the reverse sampling process approximately follows \eqref{eq:homotopy_langevin}. 
The SDE in \eqref{eq:homotopy_langevin} is an instance of homotopy based Langevin dynamics developed in \cite{eigel2022interaction}. 
Consider again the linear example $\mu_{X_t}=\mathcal{N}(0,\Sigma_t)$ from the beginning of this section. Then, the drift term has the form
$$
-\nabla \log\pi_{\max\{T-\alpha(t),0\}} = -\Sigma_{\max\{T-\alpha(t),0\}}
$$
which uniformly provides negative eigenvalues, hence stability is to be expected. This serves as a motivation for the Langevin postprocessing employed in the sampling algorithm. 
\section{Conclusion and Outlook}

We presented an interpretable solver for the HJB equation arising from Hopf-Cole transformation of the Fokker-Planck equation in the setting of Bayesian inference and Generative Modelling. The approach uses functional Tensor Trains and spatial discretization with Legendre polynomials. A surrogate replacement for the HJB equation, which reduces to an ODE on tensor space, was derived. The applicability of the method was demonstrated on linear and nonlinear test cases.

There are some obvious avenues for future work. 
\begin{itemize}
    \item Incorporating more recent state-of-the-art dynamical low rank integrators for matrix valued ODEs such as \cite{ceruti2022rank,ceruti2023parallel} to the Tensor Train setting could lead to substantial performance improvements of the proposed method. In particular, the Basis
Update $\&$ Galerkin (BUG) integrator \cite{ceruti2022rank} introduces rank adaptivity, while the fully parallel integrator \cite{ceruti2023parallel} could additionally greatly speed up computations in high dimensions.
\item Sampling from the reverse process via an Euler-Maruyama discretization usually requires a small step size and a high number of time steps. In a recent work \cite{zhang2023fast}, a \textit{Diffusion Exponential Integrator Sampler} (DEIS) was proposed, which utilizes the semilinear structure of the learned diffusion process \eqref{eq:reverse_process} to reduce the discretization error. This integrator could be applied in our setting. In particular, the combination with recent dynamical low rank solvers such as \cite{ceruti2023parallel} could lead to a greatly reduced number of necessary steps both in solving the HJB as well as in discretizing the reverse process.
\item We provided results for the FTT rank structure of the HJB solution in case of Gaussian distributions 
\ifnum\online=1
(Lemma B.1 in the Online Supplementary Material)
\else
(Lemma \ref{lem:rank_gauss})\fi 
and distributions with independent components 
\ifnum\online=1
(Lemma B.2 in the Online Supplementary Material)
\else
(Lemma \ref{lem:rank_independent})\fi 
but it is an open question if there are further rank structures that are preserved under the HJB flow. As a fist step, 
\ifnum\online=1
Lemma B.2 in the Online Supplementary Material
\else
Lemma \ref{lem:rank_independent}\fi 
can be generalized to independence between groups of components: Let $f(x) = f_1(x_1,\ldots,x_n) + f_2(x_{n+1},\ldots,x_d)$. Then the FTT ranks of both $f$ and $\operatorname{Lin}(f) + \operatorname{NonLin}(f)$ satisfy $r_n\leq 2$. We conjecture that there are further situations in which the solution ranks can be bounded:
   \begin{center}
    \textit{Conjecture: 
    The HJB flow FTT ranks $\bm{r}_t$ are (up to a constant) bounded by  $\bm{r}_0$ of  $v_0$ and $\bm{r}_\infty\equiv 2$ for $v_\infty$. }
    \end{center}
This analysis is part of investigations in a subsequent work.
\item A rigorous analysis providing error estimates between the solution of the projected equation \eqref{eq:projected_hjb} and the solution of the HJB equation \eqref{eq:shifted_hjb} needs to be carried out. For a Gaussian potential, the solution of \eqref{eq:projected_hjb} coincides with that of \eqref{eq:shifted_hjb}. For more general densities the quality of the approxmiation largely depends on the initial condition, the contraction properties of the right-hand side of the HJB equation and the projection error. 
\item Finally, in the standard reverse scheme, sample quality depends on three factors: initial measure error, accuracy of the score and the time discretization. While the initial measure error can be treated by methods such as predictor-corrector schemes, the other sources of error can only be mitigated by increased computational effort. Langevin postprocessing provides a cheaper alternative, improving the quality of the samples (up to a certain level) without requiring higher accuracy of the score. We therefore suggested a homotopy approach in \eqref{eq:homotopy_langevin}. The sample quality achieved by this method, especially in the case of low accuracy of the score, will be investigated in future work.
\end{itemize}
\section{Acknowledgements}

DS \& ME acknowledge support by the Profit project \textit{ReLkat - Reinforcement Learning for complex automation engineering} as well as support by the ANR-DFG project \textit{COFNET: Compositional functions networks - adaptive learning for high-dimensional approximation and uncertainty quantification}. 
RG, ME \& CS acknowledge support by the DFG MATH+ project AA5-5 (was EF1-25) -
\textit{Wasserstein Gradient Flows for Generalised Transport in Bayesian Inversion}.
ME acknowledges partial funding by the DFG priority program SPP 2298 ``Theoretical Foundations of Deep Learning''.
This study does not have any conflicts to disclose.

\ifnum\classstyle=0
\bibliographystyle{abbrv}
\bibliography{ref}
\fi
\ifnum\classstyle=1
\bibliographystyle{abbrv}
\bibliography{ref}
\fi
\ifnum\classstyle=2
\bibliography{ref}
\fi
\ifnum\classstyle=3
\bibliographystyle{siamplain}
\bibliography{ref}

\begin{thebibliography}{10}

\bibitem{ANDERSON1982313}
B.~D. Anderson.
\newblock Reverse-time diffusion equation models.
\newblock {\em Stochastic Processes and their Applications}, 12(3):313--326,
  1982.

\bibitem{bachmayr2023low}
M.~Bachmayr.
\newblock Low-rank tensor methods for partial differential equations.
\newblock {\em Acta Numerica}, 32:1--121, 2023.

\bibitem{bachmayr2021approximation}
M.~Bachmayr, A.~Nouy, and R.~Schneider.
\newblock Approximation by tree tensor networks in high dimensions: Sobolev and
  compositional functions.
\newblock {\em arXiv preprint arXiv:2112.01474}, 2021.

\bibitem{berner2020numerically}
J.~Berner, M.~Dablander, and P.~Grohs.
\newblock Numerically solving parametric families of high-dimensional
  kolmogorov partial differential equations via deep learning.
\newblock {\em Advances in Neural Information Processing Systems},
  33:16615--16627, 2020.

\bibitem{berner2022optimal}
J.~Berner, L.~Richter, and K.~Ullrich.
\newblock An optimal control perspective on diffusion-based generative
  modeling.
\newblock {\em arXiv preprint arXiv:2211.01364}, 2022.

\bibitem{brooks2011handbook}
S.~Brooks, A.~Gelman, G.~Jones, and X.-L. Meng.
\newblock {\em Handbook of Markov Chain Monte Carlo}.
\newblock CRC press, 2011.

\bibitem{reich2022kalman}
E.~Calvello, S.~Reich, and A.~M. Stuart.
\newblock Ensemble {K}alman methods: A mean field perspective, 2022.

\bibitem{ceruti2022rank}
G.~Ceruti, J.~Kusch, and C.~Lubich.
\newblock A rank-adaptive robust integrator for dynamical low-rank
  approximation.
\newblock {\em BIT Numerical Mathematics}, 62(4):1149--1174, 2022.

\bibitem{ceruti2023parallel}
G.~Ceruti, J.~Kusch, and C.~Lubich.
\newblock A parallel rank-adaptive integrator for dynamical low-rank
  approximation, 2023.

\bibitem{cohen2011analytic}
A.~Cohen, R.~Devore, and C.~Schwab.
\newblock Analytic regularity and polynomial approximation of parametric and
  stochastic elliptic pde's.
\newblock {\em Analysis and Applications}, 9(01):11--47, 2011.

\bibitem{dolgov2021}
S.~Dolgov, D.~Kalise, and K.~K. Kunisch.
\newblock Tensor decomposition methods for high-dimensional
  hamilton--jacobi--bellman equations.
\newblock {\em SIAM Journal on Scientific Computing}, 43(3):A1625--A1650, 2021.

\bibitem{eigel2022interaction}
M.~Eigel, R.~Gruhlke, and D.~Sommer.
\newblock Less interaction with forward models in langevin dynamics, 2022.

\bibitem{eigel2023dynamical}
M.~Eigel, R.~Schneider, and D.~Sommer.
\newblock Dynamical low-rank approximations of solutions to the
  hamilton--jacobi--bellman equation.
\newblock {\em Numerical Linear Algebra with Applications}, 30(3):e2463, 2023.

\bibitem{garbuno2020interacting}
A.~Garbuno-Inigo, F.~Hoffmann, W.~Li, and A.~M. Stuart.
\newblock Interacting {L}angevin diffusions: {G}radient structure and ensemble
  {K}alman sampler.
\newblock {\em SIAM Journal on Applied Dynamical Systems}, 19(1):412--441,
  2020.

\bibitem{garbuno2020affine}
A.~Garbuno-Inigo, N.~Nusken, and S.~Reich.
\newblock Affine invariant interacting {L}angevin dynamics for {B}ayesian
  inference.
\newblock {\em SIAM Journal on Applied Dynamical Systems}, 19(3):1633--1658,
  2020.

\bibitem{goodman2010ensemble}
J.~Goodman and J.~Weare.
\newblock Ensemble samplers with affine invariance.
\newblock {\em Communications in applied mathematics and computational
  science}, 5(1):65--80, 2010.

\bibitem{griebel2023low}
M.~Griebel, H.~Harbrecht, and R.~Schneider.
\newblock Low-rank approximation of continuous functions in sobolev spaces with
  dominating mixed smoothness.
\newblock {\em Mathematics of Computation}, 92(342):1729--1746, 2023.

\bibitem{ho2020denoising}
J.~Ho, A.~Jain, and P.~Abbeel.
\newblock Denoising diffusion probabilistic models.
\newblock {\em Advances in Neural Information Processing Systems},
  33:6840--6851, 2020.

\bibitem{holtz2012alternating}
S.~Holtz, T.~Rohwedder, and R.~Schneider.
\newblock The alternating linear scheme for tensor optimization in the tensor
  train format.
\newblock {\em SIAM Journal on Scientific Computing}, 34(2):A683--A713, 2012.

\bibitem{holtz2012manifolds}
S.~Holtz, T.~Rohwedder, and R.~Schneider.
\newblock On manifolds of tensors of fixed tt-rank.
\newblock {\em Numerische Mathematik}, 120(4):701--731, 2012.

\bibitem{huang2021variational}
C.-W. Huang, J.~H. Lim, and A.~C. Courville.
\newblock A variational perspective on diffusion-based generative models and
  score matching.
\newblock {\em Advances in Neural Information Processing Systems},
  34:22863--22876, 2021.

\bibitem{hyvarinen2005estimation}
A.~Hyv{\"a}rinen and P.~Dayan.
\newblock Estimation of non-normalized statistical models by score matching.
\newblock {\em Journal of Machine Learning Research}, 6(4), 2005.

\bibitem{lubich2007}
O.~Koch and C.~Lubich.
\newblock Dynamical low‐rank approximation.
\newblock {\em SIAM Journal on Matrix Analysis and Applications},
  29(2):434--454, 2007.

\bibitem{liu2016stein}
Q.~Liu and D.~Wang.
\newblock Stein variational gradient descent: A general purpose {B}ayesian
  inference algorithm.
\newblock {\em Advances in neural information processing systems}, 29, 2016.

\bibitem{lubich2013}
C.~Lubich and I.~Oseledets.
\newblock A projector-splitting integrator for dynamical low-rank
  approximation.
\newblock {\em BIT Numerical Mathematics}, 54:171--188, 2013.

\bibitem{lubich2015}
C.~Lubich, I.~V. Oseledets, and B.~Vandereycken.
\newblock Time integration of tensor trains.
\newblock {\em SIAM Journal on Numerical Analysis}, 53(2):917--941, 2015.

\bibitem{markowich2000trend}
P.~A. Markowich and C.~Villani.
\newblock On the trend to equilibrium for the fokker-planck equation: an
  interplay between physics and functional analysis.
\newblock {\em Mat. Contemp}, 19:1--29, 2000.

\bibitem{nusken2021solving}
N.~N{\"u}sken and L.~Richter.
\newblock Solving high-dimensional hamilton--jacobi--bellman pdes using neural
  networks: perspectives from the theory of controlled diffusions and measures
  on path space.
\newblock {\em Partial differential equations and applications}, 2:1--48, 2021.

\bibitem{oseledets2011tensor}
I.~V. Oseledets.
\newblock Tensor-train decomposition.
\newblock {\em SIAM Journal on Scientific Computing}, 33(5):2295--2317, 2011.

\bibitem{oseledets2013constructive}
I.~V. Oseledets.
\newblock Constructive representation of functions in low-rank tensor formats.
\newblock {\em Constructive Approximation}, 37:1--18, 2013.

\bibitem{oster2022approximating}
M.~Oster, L.~Sallandt, and R.~Schneider.
\newblock Approximating optimal feedback controllers of finite horizon control
  problems using hierarchical tensor formats.
\newblock {\em SIAM Journal on Scientific Computing}, 44(3):B746--B770, 2022.

\bibitem{reich2021fokker}
S.~Reich and S.~Weissmann.
\newblock {F}okker--{P}lanck particle systems for {B}ayesian inference:
  {C}omputational approaches.
\newblock {\em SIAM/ASA Journal on Uncertainty Quantification}, 9(2):446--482,
  2021.

\bibitem{rey2016improving}
L.~Rey-Bellet and K.~Spiliopoulos.
\newblock Improving the convergence of reversible samplers.
\newblock {\em Journal of Statistical Physics}, 164:472--494, 2016.

\bibitem{robert2011mcmc}
C.~Robert and G.~Casella.
\newblock {A Short History of Markov Chain Monte Carlo: Subjective
  Recollections from Incomplete Data}.
\newblock {\em Statistical Science}, 26(1):102 -- 115, 2011.

\bibitem{roberts2004general}
G.~O. Roberts and J.~S. Rosenthal.
\newblock General state space {M}arkov chains and {MCMC} algorithms.
\newblock {\em Probability surveys}, 1:20--71, 2004.

\bibitem{roberts1996exponential}
G.~O. Roberts and R.~L. Tweedie.
\newblock Exponential convergence of {L}angevin distributions and their
  discrete approximations.
\newblock {\em Bernoulli}, pages 341--363, 1996.

\bibitem{song2019generative}
Y.~Song and S.~Ermon.
\newblock Generative modeling by estimating gradients of the data distribution.
\newblock {\em Advances in neural information processing systems}, 32, 2019.

\bibitem{pmlr-v115-song20a}
Y.~Song, S.~Garg, J.~Shi, and S.~Ermon.
\newblock Sliced score matching: A scalable approach to density and score
  estimation.
\newblock In R.~P. Adams and V.~Gogate, editors, {\em Proceedings of The 35th
  Uncertainty in Artificial Intelligence Conference}, volume 115 of {\em
  Proceedings of Machine Learning Research}, pages 574--584. PMLR, 22--25 Jul
  2020.

\bibitem{song2020score}
Y.~Song, J.~Sohl-Dickstein, D.~P. Kingma, A.~Kumar, S.~Ermon, and B.~Poole.
\newblock Score-based generative modeling through stochastic differential
  equations.
\newblock {\em arXiv preprint arXiv:2011.13456}, 2020.

\bibitem{stuart_2010}
A.~M. Stuart.
\newblock Inverse problems: A bayesian perspective.
\newblock {\em Acta Numerica}, 19:451–559, 2010.

\bibitem{zhang2023transport}
B.~J. Zhang, Y.~M. Marzouk, and K.~Spiliopoulos.
\newblock Transport map unadjusted langevin algorithms, 2023.

\bibitem{zhang2020wasserstein}
K.~S. Zhang, G.~Peyr{\'e}, J.~Fadili, and M.~Pereyra.
\newblock Wasserstein control of mirror langevin monte carlo.
\newblock In {\em Conference on Learning Theory}, pages 3814--3841. PMLR, 2020.

\bibitem{zhang2023fast}
Q.~Zhang and Y.~Chen.
\newblock Fast sampling of diffusion models with exponential integrator.
\newblock In {\em The Eleventh International Conference on Learning
  Representations}, 2023.

\bibitem{zhou2021}
M.~Zhou, J.~Han, and J.~Lu.
\newblock Actor-critic method for high dimensional static
  hamilton--jacobi--bellman partial differential equations based on neural
  networks.
\newblock {\em SIAM Journal on Scientific Computing}, 43(6):A4043--A4066, 2021.

\end{thebibliography}
\fi

\newpage
\appendix

\ifnum\online=0
\ifnum\classstyle=3
\else
\section{Hopf-Cole Transformation}\label{app:hopf_cole}

Let $\pi_t$ satisfy the Fokker-Planck equation
\begin{align}
    \partial_t \pi_t =  \Delta \pi_t+x\cdot\nabla \pi_t + d\pi_t.
\end{align}

Then, by the chain and product rule, we have the identities
\begin{align}
    \partial_t \log\pi_t &= \frac{1}{\pi_t}\partial_t \pi_t = \frac{1}{\pi_t}(\Delta \pi_t+x\cdot\nabla \pi_t + d\pi_t), \\
    \Delta\log\pi_t &= \nabla\cdot \left(\frac{1}{\pi_t}\nabla\pi_t\right) = \frac{1}{\pi_t}\Delta\pi_t - \frac{1}{\pi_t^2}\nabla\pi_t\cdot \nabla\pi_t = \frac{1}{\pi_t}\Delta\pi_t - \|\nabla\log\pi_t\|_2^2.
\end{align}
Putting these together, we see that
\begin{align}
    \partial_t \log\pi_t = \Delta \log\pi_t + \|\nabla \log\pi_t\|_2^2 + x\cdot \nabla\log\pi_t + d
\end{align}
and hence $v_t = -\log\pi_t$ satisfies
\begin{align}
    \partial_t v_t = \Delta v_t - \|\nabla v_t\|_2^2 + x\cdot \nabla v_t - d.
\end{align}

\fi

\ifnum\classstyle=3
\else
\section{Optimal Control Perspective}\label{app:oc}
We consider the case where $\pi_{\ast}$ is a zero-mean Gaussian with symmetric positive definite covariance matrix $\Sigma$. Hence, \eqref{eq:shifted_hjb} becomes
\begin{equation}
    \begin{aligned}
        \partial_t v_t = \operatorname{Lin}(v_t) + \operatorname{NonLin}(v_t), \qquad v_0 = \dfrac{1}{2}x^{\intercal} \Sigma^{-1} x.
    \end{aligned}
\end{equation}
This is a familiar form in stochastic optimal control. Consider an SDE
\begin{align}\label{eq:SDE}
    \mathrm{d}X_t &= (f(X_t) + g(X_t)u_t)dt + \sigma(X_t)\mathrm{d}W_t, \\
    X_0 &= x_0,
\end{align}
with initial condition $x_0$, control $u_t$, diffusion $\sigma$, free drift part $f$ and controlled drift part $g\cdot u_t$. Associate with this SDE a cost functional given by
\begin{align}
    J(t,x,u) = \mathbb{E}\left[ \int_t^T \lambda u_t^{\intercal} u_t \mathrm{d}t + E(X_T) \big| X_t = x \right],\label{eq:stochastic_cost}
\end{align}
where $\lambda > 0$ and $E$ is a positive definite terminal cost function.
The associated HJB equation for the value function $V$ reads
\begin{align*}
    \partial_t V + f^{\intercal}\nabla V - \dfrac{1}{4\lambda}\nabla V^{\intercal} g g^{\intercal} \nabla V + \dfrac{\sigma^2}{2} \Delta V = 0, \qquad V(T,x) = E(x).
\end{align*}
Now, choose $f(x)=x$, $g(x) \equiv I_d$, $E(x) = x^{\intercal} \Sigma^{-1}x/2$, $\lambda = 1/4$ and $\sigma=\sqrt{2}$ to arrive at 
\begin{equation}\label{eq:hjb_lqr}
    \begin{aligned}
    \partial_t V &= -\nabla V^{\intercal}x + \| \nabla V \|^2 - \Delta V ,\\
    \qquad V(T,x) &= \frac{1}{2}x^{\intercal} \Sigma^{-1}x.
\end{aligned}
\end{equation}
Clearly, reversing the time by defining $\overleftarrow{V}(t,x) \coloneqq V(T-t,x)$ and regrouping the terms yields 
\begin{equation}
    \begin{aligned}
    \partial_t \overleftarrow{V} &=  \Delta \overleftarrow{V}+ \nabla \overleftarrow{V}^{\intercal}x   -\| \nabla \overleftarrow{V} \|^2 =  \operatorname{Lin}(\overleftarrow{V}) + \operatorname{NonLin}(\overleftarrow{V}),\\
    \qquad \overleftarrow{V}(0,x) &= \frac{1}{2}x^{\intercal} \Sigma^{-1}x,
\end{aligned}
\end{equation}
 and hence $v_t = \overleftarrow{V}(t,\cdot)$. In total, the log-density of the foward SDE is given as the reverse-time value function of a linear quadratic optimal control problem. The linear quadratic problem has the property that its solution is the same in the deterministic and stochastic setting. Hence, instead of the stochastic problem, we may also consider the deterministic optimal control problem defined by
\begin{align}
    \dot{x} &=  x +   u  , \qquad
    x(0) = x_0, \\
    J(t,x,u) &= \int_t^T \dfrac{1}{4} u(t)^{\intercal} u(t) \mathrm{d}t + \frac{1}{2}X(T)^{\intercal} \Sigma^{-1} X(T).
\end{align}
Now, setting $A=I_d$, $B=I_d$, $Q\equiv 0\in \mathbb{R}^{d\times d}$, $R = \frac{1}{4}I_d$ and $Q_f = \frac{1}{2}\Sigma_{\rho}^{-1}$ this leads to 
\begin{align}
    \dot{x} &= A x + B  u  , \qquad
    x(0) = x_0, \\
    J(t,x,u) &= \int_t^T x(t)Qx(t) + u(t)^{\intercal} R u(t) \mathrm{d}t + X(T)^{\intercal} Q_f X(T).
\end{align}
The solution of this problem is given by the LQR $V(t,x) = x^{\intercal}P_t x$, where $P_t$ solves a Riccati differential equation with inputs $A,B,Q,R,Q_f$. It follows that
\begin{align*}
    \nabla v_{T-t}(x) = \nabla V(t,x) = 2P_t x,
\end{align*}
leading for $\lambda = 0$ to a reverse-time generative process defined by
\begin{equation}\label{eq:lin_process}
    \mathrm{d}Y_s 
    = (I_d-4P_s) Y_s\mathrm{d}s + \sqrt{2}\mathrm{d}W_s.
\end{equation}
\fi

\section{Motivation for using (Functional) Tensor Trains}\label{app:tt_motivation}

We give an informal motivation for the use of FTTs and in particular polynomials represented in the Tensor Train format in the setting of Bayesian inference for paramteric PDEs. For more rigorous representations and decay rates in polynomial chaos representations of solutions of parametric PDEs we refer e.g. to pioneering work in \cite{cohen2011analytic} and follow up research. In this setting, the fact that usually only very few data points are available often renders high frequency components non-informative. Hence, the higher mode dimensions are often close to the prior information even after inference. Assuming that the prior is given as a (standard) Gaussian, it is reasonable to assume that these higher modes will be close to Gaussian. In particular, this motivates a form of potential similar to the nonlinear potential used in Section \ref{sec:mixed_nonlinear}. The following provides a sketch on how such a form might be obtained.
Let  $M_1\in\mathbb{N}$, $M_1 < d$, denote a number of relevant modes and for maximal polynomial degrees $d_1\geq d_2\geq \ldots \geq d_{M_1}\geq 2$ let
$$ \text{relevant} = \bigtimes\limits_{i=1}^{M_1} \{0,\ldots, d_i\}, \qquad \sqrt{\text{relevant}} = \bigtimes\limits_{i=1}^{M_1} \{0,\ldots, \lfloor\sqrt{d_i}\rfloor\}.$$

Observe that solutions $u$ of parameteric PDEs with spatial variable $x$ and parameter $y$ can often be written as
$$
u(x,y) \approx \sum\limits_{\bm\alpha \in \sqrt{\text{relevant}}} u_{\bm\alpha}(x)p_{\bm\alpha}(y)
$$
where $u_{\bm\alpha}$ is an element of some function space $V$ for every $\bm\alpha$.
Let $G(y) = u(\cdot, y)\in V$ and for some $K\in\mathbb{N}$ let $\mathcal{O}\colon V\to \mathbb{R}^K$ be a linear observation operator (e.g. point evaluations in $x$). Hence: 
$$
\mathcal{O}(G(y)) = \mathcal{O}(u(\cdot, y)) = 
\sum\limits_{\bm\alpha \in \sqrt{\text{relevant}}} \bm{u}_{\bm\alpha} p_{\bm\alpha}(y), \quad \bm{u}_{\bm\alpha} \in\mathbb{R}^K.
$$
Then, assuming a Bayesian setting with a zero mean Gaussian prior with covariance matrix $\Sigma$, the log posterior density has the form
$$
\log \pi(y) = -\frac{1}{2}\| \mathcal{O}(G(y)) - \delta \|_{\sigma I_d}^2 - \frac{1}{2}\|y\|_\Sigma^2
$$
where $\delta$ is an observation and $\sigma I_d$, $\sigma>0$, is the covariance of the zero mean Gaussian observational noise.
By the form of $\mathcal{O}(G(y))$, it then follows that there are coefficient tensors $c^{\mathrm{prior}}$ and $c^{\mathrm{likelihood}}$ such that the potential or negative log posterior density is of the form 
\ifnum\classstyle=1
\begin{equation}
\begin{aligned}
\Phi(y)  &=\frac{1}{2\sigma}\sum_{k=1}^K \left(\mathcal{O}(G(y)) - \delta\right)_k^2 +  \sum\limits_{|\bm\beta|=2} c^{\text{prior}}[\bm\beta] P_{\bm\beta}(y)\\
&= 
 \underbrace{
\sum\limits_{\bm\beta \in \text{relevant}} c^{\text{likelihood}}[\bm\beta]P_{\bm\beta}(y)
 +  \sum\limits_{|\bm\beta|=2, \bm\beta_i\equiv 0 ~\textnormal{for}~ i > M_1} c^{\text{prior}}[\bm\beta] P_{\bm\beta}(y)}_{\text{non-Gaussian component}}
 +  \underbrace{\sum\limits_{|\bm\beta|=2, \bm\beta_i\equiv 0 ~\textnormal{for}~ i \leq M_1} c^{\text{prior}}[\bm\beta] P_{\bm\beta}(y) }_{\text{Gaussian (uninformed) component}},
 \end{aligned}
\end{equation}
\fi
\ifnum\classstyle=2
\begin{equation}
\begin{aligned}
\Phi(y)  &=\frac{1}{2\sigma}\sum_{k=1}^K \left(\mathcal{O}(G(y)) - \delta\right)_k^2 +  \sum\limits_{|\bm\beta|=2} c^{\text{prior}}[\bm\beta] P_{\bm\beta}(y)\\
&= 
 \underbrace{
\sum\limits_{\bm\beta \in \text{relevant}} c^{\text{likelihood}}[\bm\beta]P_{\bm\beta}(y)
 +  \sum\limits_{|\bm\beta|=2, \bm\beta_i\equiv 0 ~\textnormal{for}~ i > M_1} c^{\text{prior}}[\bm\beta] P_{\bm\beta}(y)}_{\text{non-Gaussian component}}
 \\
 &~~ +  \underbrace{\sum\limits_{|\bm\beta|=2, \bm\beta_i\equiv 0 ~\textnormal{for}~ i \leq M_1} c^{\text{prior}}[\bm\beta] P_{\bm\beta}(y) }_{\text{Gaussian (uninformed) component}},
\end{aligned}
\end{equation}
\fi
with non-Gaussian parts confined to the relevant modes $1,\ldots,M_1$.


\section{Functional Tensor Train rank of HJB solutions}\label{app:ftt_ranks}

\begin{lemma}[Gaussian distributions]\label{lem:rank_gauss}
    Let $d\in\mathbf{N}$ and $f\colon\mathbb{R}^d\rightarrow\mathbb{R}$ admit the form $f(x) = x^{\intercal}M x$ for a symmetric positive definite matrix $M \in\mathbb{R}^{d,d}$. Then $f$ has finite FTT rank $\bm{r}\in\mathbb{N}^{d-1}$. In particular for $d\geq 3$,
    $$
    \bm{r} \leq \overline{\bm{r}}:= 2 + \begin{cases}
                 \left(1, 2 , \ldots, \frac{d}{2},  \ldots,2, 1\right), & d \text{ even}, \\
                \left(1, 2, \ldots, \frac{d-1}{2},\frac{d-1}{2}, \ldots, 2,1\right), & d \text{ odd},
                 \end{cases}
    $$
    and $\bm{r}=2\in\mathbb{N}$ for $d=2$.
\end{lemma}
\begin{proof}
    The case $d=2$ follows since $M$ is invertible and the TT rank coincides with the matrix rank.
    Let $d\geq 3$ and write $ M= (m_{ij})$ and $\overline{\bm{r}}=(\overline{r}_i)_{i=1}^{d-1}$, 
    $\overline{r}_0=\overline{r}_d=1$.
    We seek a representation 
    $$
    f(x) = U_1(x_1) U_2(x_2) \cdots U_d(x_d), \quad U_i(x_i)\in \mathbb{R}^{r_{i-1}, r_i}, i=1,\ldots,d.
    $$
    Let $I_{n} \in \mathbb{R}^{n,n}$ denote the identity matrix and  $\bm{0}_{k,l}\in\mathbb{R}^{kl}$ be a  zero matrix and $\bm{0}_k\in\mathbb{R}^k$ be a zero vector.

Define the matrices $\tilde{U}_i(x_i)$ for $i=1,d$ as 
    $$
    \tilde{U}_1(x_1) = \begin{pmatrix} 1  &2x_1 & m_{11}x_1^2\end{pmatrix} \in \mathbb{R}^{1,\overline{r}_1}, \qquad
    \tilde{U}_d(x_d) = \begin{pmatrix} 1  &2x_d & m_{dd}x_d^2\end{pmatrix}^\intercal \in \mathbb{R}^{\overline{r}_{d-1},1}.
    $$
Moreover, for $i=2,\ldots,d-1$ except for $i=\frac{d-1}{2}+1$ in case hat $d$ is odd let
$$
\tilde{U}_i(x_i) = 
\begin{bNiceArray}{ccccc|c|c}[margin]
\Block{4-5}{I_i}      &&  &  &  & 2x_i& m_{ii}x_i^2 \\
    &   &        & & & \Block{3-1}{\bm{0}_{i-1}}  & m_{1i}x_i   \\
    &                                              &     & &   &    & \vdots        \\
    &               &                                &        &   &  & m_{i-1,i}x_i \\
\hline
\Block{1-5}{\bm{0}_{i}^\intercal   }    &                                              &  & && 0 &  1
\end{bNiceArray} 
\in \mathbb{R}^{\overline{r}_{i-1}, \overline{r}_i},
\quad i = 2,\ldots  ,\left\lceil \frac{d-1}{2} \right\rceil,
$$

$$
\tilde{U}_i(x_i) = 
\begin{bNiceArray}{c|ccc|c}[margin]
1 & \Block{2-3}{\bm{0}_{2,d-i}}  &  &  & \Block{2-1}{\bm{0}_{2}} \\
2x_i                 &                                       &        &   &  \\
\hline
\Block{4-1}{\bm{0}_{d-i}}&\Block{4-3}{I_{d-i}}   &  &  & \Block{4-1}{\bm{0}_{d-i}} \\
      &  &  &  & \\
      &  &  &  & \\
      &  &  &  & \\
\hline
m_{i,i}x_i^2      & m_{d,i}x_i      &  \cdots&  m_{i+1,i}x_i &  1
\end{bNiceArray} 
\in \mathbb{R}^{\overline{r}_{i-1}, \overline{r}_i}
\quad i = \left\lfloor \frac{d+3}{2}\right\rfloor,\ldots,  d -1.
$$

If $d$ is odd we define the middle square component  $U_i(x_i)$ for $i=\frac{d-1}{2}+1$ by
$$
\tilde{U}_i(x_i) = 
\begin{bNiceArray}{c|ccc|c}[margin]
m_{ii}x_i^2 & m_{i,i+1}x_i & \cdots & m_{i,d}x_i  & 1 \\
\hline
m_{1,i}x_i &  \Block{3-3}{\frac{1}{2}M_{1:i-1,i+1:d}} & & & \Block{3-1}{\bm{0}_{i-1}} \\
\vdots &                          & & &   \\
m_{i-1,i}x_i &   & & &  \\
\hline 
1 & \Block{1-3}{\bm{0}_{d-i}^\intercal} & &  & 0 
\end{bNiceArray}, \quad i=\frac{d-1}{2}+1.
$$

For $d$ even we define for $i=\frac{d}{2}+1 = \left\lfloor\frac{d+3}{2}\right\rfloor$
$$
U_i(x_i) = 
\begin{bNiceArray}{c|cccccc|c}[margin]
0 & \Block{1-6}{\bm{0}_{d-i}^\intercal} &&  &&&   & 1 \\
\hline
\Block{3-1}{\bm{0}_{i-1}} &  \Block{3-6}{M_{1:i-1,i+1:d}} &&&& & & \Block{3-1}{\bm{0}_{i-1}} \\
 &                       &   & & &   &&\\
 &   & &&&& &  \\
\hline 
1 & \Block{1-6}{\bm{0}_{d-i}^\intercal}& &&& &  & 0 
\end{bNiceArray} \tilde{U}(x_i)
$$
and in any other case set $U_i(x_i)=\tilde{U}_i(x_i)$.
\end{proof}

\begin{lemma}[Measures of independent variables]\label{lem:rank_independent}
    Let $d\in\mathbb{N}$ and $f\colon\mathbb{R}^d\rightarrow\mathbb{R}$ admit the form $f(x) = \sum_{i=1}^d f_i(x_i)$ for  $f_i\in\mathcal{C}^2(\mathbb{R},\mathbb{R})$, $i=1,\ldots,d$. Then both $f$ and $\operatorname{Lin}(f)+\operatorname{NonLin}(f)$ have FTT rank $\mathbf{r}=(2,\ldots,2)^{\intercal}\in\mathbb{N}^{d-1}$.
\end{lemma}

\begin{proof}
    The result follows immediately from \cite[Theorem 2]{oseledets2013constructive} and the structure of $\operatorname{Lin}(f)+\operatorname{NonLin}(f)$.
\end{proof}


\section{Details of HJB solutions}\label{app:hjb_details}

Let $p_{\alpha}^{(\mathrm{mon})}$ for $\alpha\in\mathbb{N}_0$ denote the $\alpha$-th monomial, i.e. $p_{\alpha}^{(\mathrm{mon})}(x) = x^{\alpha}$. As in Section \ref{sec:TTs}, let $T_{i,n}\in\mathbb{R}^{n+1,n+1}$ denote the basis transformation matrix between Legendre polynomials of degree $n$ on $[a_i,b_i]$ and the monomials up to degree $n$.

\subsection{Derivation matrices in the linear operator part}\label{sec:LinOp}

Note that for every $i=1,\ldots,d$ and for every $c\in\mathbb{R}^{n_i+1}$, we have
\begin{align}
    \partial^2_x \sum^{n_i}_{\alpha=0} c_\alpha p^{(\mathrm{mon})}_\alpha = \sum^{n_i}_{\alpha=0} (M^{i}_{dd} c)_\alpha p^{(\mathrm{mon})}_\alpha,
\end{align}
where
\begin{align}
    M_{dd}^{i} = 
    \begin{pmatrix}
          &  & 2 \\
         & & & 6  \\
          & & & & \ddots  \\
        & & & & & n_i(n_i-1) \\
          & & & & & 0 \\
        & & & & & 0
    \end{pmatrix}\in\mathbb{R}^{(n_i+1)\times (n_i + 1)}.
\end{align}
In a similar way, we get 
\begin{align}
    x\partial_x \sum^{n_i}_{\alpha=0} c_\alpha p^{(\mathrm{mon})}_\alpha = \sum^{n_i}_{\alpha=0} (M^{i}_{xd} c)_\alpha p^{(\mathrm{mon})}_\alpha,
\end{align}
where
\begin{align}
    M_{xd}^{i} = 
    \begin{pmatrix}
        0 \\
        & 1 \\
        & & 2 \\
        & & & \ddots \\
        & & & & n_i
    \end{pmatrix}\in\mathbb{R}^{(n_i+1)\times (n_i+1)}.
\end{align}
With the basis transformation matrix $T_{i,n_i}$ we can express the action of these operators on the coefficients of the original basis $p_i$. In particular, we have
\begin{equation}\label{eq:app_di}
\begin{aligned}
     ( \partial_x^2 + x\partial_x ) \sum^{n_i}_{\alpha=0} c_\alpha p^i_\alpha &= ( \partial_x^2 + x\partial_x )\sum^{n_i}_{\alpha=0} (T_{i,n_i}c)_\alpha p^{(\mathrm{mon})}_\alpha \\
   &= \sum^{n_i}_{\alpha=0} ((M_{dd}^{i}+M^{i}_{xd})T_{i,n_i}c)_\alpha p^{(\mathrm{mon})}_\alpha \\
   &= \sum^{n_i}_{\alpha=0} (T_{i,n_i}^{-1}(M_{dd}^{i}+M^{i}_{xd})T_{i,n_i}c)_\alpha p_\alpha = \sum^{n_i}_{i=0} (D_{i}c)_\alpha p^i_\alpha,
\end{aligned}
\end{equation}
with $D_{i} \coloneqq T_{i,n_i}^{-1}(M_{dd}^{i}+M^{i}_{xd})T_{i,n_i} \in \mathbb{R}^{(n_i+1)\times (n_i+1)}$.

\subsection{Derivation of the nonlinear part}\label{app:nonlinear_part}

Note that for every $c\in\mathbb{R}^{n_i+1}$, we have
\begin{align}
    \partial_x \sum^{n_i}_{\alpha=0} c_\alpha p^{(\mathrm{mon})}_\alpha = \sum^{n_i}_{\alpha=0} (M^{i}_{d} c)_\alpha p^{(\mathrm{mon})}_\alpha,
\end{align}
where
\begin{align}
    M_{d}^{i} = 
    \begin{pmatrix}
          & 1 \\
         & & 2  \\
          & & & \ddots  \\
        & & & & n_i \\
          & & & & 0 
\end{pmatrix}\in\mathbb{R}^{(n_i+1)\times (n_i+1)},
\end{align}
and hence 
\begin{align}
    \partial_x \sum^{n_i}_{\alpha=0} c_\alpha p^i_\alpha &= \partial_x \sum^{n_i}_{\alpha=0} (T_{i,n_i}c)_\alpha p^{(\mathrm{mon})}_\alpha \\
   &= \sum^{n_i}_{\alpha=0} (M_{d}^{i}T_{i,n_i}c)_\alpha p^{(\mathrm{mon})}_\alpha \\
   &= \sum^{n_i}_{\alpha=0} (T_{i,n_i}^{-1}M_{d}^{i}T_{i,n_i}c)_\alpha p^i_\alpha = \sum^{n_i}_{\alpha=0} (D_{x_i}c)_\alpha p^i_\alpha,
\end{align}
with $D_{x_i} = T_{i,n_i}^{-1}M_d^{i}T_{i,n_i} \in \mathbb{R}^{(n_i+1)\times (n_i+1)}$.

\subsection{Estimating the eigenvalues for Gaussian distributions}\label{app:eigenvalues}

Let $\mathbf{n} = (2,\ldots,2)^{\intercal} \in \mathbb{N}^{d}$ and $g(x) = \frac{1}{2}x^{\intercal}Ax$ for all $x\in\mathbb{R}^d$, where $A\in\mathbb{R}^{d\times d}$. Note that $\nabla g(x) = Ax$ and hence for any $v\in\operatorname{span}\Pi_{\mathbf{n}}$, we have $\operatorname{NonLin}_g v = \langle Ax, \nabla v \rangle = (Ax)\cdot\nabla v$. Hence, the linearized HJB at $g$ reads
\begin{align}
    \dot{v} = x\cdot \nabla v + \Delta v -2\langle Ax,\nabla v\rangle + \langle Ax, Ax \rangle, \qquad v(0) = g
\end{align}
In order to determine the stiffness, we need to determine the effect of this right-hand side on the coefficient tensor of $v$. From now on, let $v=v_{\bm{C}}$, where $\bm{C}\in\mathbb{R}^{\bm{n}+1}$. Assume that $v_{\bm{C}}(x) = \frac{1}{2}x^{\intercal}\hat{C}x$ for some $\hat{C}\in\mathbb{R}^{d\times d}$, i.e. $V_{\bm{C}}$ has only terms with degree $2$. We know that $x\cdot\nabla v_{\bm{C}} +\Delta v_{\bm{C}} = v_{\bm{L}\bm{C}}$ and $\langle Ax, \nabla v_{\bm{C}} \rangle = \sum_{i=1}^d \sum_{j=1}^d a_{ij}x_j \partial_i v_{\bm{C}}$. Now, note that $\partial_i v_{\bm{C}} = v_{(I_d\otimes \ldots \otimes P_i \otimes \ldots \otimes I_d) \bm{C}}$, where
\begin{align}
    P_i = T_{i,2}^{-1} \begin{pmatrix}
        0 & 1 & 0\\
        0 & 0 & 2 \\
        0 & 0 & 0
    \end{pmatrix} T_{i,2},
\end{align}
and $x_j v_{\bm{C}} = v_{(I_d\otimes \ldots \otimes X_j \otimes \ldots \otimes I_d) \bm{C}}$, where
\begin{align}
    X_j = T_{j,2}^{-1} \begin{pmatrix}
        0 & 0 & 0\\
        1 & 0 & 0 \\
        0 & 1 & 0
    \end{pmatrix} T_{j,2}.
\end{align}
In the case of $i=j$ this leads to $x_i\partial_i V_C = v_{(I_d\otimes \ldots \otimes T_{i,2}^{-1} M^{i}_{xd} T^{-1}_{i,2} \otimes \ldots \otimes I_d) \bm{C}}$, where
\begin{align}
    M^{i}_{xd} = \begin{pmatrix}
        0 & 0 & 0\\
        0 & 1 & 0 \\
        0 & 0 & 2
    \end{pmatrix}.
\end{align}
Let 
\begin{align}
    \bm{M} = \sum_{i,j=1}^d a_{ij} I_d \otimes \ldots \otimes I_d \otimes  P_i \otimes I_d \otimes \ldots \otimes I_d \otimes X_j \otimes I_d\otimes \ldots \otimes I_d,
\end{align}
then we have $\langle Ax, \nabla v_{\bm{C}} \rangle = v_{\bm{M}\bm{C}}$. 

\textbf{Diagonal covariance.} We consider the special case where $A = \mathrm{diag}(a_{ii},i=1,\ldots,d)$ is a diagonal matrix. In this case, we have 
\begin{align}
    \bm{M} = \sum_{i=1}^d a_{ii} I_d \otimes \ldots  I_d \otimes T_{i,2}^{-1} M^{i}_{xd} T_{i,2} \otimes I_d\otimes \ldots \otimes I_d
\end{align}
and hence the linear operator governing the right-hand side is given by
\begin{align}
    \bm{L}-2\bm{M} = \sum_{i=1}^d  I_d \otimes \ldots  I_d \otimes H_{i} \otimes I_d\otimes \ldots \otimes I_d,
\end{align}
where $H_i \coloneqq D_i-2a_{ii} T_{i,2}^{-1} M^{i}_{xd} T_{i,2} = T_{i,2}^{-1} (M^{i}_{dd}+(1-2a_{ii})  M^{i}_{xd}) T_{i,2} $ and
\begin{align}
    M^{i}_{dd}+(1-2a_{ii})  M^{i}_{xd} = \begin{pmatrix}
        0 & 0 & 2\\
        0 & 1-2a_{ii} & 0 \\
        0 & 0 & 2(1-2a_{ii}).
    \end{pmatrix}
\end{align}
The point spectrum $\sigma(H_i)$ of $H_{i}$ is given by $\sigma(H_i) = \{0, 1-2a_{ii}, 2(1-2a_{ii})\}$. The eigenvector corresponding to the eigenvalue with largest absolute value $2(1-2a_{ii})$ is given by
\begin{align}
    \hat{v}_i = T_{i,2}^{-1}\left( \frac{1}{1-2a_{ii}},0,1 \right)^{\intercal}T_{i,2}.
\end{align}
Let $v_i$ denote any eigenvector of $H_i$. Then, $v = (v_1 \otimes \ldots  \otimes v_d)$ is an eigenvector of $\bm{L}-2\bm{M}$. Since this leads to $3^d$ possible combinations, the whole spectrum of $\bm{L}-\bm{M}$ is defined by such eigenvectors. Moreover, since the eigenvalues $H_{i}$ are bounded by $|2(1-2a_{ii})|$, the  largest absolute eigenvalue of $\bm{L}-2\bm{M}$ is given by $2\sum_{i=1}^d |1-2a_{ii}|$.

\fi
\ifnum\online=1
\fi

\end{document}